\documentclass[journal]{IEEEtran}
\usepackage{times}
\usepackage{helvet}
\usepackage{courier}
\usepackage{amssymb}
\usepackage{epsfig}
\usepackage{graphicx}
\usepackage{amsthm}
\usepackage{amsmath}
\usepackage{lineno}
\usepackage{color}
\usepackage{float}
\usepackage{subfigure}
\usepackage[linesnumbered,boxed,ruled]{algorithm2e}
\usepackage{url}
\usepackage{multirow}
\usepackage{diagbox}
\usepackage{comment}
\usepackage{threeparttable}
\usepackage{emp}
\usepackage{bbding}
\usepackage{hyperref}
\hypersetup{
	colorlinks=true,
	linkcolor=red,
}

\renewcommand{\tabcolsep}{0.025cm}
\newtheorem{thm}{Theorem}

\newcommand{\whll}{\textcolor[rgb]{0.00, 0.00, 0.00}}
\newcommand{\whlll}{\textcolor[rgb]{0.00, 0.00, 0.00}}
\newcommand{\whllll}{\textcolor[rgb]{0.00, 0.00, 0.00}}
\newcommand{\whlllll}{\textcolor[rgb]{0.00, 0.00, 0.00}}
\newcommand{\WH}{\textcolor[rgb]{0.00, 0.00, 0.00}}
\newcommand{\WWH}{\textcolor[rgb]{0.00, 0.00, 0.00}}
\newcommand{\WWWH}{\textcolor[rgb]{0.00, 0.00, 0.00}}
\newcommand{\WWWWH}{\textcolor[rgb]{0.00, 0.00, 0.00}}
\newcommand{\WWWWWH}{\textcolor[rgb]{0.00, 0.00, 0.00}}
\newcommand{\wss}{\textcolor[rgb]{0,0,0}}
\newcommand{\WWWWWWH}{\textcolor[rgb]{0.00, 0.00, 0.00}}

\begin{document}
\title{One-pass Person Re-identification by \\  Sketch Online Discriminant Analysis}
%\title{Sketch Online Discriminant Analysis for Person re-identification }
\author{Wei-Hong Li,~Zhuowei Zhong,~and~Wei-Shi Zheng$^*$% <-this % stops a space%I love CV%Weihong Lee,~Zhuowei Zhong,~Wei-Shi Zheng,~and ...% <-this % stops a space
%\thanks{iSEE}% <-this % stops a space
\thanks{Wei-Hong Li is with the School of Electronics and Information Technology, Sun Yat-sen University, Guangzhou, China.  
E-mail: liweih3@mail2.sysu.edu.cn}
\thanks{Zhuowei Zhong is with the School of Data and Computer Science, Sun Yat-sen University, Guangzhou, China. 
%\newline 
E-mail: zhongzhw6@gmail.com}
\thanks{Wei-Shi Zheng is with the School of Data and Computer Science,
Sun Yat-sen University, Guangzhou, China.
% Guangdong
%Provincial Key Laboratory of Computational Science, Guangzhou, China and
%Collaborative Innovation Center of High Performance Computing, National
%University of Defense Technology, Changsha, China.
%\newline
E-mail: wszheng@ieee.org/zhwshi@mail.sysu.edu.cn}
\thanks{* Corresponding author.}}
\maketitle

\begin{abstract}	

Person re-identification (re-id) is to match people across disjoint camera views in a multi-camera system, and re-id has been an important technology applied in smart city in recent years. However, the majority of existing person re-id methods are not designed for processing sequential data in an online way. This ignores the real-world scenario that person images detected from multi-cameras system are coming sequentially. While there is a few work on discussing online re-id, \WH{most of them} require \WWH{considerable} storage of all passed data samples that have been ever observed, and this could \WWH{be unrealistic}
%not be realistic 
for processing data from a large camera network. In this work, we present an one-pass person re-id model that adapts the re-id model based on each newly observed data and no passed data are directly used for each update. 
\WH{More specifically, we develop an Sketch} online Discriminant Analysis (SoDA) by embedding sketch processing into Fisher discriminant analysis (FDA). SoDA can efficiently keep the main data variations of all passed samples in a low rank matrix when processing sequential data samples, and \wss{estimate the approximate within-class variance} (i.e. within-class covariance matrix) from the sketch data information.
We provide theoretical analysis on the effect of the estimated approximate within-class covariance matrix. In particular, we  derive upper and lower bounds on the \WWH{Fisher discriminant score} (i.e.  the quotient between between-class \WWH{variation} and within-class \WWH{variation} after feature transformation) in order to investigate how the optimal feature transformation learned by SoDA \wss{sequentially} approximates the offline FDA that is learned on all observed data.
Extensive experimental results have shown the effectiveness of our SoDA and empirically support our theoretical analysis.

%This enables our SoDA to be directly learned on extremely high dimensional data with a high efficiency on both space and time, which is also characteristic as compare with other online/incremental FDA models.

%Further, we estimate between-class and within-class covariance matrices using the sketch matrix and class mean obtained during sketch.  
%and thus the size of dataset is large or infinity.
%Due to the fact that most of recent effective features proposed for \whlll{discriminating} different identities are very high dimensional, which greatly increases the computation and memory complexity, we adopt a set of orthogonal basis learned by sketch itself for dimension reduction

%We empirically evaluate SoDA on four applications with large-scale data: handwritten digit recognition, face recognition, person re-identification, and fast object tracking.
%We have compared the related incremental FDA methods and recent online classifiers. 

%We empirically evaluate SoDA on three large datasets by using three types of recent features.
%The experimental results have verified the effectiveness of our theoretical analysis.

\end{abstract}


\section{Introduction}

% 1. what is Person re-id and describe problems about Person re-id
% (1) the variance caused by illumination, viewpoint, occlusions and background clutter
% (2) the scale of some recent proposed dataset
% (3) curse of dimension
% 2. Why is FDA (minimize the within-class variance and maxmize the between-class variance)
% 3. mention why we need a online method
% 4. We proposed SoDA to overcome this difficulty and the solid theoretical analysis is also provided.
% 5. Our contribution

%Person re-identification (re-id), an essential task in computer vision, is to automatically re-identify individual person across non-overlapping camera views distributed at different physical locations.

Person re-identification (re-id) \cite{reiddata_zheng2015scalable,reid_ahmed2015improved,reid_jing2015super,reid_liao2015person,reid_paisitkriangkrai2015learning,reid_li2015multi,reid_zheng2015partial} is crucially important for successfully tracking people in a large camera network. It is to match the same person's images captured at non-overlapping camera views at different time.
Person re-id by visual matching is inherently challenging because of the existence of many visually similar persons and dramatic appearance changes of the same person caused by the serious cross-camera-view variations such as illumination, viewpoint, occlusions and background clutter.
Recently, a large number of works \cite{reid_liao2015person,reid_liao2015efficient,reidsub_chen2015mirror,reidsub_koestinger2012large,reidsub_ma2014person,reidsub_mignon2012pcca,reidsub_prosser2010person,reidsub_xiong2014person,reidsub_zheng2011person} have been reported to solve this challenge.

However, it is largely unsolved to perform online learning for person re-identification, since most person re-id models except \whllll{\cite{hitl_liu2013pop,hitl_wang2016human,hitl_martinel2016temporal,OLDML_sun2014online}} are only suitable for offline learning. 
On one hand, the offline learning mode cannot enable a real-time update of person re-id model when a large amount of persons are detected in a camera network. An online update is important to keep the cross-view matching system work on recent mostly interested persons, that is to make the whole re-id system work on sequential data. On the other hand, online learning is helpful to alleviate the large scale learning problem (either with high-dimensional feature, or on large-scale data set, or both) nowadays. By using online learning, especially the one-pass online learning, it is not necessary to always store (all) observed/passed data samples.

In this paper, we overcome the limitation of offline person re-id methods by developing an effective online person re-id model. We proposed to embed the sketch processing into Fisher discriminant analysis (FDA), \WWH{and the new model is called Sketch online Discriminant Analysis (SoDA). In SoDA, the sketch processing preserves the main variations of all passed data samples in a low-rank sketch matrix, and thus SoDA enables selecting data variation for acquring discriminant features during online learning. \whlll{SoDA enables} the newly learned discriminant model to embrace information from a new coming data sample in the current round and meanwhile retain important information learned in \whlll{previous rounds} in a light and fast manner without directly saving any passed observed data samples and keeping large-scale covariance matrices, so that SoDA is formed as an one-pass online adaptation model. While no passed data samples are saved in SoDA, we propose to estimate the within-class variation from the sketch information (i.e. a low-rank \emph{sketch matrix}), and thus in SoDA an approximate within-class covariance matrix can be derived. We have provided in-depth theoretical analysis on how sketch affects the discriminant feature extraction in an online way. The rigorous upper and lower bounds on how SoDA approaches its offline model (i.e. the classical Fisher Discriminant Analysis \WH{\cite{webb2003statistical}}) are presented and proved. }

%\WWH{The latter case is reasonable because although the sketch in SoDA enables selecting data variation during the online learning, more data information is kept when a much larger sketch matrix $\mathbf{B}$ is used, and this will be verified in the experiments (see Figure \ref{fig:FisherScore} for example)}. 
%However, keeping more information without selection does not mean better recognition performance, because 

%, and thus it could generate more discriminant feature transformations

% In particular, for discriminant feature extraction in an one-pass learning way, we extract the approximate within-class scatter information from , which . 

% While the between-class covariance can be estimated based on the updated the class means, we are able to update a discriminant feature model by minimizing the approximate within-class scatter and maximizing the between-class scatter information. 

Compared to existing online models for person re-id \whllll{\cite{hitl_liu2013pop,hitl_wang2016human,hitl_martinel2016temporal,OLDML_sun2014online}}, SoDA is succinct\wss{, but it is theoretically guaranteed and effective}. While \WH{most }existing online re-id \WWH{models} \WWWWH{have} to retain all observed passed data samples, the proposed SoDA relies on the sketch information from historical data without any explicit storage of passed data samples, and sketch information will \WWH{assist our online model in preventing one-pass online model from being biased by a new coming data.}
%help our online model to alleviate the bias of new coming data. 
While a more conventional way for online learning of FDA is to update both within-class and between-class covariance matrices directly \WWH{\cite{pang2005incremental,ye2005idr,uray2007incremental,lu2012incremental,peng2013chunk,kim2011incremental}, we introduce a novel approach to realize online FDA by mining any within-class information from a sketch data matrix, and this provides a lighter, more effecient and effective  online learning for FDA.  We also find that} an extra benefit of embedding sketch processing in SoDA is to simultaneously embed dimension reduction as well, so that no extra learning task on learning dimension reduction technology (e.g. PCA) is required and SoDA is more flexible when learning on some high dimensional data 
 \whlll{\cite{reid_liao2015person,yingcongpami_chen2017person}} in an online manner. 

We have conducted extensive experiments on three \WWH{largest} scale person re-identification datasets in order to evaluate the effectiveness of SoDA for learning person re-identification model in an online way. Extensive experiments are also included for comparing SoDA with related online learning models, even though \whlll{they were} not applied to person re-identification before.

%In our development, a matrix sketch technique is employed to maintain the main variations of all passed data in a low rank sketch matrix of which each row is a carefully chosen frequent direction during the whole online process.
%Due to the small sketch size, both the memory used to store data and the time spent on handling sequential data at each round are limited and fixed.
%Solid theoretical analysis is presented to give a theoretical upper and lower bounds for the proposed online person re-id method on the Fisher discriminant score.
%2) We introduce to adopt a set of orthogonal basis learned by sketch itself for dimension reduction so that our SoDA can be directly learned on very high dimensional feature efficient.
%This enables our online person re-id model achieve better performance easily by learning on a kind of extremely high dimensional feature which concatenates several strong features.

The rest of the paper is organized as follows. In Sec. \ref{section:relatedwork}, the related literatures are first reviewed. We elaborate our  online algorithm and analyze the space and time complexity of SoDA in Sec. \ref{section:SoDA_method}. Then we present theoretical analysis on the relationship between our SoDA and the \wss{offline} FDA in Sec. \ref{section:theory}. Experimental results for evaluation and verification of our theoretical analysis are reported in Sec. \ref{section:Experiments} and finally we conclude the work in Sec. \ref{section:conclusion}.
%/*** Need to consider why selecting FDA as an instance for online re-identification ***/

\section{Related Work}\label{section:relatedwork}

\vspace{0.15cm}

\noindent \textbf{Online Person re-identification}.
While person re-identification has been \whlll{investigated} \wss{in} a large number of works \whlll{\cite{reiddata_zheng2015scalable,reid_ahmed2015improved,reid_jing2015super,reid_paisitkriangkrai2015learning,reid_li2015multi,reid_zheng2015partial,reid_liao2015person,reid_liao2015efficient,reidsub_chen2015mirror,reidsub_koestinger2012large,reidsub_ma2014person,reidsub_mignon2012pcca,reidsub_prosser2010person,reidsub_xiong2014person,reidsub_zheng2011person,reid_panda2017unsupervised,reid_martinel2015re,reid_yaolarge}}, the majority of them only address \WH{by} offline learning. That is person re-id model \whlll{is learned on} a fixed training dataset. This ignores the increase demand of data from a visual surveillance system, since thousands of person images are captured day by day and it is demanded to train a person re-id system on streaming data so as to keep the system update to date.

\whlll{Recently, only a few works \cite{OLDML_sun2014online,hitl_liu2013pop,hitl_wang2016human,hitl_martinel2016temporal} have been developed towards online processing for person re-identification. 
The most related work is the incremental distance metric based online learning mechanism (OL-IDM) proposed in \cite{OLDML_sun2014online}. For updating the KISSME metric \cite{KISSME_koestinger2012large}, the OL-IDM utilizes the modified Self-Organizing Incremental Neural Network (SOINN) \cite{SOINN_furao2006incremental} to produce two pairwise sets: a similar pairs set and a dissimilar pairs set.
%Despite the SOINN enable the KISSME to be learned on sequential data efficiently, it has to keeps all passed data in the network, which is costly when the feature dimension is high or the method was applied to real-world scenario.
Although SOINN enables learning KISSME \WH{\cite{KISSME_koestinger2012large}} on sequential data, SOINN has to compare the newly observed sample with all \whlllll{the preserved nodes and adds the newly observed sample as a new node if it does not appear in the network.} This would be costly as sequential data increase and when feature dimension is high.
}

\whlll{
Another related work is the human-in-the-loop \WH{ones} \cite{hitl_wang2016human,hitl_liu2013pop,hitl_martinel2016temporal}, which proposed incremental method learned with the involvement of humans' feedback.
%is an incremental learning method with the involvement of humans' feedback.
%Wang et al. \cite{hitl_wang2016human} proposed to employ an operator to browse the gallery ranking list which was previously produced and select the true match for each probe and update the model.
\whlllll{Wang et al. \cite{hitl_wang2016human} assumes that an operator is available to scan the rank list provided by the proposed algorithm when matching a new probe sample with existing observed gallery ones\WWH{, and this operator will select the true match, strong-negative match, and weak-negative match for the probe}. After having the human feedback, the algorithm is able to be update.}
%Wang et al. \cite{hitl_wang2016human} assumes that an operator is available to provide the rank list when matching a new sample with existing observed ones and thus the algorithm is able to select the true-match, strong-negative, and weak-negative for each probe and update the model.
Martinel et al. presented a graph-based approach to exploit the most informative probe-gallery pairs for reducing human efforts and developed an incremental and iterative approach based on the feedback \cite{hitl_martinel2016temporal}. 
%, which verifies whether two samples in a pair is the same persons.
%The feedback that verifies whether two samples in a pair is the same persons was then used for updating the offline training model.
%The human involvement means that .....
}
%Recently, only a few works \cite{hitl_liu2013pop,hitl_wang2016human} are towards online processing for person re-identification. The most related work is the human-in-the-loop work \cite{hitl_wang2016human}, which is an incremental learning method with the involvement of humans' feedback.  
%The human involvement means that .... 
%Our work is orthogonal to this work, since we discuss how to automatically update a person re-identification model without any human interaction, and thus our work and the human-in-the-loop work can accompany each other.

%/***Need to enrich the related online methods here ***/
%\whlll{/***weihong: DML has already been added above***/}

%Compared with these models, 
%\WH{Unlike these models,} we design an sketch FDA model called SoDA for one-pass online learning, without any storage of passed observed samples, but instead maintaining a small size sketch matrix on handling streaming data so that the discriminant projections can be updated efficiently for extracting discriminative features for identifying different individuals.

Unlike these models, we design \wss{a} sketch FDA model called SoDA for one-pass online learning, without any storage of passed observed samples, \WWH{maintaining} a small size sketch matrix on handling streaming data so that the discriminant projections can be updated efficiently for extracting discriminative features for identifying different individuals.

%to update discriminant projections for extracting discriminant features to identify different individuals in test.
Thanks to the sketch matrix, our SoDA is capable of obtaining comparable performance with offline FDA models on streaming data or large and high dimensional datasets with very low cost on space and time. Compared to the related online person re-id models, SoDA is theoretically sounded since the bounds on approximating the offline model is provided.

In particular, compared to Wang et al.'s and Martinel et al.'s work, our work has the following distinct aspects: Firstly, the proposed SoDA is developed for the one-pass online learning, while Wang et al.'s and Martinel et al.'s work cannot work for one-pass online learning, because \whlllll{the former one \WH{requires} human feedback between probe sample and all preserved gallery samples, and the latter one needs to store all sample pairs during interative learning.}
%they need the feedback between a newly observed sample and the passed samples. 
Secondly, the proposed SoDA could be orthogonal to the human-in-the-loop work, since we discuss how to automatically update a person re-identification model on streaming data without elaborated human interaction (feedback), and thus our work and the idea of incorporating more human interaction in human-in-the-loop work can accompany each other.

%Compared with these models, we designed an online FDA model which is 

%% and develop approaches for discriminant different identities.
%However, In the real world, data are captured by vision intelligent system, and can be presented sequentially (/***weihong: cite a figure***/).
%Therefore, models learned offline can not adapt to the variance of the new data.
%Besides, over time, the size of data become larger and event infinity. These offline methods have incapacity to learned on such large or infinity since it takes much time and memory.
%In comparison to offline learning, \textbf{online learning} methods handle data sequentially and are highly successful at rapidly reducing the test error on large, high-dimensional datasets.
%/***weihong: complement more***/

\vspace{0.15cm}

\begin{figure*}[!htp]
	\begin{center}
		\label{fig:datastream}
		\includegraphics[height=0.49\linewidth]{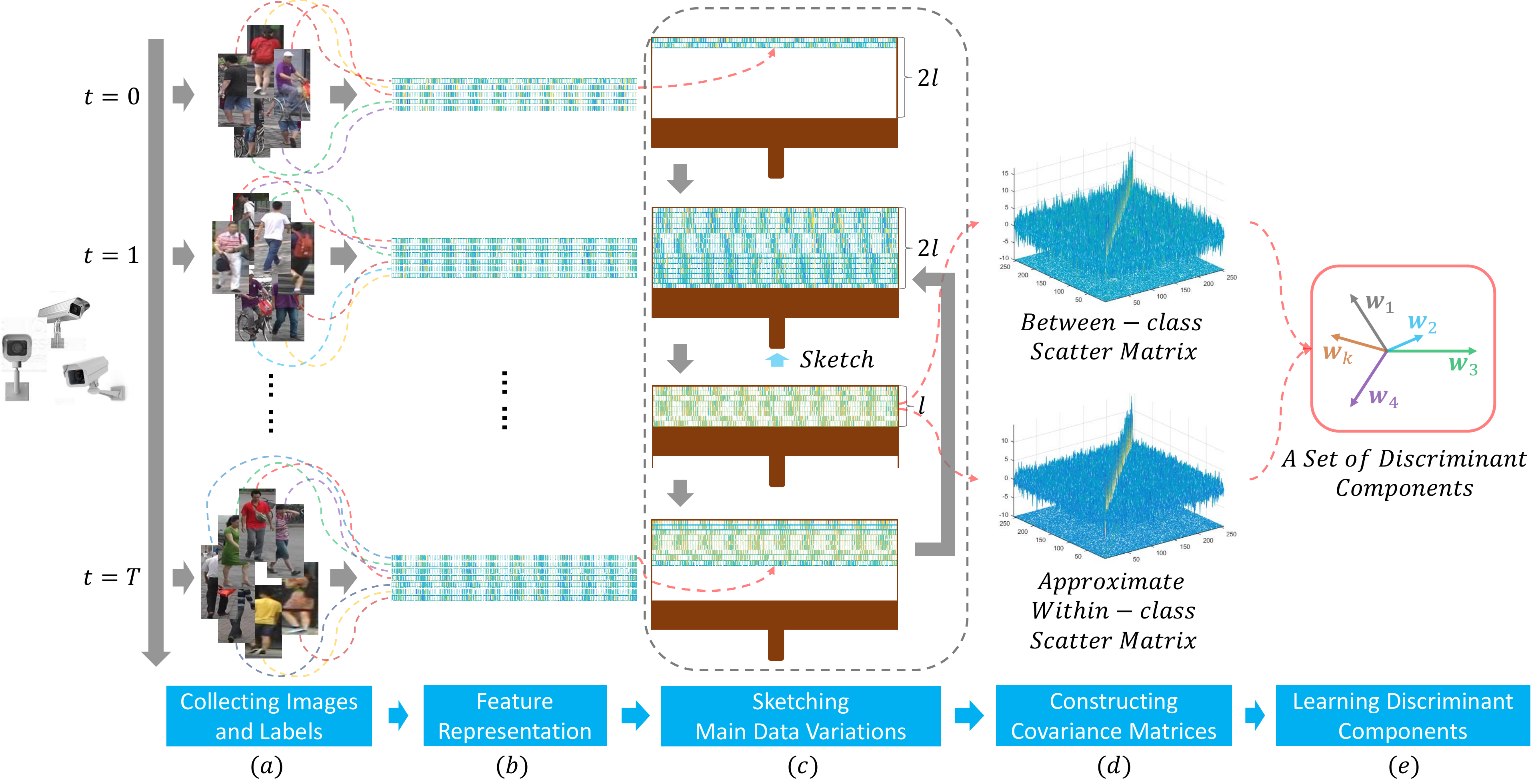}
		%		\fbox{\rule{0pt}{1.7in}\rule{0.9\linewidth}{0pt}}
		%\includegraphics[width=\linewidth]{figure/ImageDemo.pdf}
		\centering\small\caption{Illustration of our proposed Sketch online Discriminant Analysis (SoDA) (Best viewed in color). (a) In real-world application, images are generated endlessly from visual surveillance camera network. (b) ($t=0,1,\cdots,T$), every presented image is represented by a $d-$dimensional row feature vector. (c) We maintain a low rank sketch matrix to summarize all passed data by matrix sketch: 1) At the begining, we set $\mathbf{B}\in\mathcal{R}^{2\ell\times d}$, the sketch matrix, to be \WWH{a zero} matrix. 2) All rows of $\mathbf{B}$ would \whlll{be} filled by $2\ell$ samples from top to bottom one by one. 3) we maintain the main data \whlll{variations} in the upper half of $\mathbf{B}$ by sketch. 4) Each row of the lower half of $\mathbf{B}$ is set to be \whlll{all} zero and will be replaced by a new sample. (d) After sketch, the between-class and within-class covariance matrices are constructed. (e) Due to the sketch, we can \whlll{update} a set of discriminant components efficiently only using \whlll{limited} space and time.}
		%		\fbox{\rule{0pt}{2.7in}\rule{0.9\linewidth}{0pt}}
		%		%\includegraphics[width=\linewidth]{figure/ImageDemo.pdf}
		%		\centering\small\caption{/***weihong: add figure for illustrating how data are acquired sequentially ***/}
		\label{fig:datastream}
	\end{center}
\end{figure*}

\noindent \textbf{SoDA vs. Incremental Fisher Discriminant Learning}.
% /***Jason: can we have a table***/
%\whlll{/***weihong: why we should have a table?***/}
SoDA is related to existing incremental/online Fisher Discriminant Analysis (FDA) methods\whlll{,} which aim to update within-class and between-class covariance matrix sequentially.
Pang et al. proposed to directly update the between-class and within-class scatter matrices \cite{pang2005incremental}.
However, Pang et al.'s method has to preserve the whole scatter matrices in the memory, which becomes impractical for high dimensional data. 
Ye et al. \cite{ye2005idr} and Uray et al. \cite{uray2007incremental} performed online learning by updating PCA components to derive an approximate update of scatter matrices. Compared to Pang's method, Ye's and Uray's can only perform online learning sample by sample, which can be time consuming for large scale data. Also, Ye's method is based on QR decomposition of between-class covariance matrix, and therefore it would increase computational cost when the number of class is large. Since, Ye's method is limited to learning discriminant projections in the range space of between-class covariance matrix but not the range space of total-class covariance matrix \cite{yang2005kpca}, \whlll{which} may lose discriminant information. Lu et al. proposed a complete model that picks up the lost discriminant information \cite{lu2012incremental}. But Lu's method only can update the model sample by sample.
Peng et al. alternately proposed a chuck version of Ye's method in order to process multiple data points at a time \cite{peng2013chunk}.  
Kim et al. proposed a sufficient spanning set based incremental FDA \cite{kim2011incremental} to overcome the limitations in the previous works. Since it is hard to directly update the discriminant components in FDA, Yan et al. \cite{yan2004immc} and Hiraoka et al. \cite{hiraoka2000successive} modified FDA in order to get the discriminant components updated. They proposed iterative methods for directly updating discriminant projections. 

%Compared to these methods,  a characteristic of the proposed SoDA is to update the selects main variation of passed data by sketch processing for online learning. More importantly, we present the bound analysis about the sketch approximation for discriminant feature extraction.

Compared to the above mentioned incremental/online FDA methods, \WWH{our proposed SoDA embeds sketch processing into FDA and therefore mines the within-class scatter information from a sketch data matrix rather than directly from samples.} This gives the benefit that while the passed data samples are not necessary to \whlll{be saved}, SoDA is still able to extract useful within-class information from the compressed data information contained in the sketch matrix. In general, SoDA is an online version of FDA, and SoDA can not only approximate the FDA\WH{,} which optimizes discriminant components on whole data directly, but also run \WWH{faster} with limited memory. Also, dimension reduction is naturally embedded into SoDA and no extra online model for dimension reduction \WH{is required.}
% to learn.
In-depth theoretical investigation is provided in Sec. \ref{section:theory} to explain its rationale and \WH{to} guarantee its effectiveness.

%in a low rank and utilizes a set of orthogonal basis learned during sketch for dimension reduction. Our approach  In some case, our SoDA outperforms the FDA since there is noise which can be alleviated by sketch.

Although the proposed SoDA can be seen as embedding sketch processing into FDA, we contribute solid theoretical analysis on how SoDA will approximate the Batch mode FDA when estimating the within-class variations from sketch information, where 
% verify the use of sketch for online discriminant feature extraction by presenting  
 the lower bound and upper bound are provided. The theoretical analysis guarantees SoDA to be an effective and efficient online learning method.

%the related several developments of online FDA, which is sometimes known as incremental FDA. So far, the existing online FDA models focus on the strategies for updating the between-class covariance and within-class covariance scatter matrices \cite{pang2005incremental,ye2005idr,uray2007incremental,peng2013chunk,yang2005kpca,lu2012incremental,kim2011incremental,yan2004immc,hiraoka2000successive}. Compared to these works, the most characteristic of our method is that only main data variation is retained in a low-rank matrix by sketch processing, and not all newly available instances will trigger the update of the sketch matrix until enough new samples are received. In this aspect, the update of SoDA is lighter. In addition, it is known that the dimension reduction is necessary for FDA when processing high-dimensional data. While it is always necessary to conduct extra online learning processing on dimension reduction, we find that the sketch itself can be considered as learning a set of basis in online mode, and therefore dimension reduction is naturally embedded in SoDA, which is also unique as compared to the other online/incremental FDA methods.
%

\vspace{0.15cm}

\noindent \textbf{Online Learning}.
SoDA is an online learning methods. In literatures, online learning \cite{online_al_chechik2010large,online_al_crammer2006online,online_al_jain2009online,online_al_warmuth2008randomized,online_al_huang2017online} is known as a light and rapid \whlll{means} to process streaming data or large-scale datasets, and it has been widely exploited in many real-world tasks such as Face Recognition \cite{online_fr_kim2010line,online_fr_schroff2015facenet}, Images Retrieval \cite{online_ir_liang2017semisupervised,online_ir_wu2016online} and Object Tracking \cite{li2016online,online_tracking_li2016deeptrack}.
It enables learning a up-to-date model based on streaming data.
However, most of these online leaning based models \cite{online_al_crammer2006online,online_tracking_li2016deeptrack,li2016online} are not suitable for person re-identification, since they are incapable of predicting labels of data samples from unseen classes which do not appear in the training stage.

%/***weihong: do we have to comment these online models?***/

%SoDA will perform selection on data variation and only the frequent data variation will be kept during online learning due to the sketch processing in SoDA. This is a major difference from 
%
%%  As closely related to SoDA, there exists 
%

%\noindent \textbf{Matrix Sketch}

%In this work, we consider the following real-world setting.
%At each round, after all new coming samples are received, the learner updates the between-class and within-class scatter information and afterwards the discriminant projection $\mathbf{W}\in \mathcal{R}^{d\times m}$ (with $m$ subspace dimension) should be updated as efficient as possible if the update takes place.

%\section{Proposed Method}
\section{Sketch online Discriminant Analysis (SoDA)}\label{section:SoDA_method}
%In this section, 

In this section, we start to present the Sketch online Discriminant Analysis (SoDA) for Person re-identification.
In real-world scenario, samples \whlll{come} endlessly and sequentially from vision system (Figure \ref{fig:datastream}).
The number of samples received in each round \WWWWH{is random}, and the individual sample obtained is also stochastic.
Suppose the $t^{th} (t = 1,2,\cdots)$ new coming sample represented as a $d-$dimensional \WWWH{feature vector $\mathbf{x}_{i} \in \mathcal{R}^{d}$} is labelled with class label \whlll{$\mathbf{y}_i$}.
For convenience, at the $t^{th}$ round, we denote all passed data (i.e. $N$ training samples collected in the current and previous rounds) as a  training sample matrix \WWWH{$\mathbf{X} = {[\mathbf{x}_{1}, \mathbf{x}_{2}, \cdots, \mathbf{x}_{N}]}^{T} \in \mathcal{R}^{N\times d}$}, and denote all the corresponding labels as $\mathbf{y} = [\mathbf{y}_1, \mathbf{y}_2, \cdots, \mathbf{y}_N]^T\in\mathcal{R}^{N}$ where $\mathbf{y}_i$ is the class label of \WWWH{$\mathbf{x}_{i}$} and $\mathbf{y}_i \in \{1,2,...,C\}$.
% /***Should make the notation consistent here ***/
%\whlll{/***weihong: I think the notation may be consistent now***/}

At each round ($t = 1, 2, \cdots$), the proposed SoDA maintains the main variations of all passed data \WH{($\mathbf{X}\in\mathcal{R}^{N\times d}$)} in a low rank matrix\WH{,} which is named as the ``sketch matrix''.
The sketch matrix keeps a small number of selected frequent directions\whlll{,} which are obtained and updated by a matrix sketch technique during the whole online learning process.
%a small set of frequent directions in a low rank matrix which named a sketch matrix to represent 
% our SoDA sketches all of passed data denoted by $\mathbf{X}\in\mathcal{R}^{N\times d}$ into a low rank $\mathbf{B}\in\mathcal{R}^{\ell\times d}$. %, where $\ell<<N$ and is still .
While sketching main data variations, the population mean and the one of each class are also updated.
We further utilize these updated means and the low rank sketch matrix to estimate between-class \WH{covariance matrix} and derive the approximate \WWH{within-class covariance matrix} after all new coming samples are \WWH{compressed} into the sketch matrix.
%After sketch, we utilize the updated mean and the low rank sketch matrix to estimate between-class and within-class covariance matrix.
Finally, we generate discriminant components by eigenvalue decomposition for simultaneously minimizing the approximate within-class variance and maximizing the between-class variance. 
The whole procedure of SoDA is illustrated in Figure \ref{fig:datastream} and presented in Algorithm \ref{algorithm:SoDA}.
The in-depth theoretical investigation to explain why SoDA can approximate the offline FDA model by sketch and guarantee its effectiveness on extracting discriminant components is provided in Sec. \ref{section:theory}.

%Based on the updated mean and the sketch low rank matrix, 
%Suppose we have a training set $\mathbf{X} = {[\mathbf{x}_1, \mathbf{x}_2, \cdots, \mathbf{x}_N]}^{T}\in\mathcal{R}^{N\times d}$ where $\mathbf{x}_t$ is a $d-$dimensional training sample feature captured sequentially from a Person re-identification Camera System.
%In real world, since the training data are collected sequentially, $N$ is large or even infinity.
%Moreover, most of recently proposed person re-id features are high dimensional.
%Therefore, 

%As describing in Sec. \ref{section:PS}, we assume that training samples are captured by a Person Re-Indentification System sequentially.
%For convenience of description, we denote all training samples by $\mathbf{X}$ consisting of rows representing sample $\mathbf{x}_i$ collected at $i$ th time step.
%In following, we present a Sketched Online Discriminant Analysis for Person Re-Indentification.

%We first present an overview of the algorithm in Sec. \ref{section:overview}, and then each step is detailed in Sec. \ref{section:main_sketch}. While our strategy seems direct to fuse Sketch and FDA, the main characteristic of SoDA is to approximate the within-class covariance matrix under the selection of main data variation by sketch. And MORE IMPORTANTLY, we present in-depth theoretical investigation to explain its rationale and guarantee its effectiveness in Sec. \ref{section:theory}.

\begin{algorithm}[!t]
	\footnotesize
	\caption{Sketch online Discriminant Analysis}\label{algorithm:SoDA}
	\KwIn{$\mathbf{X}={[\mathbf{x}_{1}, \mathbf{x}_{2}, \cdots, \mathbf{x}_{N}]}^{T}\in\mathcal{R}^{ N \times d},\mathbf{y}\in\mathcal{R}^{N},\lambda > 0$}%, \mathbf{B}\in\mathcal{R}^{\ell \times d}$, where~$\ell>0$}
%	$\tilde{\mathbf{A}}^+=\mathbf{B}$;
	$\mathbf{B}\longleftarrow$ zero matrix $\in\mathcal{R}^{2\ell \times d}$;
%	($\tilde{\mathbf{A}}$, a $2\ell \times d$ zero matrix, is used to sketch all passed samples to $\mathbf{B}$)
	\\
	\For{each data \WWWH{$\mathbf{x}_{i}\in \mathcal{R}^{d}$} and label $\mathbf{y}_i$}
	{
		using \WWWH{$\mathbf{x}^{T}_{i}$} to replace one zero row of $\mathbf{B}$\;
		\If{all samples in $\mathbf{X}$ are processed}
		{
			deleting all zero rows of $\mathbf{B}$;
		}
		
		\If{ $\mathbf{B}$ has no zero rows }
		{$[\mathbf{U},\mathbf{\Sigma},\mathbf{V}] = \mathbf{SVD}(\mathbf{B})$\;
			
			setting $\xi$ as the $(\ell+1)^{th}$ \WWH{largest} element $\mathbf{\Sigma}_{\ell+1}$ of $\mathbf{\Sigma}$\;
			$\WWWWH{\hat{\mathbf{\Sigma}}} = \sqrt{max(\mathbf{\Sigma}^2-\mathbf{I}_{2\ell}\WWH{\xi^2},\mathbf{O})} $\;
			$\mathbf{B} = \WWWWH{\hat{\mathbf{\Sigma}}}\mathbf{V}^{T}$($\mathbf{B}$ contains $\ell$ rows non-zero \whlll{values})\;
		}
		$\mathbf{m}_c \longleftarrow ( N_c \mathbf{m}_c + \WWWWH{\mathbf{x}_{i}} ) / ( N_c + 1)$\ ( $c = 0, \mathbf{y}_i$ )\;
		$N_c \longleftarrow  N_c + 1$\ ( $c = 0, \mathbf{y}_i$ )\;
	}
	$\mathbf{B} \longleftarrow \mathbf{B}^{+}, \mathbf{P} = \mathbf{V}^+$\;
%	$\mathbf{B} = \tilde{\mathbf{A}}^+,$ $\mathbf{P} = \mathbf{V}^+$\;
	$\mathbf{S}_b = \sum_{c=1}^C \frac{N_c}{N_0 } (\mathbf{m}_c - \mathbf{m}_0)(\mathbf{m}_c - \mathbf{m}_0)\WWH{^{T}}$\;
	$\tilde{\mathbf{S}}_t = \mathbf{B}^T\mathbf{B} / N_0 - \mathbf{m}_0\mathbf{m}_0^{T}$\;
	$\tilde{\mathbf{S}}_w = \tilde{\mathbf{S}}_t - \mathbf{S}_b$\;
%	$\hat{\mathbf{S}}_t = \mathbf{P}^T\tilde{\mathbf{S}}_t\mathbf{P}$\;
	$\hat{\mathbf{S}}_b = \mathbf{P}^T\mathbf{S}_b\mathbf{P}$\;
	$\hat{\mathbf{S}}_w = \mathbf{P}^T\tilde{\mathbf{S}}_w\mathbf{P}$\;
%	$\hat{\mathbf{S}}_w = \hat{\mathbf{S}}_t - \hat{\mathbf{S}}_b$\;
	$[\mathbf{W},\mathbf{\Lambda}] = \mathbf{EVD}(\hat{\mathbf{S}}_b,\hat{\mathbf{S}}_w)$\;
	
	\KwOut{$\mathbf{B}, \mathbf{W}, \WWWWWWH{\mathbf{\Lambda}}$}
\end{algorithm}

%\subsection{Sketching Main Data Variations}
%In order to enable the FDA model to adapt to sequential data, 
%In order to enable FDA models to embrace information of all data 

%Our SoDA processes rows of $\mathbf{X}$ one by one and produces a sketch matrix $$
%The zero rows of $\mathbf{B}$ from top to bottom are replaced by the new coming samples 
%In this work, all training samples acquired sequentially from vision system are view as a large matrix $\mathbf{X}\in\mathcal{R}^{N\times d}$ consisting of $N$ training samples.
%Each sample is represented as a $d-$dimensional feature vector.
%Processing and preserving $\mathbf{X}$ require large memory and much time. 
%In order to reduce the complexity of memory and computation, we propose a light way SoDA which sketches all passed data into a low rank matrix $\mathbf{B}\in\mathcal{R}^{\ell\times d}$ such that , in which $\ell << N$.
%More importantly, 
%/***weihong: need to add***/
%In general, training samples are obtained sequentially, the totally collected training sample matrix $\mathbf{X}\in\mathcal{R}^{N\times d}$ is large or infinity.
%we sketch all passed training samples into a low rank matrix $\mathbf{B}$ at each time step so that it takes less time and memory for processing training data in a real world vision system. 
%Initially, we set matrix $\mathbf{B} \in \mathcal{R}^{2\ell \times d}$ to zero matrix.
%At each time step, each new received trainining samples replaces one all zero row of $\mathbf{B}$ from top to bottom until the $\mathbf{B}$ is full without any all zero rows.

\subsection{Estimating Between-class covariance matrix}

During online learning, we keep updating the population mean $\mathbf{m}_0$ and mean of each class $\mathbf{m}_c$ ($c$ = 1, 2, $\ldots$, $C$) so as to construct the between-class covariance matrix $\mathbf{S}_b$.
When having a new coming sample \WWWH{$\mathbf{x}_{i}$} with class label $\mathbf{y}_i$, the population mean and mean of class \WWH{$\mathbf{y}_i$} are updated by
\begin{equation}
\small
%\mathbf{m}_c = (  N_c\mathbf{m}_c + \mathbf{x}_{i,:} ) / ( N_c + 1), \ c = 0,1,2,...,C
\whlllll{\mathbf{m}_c = (  N_c\mathbf{m}_c + \WWWH{\mathbf{x}_{i}} ) / ( N_c + 1), \ c = 0, \mathbf{y}_i,}
\end{equation}
and the population number and the number of samples for class \WWH{$\mathbf{y}_i$} are also updated by:
\begin{equation}
\small
%N_c =  N_c  + 1, \ c = 0,1,2,...,C.
\whlllll{N_c =  N_c  + 1, \ c = 0, \mathbf{y}_i.}
\end{equation}

We then use the updated means to estimate the between-class covariance matrix as follows:
\begin{equation}\label{equation:sketch_noise_r}
%\begin{split}
\mathbf{S}_b =  \sum_{c=1}^C \frac{ N_c}{ N_0 } (\mathbf{m}_c - \mathbf{m}_0)(\mathbf{m}_c - \mathbf{m}_0)\WWWH{^{T}}.
%&\tilde{\mathbf{S}}_t = \mathbf{B}^T\mathbf{B} / N_0  - \mathbf{m}_0^T\mathbf{m}_0,
%\end{split}
\end{equation}

%Suppose that the number of the observed samples with label $c$ is $ N_c$ and the number of population is $N_0$.

%During sketch, our method processes rows of $\mathbf{X}$ one by one.
%when a sample $\mathbf{x}^{T}_i$ replaces one zero row of $\mathbf{B}$, a label $\mathbf{y}_i$ corresponding to $\mathbf{x}_i$ is also acquired.
%When a new training sample $\mathbf{x}_t$, a $d-$dimensional vector, and a label $\mathbf{y}_t$ corresponding to $\mathbf{x}_t$ are presented from the data stream at time point $t$. 

\subsection{Estimating Approximate Within-class covariance matrix}

For realizing one-pass online learning, we aim to \whlllll{update/form} the within-class covariance matrix which describes the within-class variation without using any passed observed data samples. \WWH{Different from previous online FDA approaches, we \wss{embed} sketch processing into FDA and derive a novel approximate within-class covariance matrix efficiently and effectively.} For this purpose, we first employ the sketch technique \cite{LibertySketching} to compress the passed data samples into a sketch matrix so as to maintain the main variations of passed data. More specifically, we maintain the main variations of all passed data $\mathbf{X}$ in a small size matrix $\mathbf{B}\in\mathcal{R}^{2\ell \times d}$\WWWH{,} called a sketch matrix, where $\mathbf{B}$ is initialized by a zero matrix. 
Each new coming sample \WWWH{$\mathbf{x}^{T}_{i}$} (i.e. the $i$-th row of $\mathbf{X}$) replaces a zero row of $\mathbf{B}$ from top to bottom until $\mathbf{B}$ is full without any all zero rows. When $\mathbf{B}$ is full, we apply Singular Value Decomposition (SVD) on $\mathbf{B}$ such that $\mathbf{U}\Sigma\mathbf{V}^T = \mathbf{B}$, where $\Sigma$ is a diagonal matrix with singular values on the diagonal in decreasing order. 
Each row in $\mathbf{V}^T$ corresponds to a singular value in $\Sigma$, and let \WWWH{vectors} \WWWH{$\{\mathbf{v}_{j}\}$} of $\mathbf{V}^{T}$ corresponding to the first half singular values \WWH{denoted} as $frequent\ directions$  and the ones corresponding to lower half singular values \WWH{denoted} as $unfrequent\ directions$.
By employing the sketch algorithm, the frequent directions \WWWH{$\mathbf{v}_{j}$} are scaled by $\sqrt{\lambda_i^2-\xi^2}$ and retained in $\mathbf{B}$, \whlll{where $\xi$ is the $(\ell+1)^{th}$ \WWH{largeast singular value in} 
%	/***Jason: am i right? if right, please also update Algorithm 1***/ 
$\mathbf{\Sigma}_{\ell+1}$ of $\mathbf{\Sigma}$}.
In this way, the sketch matrix $\mathbf{B}$ is obtained by $\mathbf{\hat{\Sigma}} \mathbf{V}^{T}$, where \WWH{$\mathbf{\hat{\Sigma}} =  \sqrt{\max(\mathbf{\Sigma}^{2} - \mathbf{I}_{2\ell} \xi^2, \mathbf{O})}$} 
% /***Jason: am i right? if right, please also update Algorithm 1***/ 
and $\mathbf{O}$ is a zero matrix.
\whlll{Therefore, the sketch matrix $\mathbf{B}$ is a $2\ell \times d$ matrix, where $\mathbf{B}^{+}$, the upper half of $\mathbf{B}$, retains the main variations of passed data samples, and $\mathbf{B}^{-}$, the lower half of $\mathbf{B}$, is reset to zero.}

%Note that, at the beginning, the sketch matrix $\mathbf{B}$ is initialized to zero matrix, and after that, $\mathbf{B}^{+}$, the upper half of $\mathbf{B}$, retains the main variations of the passed data while the lower half of $\mathbf{B}$ is set to zeros.

Although no passed observed data samples are saved, we propose to derive an approximate within-class covariance matrix using the sketch matrix $\mathbf{B}$ below:
\begin{equation}\label{eq:Swupdated}
%\small
\tilde{\mathbf{S}}_w = \tilde{\mathbf{S}}_t - \mathbf{S}_b,
\end{equation}
where \begin{equation}\label{equation:sketch_noise_r}
%\begin{split}
%&\mathbf{S}_b =  \sum_{c=1}^C \frac{ N_c}{ N_0 } (\mathbf{m}_c - \mathbf{m}_0)^T(\mathbf{m}_c - \mathbf{m}_0),\\
\tilde{\mathbf{S}}_t = \mathbf{B}^T\mathbf{B} / N_0  - \mathbf{m}_0\mathbf{m}_0^{T}.
%\end{split}
\end{equation}
In the above, $\tilde{\mathbf{S}}_w$ is not always the exact within-class covariance matrix but it is an approximate one. In Sec. \ref{section:theory}, we will provide in-depth theoretical analysis of the bias of this approximation on discriminant feature component extraction.

\subsection{Dimension Reduction and Extraction of Discriminant Components }

Normally, after updating the two covariance matrices $\mathbf{S}_b$ and $\tilde{\mathbf{S}}_w$, it is only necessary to compute the generalized eigen-vectors of \WWH{$\mathbf{\Lambda} \tilde{\mathbf{S}}_w \mathbf{W} = \mathbf{S}_b \mathbf{W}$.} However, in person re-identification, some kinds of features are of high dimensionality such as HIPHOP \whlll{\cite{yingcongpami_chen2017person}}, LOMO \whlll{\cite{reid_liao2015person}} and etc, and the size of the two covariance matrices $\mathbf{S}_b$ and $\tilde{\mathbf{S}}_w$ was determined by the feature dimensionality.
Thus the above eigen-decomposition remains costly when the size of both $\mathbf{S}_b$ and $\tilde{\mathbf{S}}_w$ are large.
%the size of $\mathbf{S}_b$ and $\tilde{\mathbf{S}}_w$ which rely on feature dimension.
%The above direct eigen-decomposition on such large scale matrix remains costly
%This remain high the above direct eigen-decomposition is very costly for high dimensional data. 

An intuitive solution is to conduct another online learning for dimension reduction, which spends extra time and space. However, SoDA does not require such an extra learning. 
Due to sketch, SoDA actually maintains a set of frequent directions that describe main data variations. 
And thus we take these frequent directions as basis vectors and the span of them can approximate the data space.
Hence, we set ${\mathbf{P}=\mathbf{V}^{T}}^+$, the upper half of matrix $\mathbf{V}^{T}$ (Line \whlll{16} in Algorithm \ref{algorithm:SoDA}), and the dimension reduction is performed by:
\begin{equation}\label{eq:reducedim}
\begin{split}
&\hat{\mathbf{S}}_b = \mathbf{P}^T\mathbf{S}_b\mathbf{P}, \\
&\hat{\mathbf{S}}_w = \mathbf{P}^T\tilde{\mathbf{S}}_w\mathbf{P},
\end{split}
\end{equation}
where \WWWH{$\mathbf{P} = [\mathbf{v}_{1}, \mathbf{v}_{2}, \ldots, \mathbf{v}_{k}]$} consists of $k$ frequent directions. In this way, $\hat{\mathbf{S}}_b$ and $\hat{\mathbf{S}}_w$ become matrices in $\mathcal{R}^{k\times k}$, and computing generalized eigen-vectors will become much faster. Finally, the generalized eigen-vectors (\WWH{Line 22 in Algorithm 1}) are computed by \WWH{$\mathbf{\Lambda}\hat{\mathbf{S}}_w\mathbf{W} = \hat{\mathbf{S}}_b\mathbf{W}$,}
%$\hat{\mathbf{S}}_w\mathbf{W} = \mathbf{\Lambda}\hat{\mathbf{S}}_b\mathbf{W}$, 
and they are the 
discriminant components we pursuit.

\subsection{Computational Complexity}\label{section:CC}

As presented above, after processing all observed samples, we maintain $\mathbf{B}\in\mathcal{R}^{\ell \times d}$, $\mathbf{P}\in\mathcal{R}^{d \times k}$, $\mathbf{m}_c\in\mathcal{R}^{d}$ and {$ N_c$}($c$ = 0, 1, 2, \ldots, $C$).
The time and space cost of the rest procedure is \WWH{$\mathcal{O}( d\ell^2 )$}
%$\mathcal{O}( Nd\ell )$
\whlllll{(After the whole processing, $N_0$ is equal to $N$)} and $\mathcal{O}((\ell+C) d)$, respectively.
%, and thus the cost is . 
Therefore, the cost of time and space is \WWH{$\mathcal{O}( d\ell^2 )$}
%\whlllll{$\mathcal{O}(Nd\ell)$} 
and $\mathcal{O}((\ell + k +\whlllll{C}) d)$, respectively, almost the same as the cost of sketch algorithm \cite{LibertySketching}.

\section{Theoretical Analysis}\label{section:theory}
In this section, we theoretically show that SoDA approximates FDA in a principled way\WWH{, although SoDA is formed based on the approximate within-class covariance matrix mined from sketch data information. }
%outperforms FDA since there is noise in data which can be alleviated during sketch.

\subsection{Fisher Discriminant Analysis}

Fisher discriminant analysis (FDA) aims to seek discriminant projections for minimizing within-class variance and maximizing between-class variance\whlll{,} which are estimated over the \WWH{data matrix} $\mathbf{X}$ and \WWH{its label set} $\mathbf{y}$ in an offline way. 
%We call this the conventional FDA model the \whlll{FDA} in this work, in contrast to online FDA.
There are several equivalent criteria $\mathbf{J}_F$ for the multi-class case. 
For analysis, we consider the one that maxmizes the following criterion:
\begin{equation}
\small
 \mathbf{J}_F(\mathbf{W}) = \frac {\mathbf{W}^T\mathbf{S}_b\mathbf{W}}{\mathbf{W}^T\mathbf{S}_w\mathbf{W}},
\end{equation}
where 
%$\mathbf{x}^{c}_{i}$ is the $i^{th}$ sample in the $c^{th}$ class, 
$\mathbf{S}_b$ is the between-class covariance matrix and  $\mathbf{S}_w$ is the within-class covariance matrix. 
They are given by
\begin{equation}\label{eq:SB}
\small
\mathbf{S}_b =  \sum_{c=1}^C \frac{ N_c}{ \WWWWWWH{N} } (\mathbf{m}_c - \mathbf{m}_0)(\mathbf{m}_c - \mathbf{m}_0)\WWWH{^{T}},
\end{equation}
\begin{equation}\label{eq:SW}
\small
\WWWH{\mathbf{S}_w = \displaystyle \sum_{c=1}^C \frac{ N_c}{ \WWWWWWH{N} } \sum_{\mathbf{y}_i=c} \frac{1}{ N_c}({\mathbf{x}}_{i} - \mathbf{m}_c )({\mathbf{x}}_{i} - \mathbf{m}_c )\WWWH{^{T}},}
\end{equation}
where $\mathbf{m}_c$ and $N_c$ are the data mean and the number of samples of the $c^{th}$ class, respectively, 
%and  $ N_0 = N $ 
and \WWWWWWH{$N$ and $\mathbf{m}_0$} are the population number and population mean, respectively.
And the total covariance matrix is
\begin{equation}\label{St}
\small
\WWWH{\mathbf{S}_t = \mathbf{S}_w + \mathbf{S}_b = \displaystyle \frac{1}{ \WWWWWWH{N} } \sum_{c=1}^C\sum_{\mathbf{y}_i=c} ({\mathbf{x}}_{i} - \mathbf{m}_0)({\mathbf{x}}_{i} - \mathbf{m}_0)\WWWH{^{T}}.}
\end{equation}
Generally, the analysis seeks a set of feature vectors $\{\mathbf{w}_{j}\}$ that maximize the criterion subject to the normalization constraint $tr(\mathbf{W}^T\mathbf{S}_b\mathbf{W}) =1$, where $\mathbf{W}$ is the matrix whose columns are \WWH{$\{\mathbf{w}_{j}\}$}. This leads to the computation of generalized \whlll{eigenvectors}, that is \WWH{ $\mathbf{\Lambda}\mathbf{S}_w\mathbf{W} = \mathbf{S}_b\mathbf{W}$}
and $\mathbf{\Lambda}$ is a diagonal matrix with generalized eigenvalues on the diagonal. 
Here, the eigenvectors corresponding to the largest eigenvalues are used to compress a high dimensional data vector to a low dimensional feature representation.

%Such kind of FDA that learned in a batch way is called Batch FDA in this work, in contrast to online FDA.
%Such kind of FDA that works in a batch way with the whole dataset given in advance is called Batch FDA in this work, in contrast to online FDA. 

%Though Batch FDA can gain well performance and has been widely used in person re-id.
%Due to the several aforementioned limitations, adopting Batch FDA to learned on real-world application would lead to high time consumption and large space cost.
%Hence, an online learning is necessary.

\subsection{Relation between SoDA and FDA}

Before presenting the theoretical analysis, we first define the following notations.
Let \WWH{
\begin{equation}\label{eq:fisher_score_two}
\begin{split}
\mathbf{J}_F^1(\mathbf{W}) = \frac{tr( \mathbf{W}^T\mathbf{S}_b\mathbf{W})}{tr(\mathbf{W}^T\mathbf{S}_w\mathbf{W})}, \\
\mathbf{J}_F^2(\mathbf{W}) = \frac{tr(\mathbf{W}^T\mathbf{S}_b\mathbf{W})}{tr(\mathbf{W}^T\tilde{\mathbf{S}}_w\mathbf{W})},
\end{split}
\end{equation}
where $\mathbf{J}_F^1(\mathbf{W})$ is the conventional FDA criterion and $\mathbf{J}_F^2(\mathbf{W})$ is SoDA criterion by replacing $\mathbf{S}_w$ with $\tilde{\mathbf{S}}_w$ that is mined from sketch data information.}

Let \WWH{the largest Fisher scores in the above equations be}
\begin{equation}\label{eq:fisher_score_two}
\begin{split}
\mathbf{J}_F^1(\whlll{\mathbf{W}^1}) &= \underset{{\mathbf{W} \in R^{d\times k}}}{max}\mathbf{J}_F^1(\mathbf{W}) = \mu_1, \\
\mathbf{J}_F^2(\whlll{\mathbf{W}^2}) &= \underset{\mathbf{W} \in R^{d\times k}}{max}\mathbf{J}^2_{F}(\mathbf{W}) = \mu_2.
\end{split}
\end{equation}
Since for optimizing Eqs. (\ref{eq:fisher_score_two}), we can form a Lagrangian function by imposing the constraint $tr(\mathbf{W}^T\mathbf{S}_b\mathbf{W})=1$ for both criteria \WWH{\cite{webb2003statistical}}
% /***Jason: add a citation on classical pattern recognition book***/. 
We define $\mathcal{D}=\{ \mathbf{W}=[\mathbf{w}_{1},\cdots,$ $\mathbf{w}_{k}] \in R^{d\times k}| tr(\mathbf{W}^T\mathbf{S}_b\mathbf{W}) =1  \}$, and thus we can reform Eqs. (\ref{eq:fisher_score_two}) by:
\begin{equation}\label{eq:fisher_score_two_reduced}
\begin{split}
\mu_1^{-1}&=\underset{\whlll{\mathbf{W}^1} \in \mathcal{D}}{\min} \{\mathbf{J}_F^1(\mathbf{W}_1)\}^{-1} =tr(\mathbf{W}^T\mathbf{S}_w\mathbf{W}), \\
\mu_2^{-1}&=\underset{\whlll{\mathbf{W}^2} \in \mathcal{D}}{\min} \{\mathbf{J}^2_{F}(\mathbf{W}_2)\}^{-1}=tr(\mathbf{W}^T\tilde{\mathbf{S}}_w\mathbf{W}).
\end{split}
\end{equation}

\begin{figure}[t]
	\begin{center}
		\includegraphics[height=0.7\linewidth]{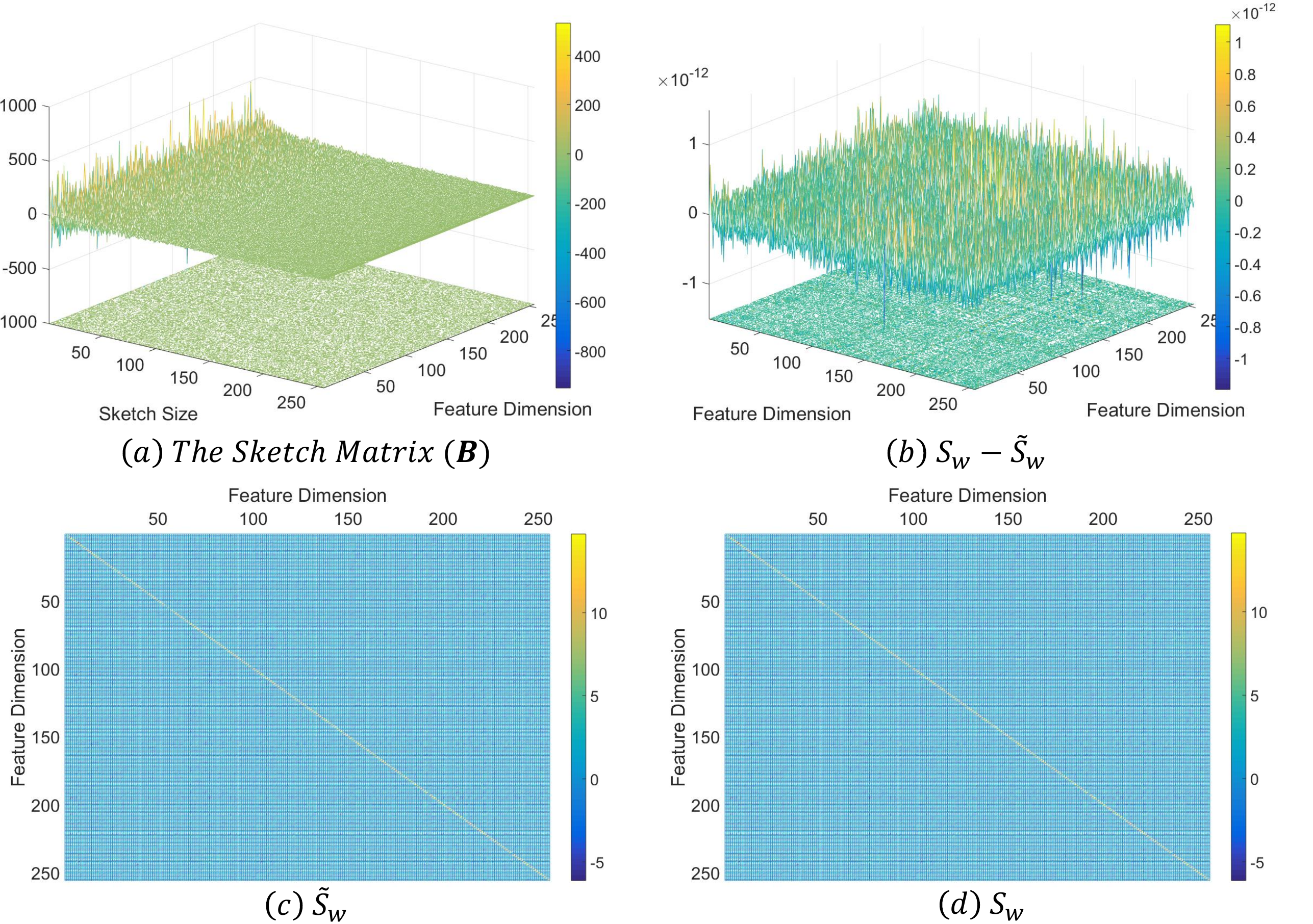}
		%		\fbox{\rule{0pt}{1.7in}\rule{0.9\linewidth}{0pt}}
		%		\includegraphics[width=\linewidth]{figure/ImageDemo.pdf}
		\centering\small\caption{\whlllll{(a) is the sketch matrix ($\mathbf{B}$). (c) is the approximate within-class covariance matrix ($\tilde{\mathbf{S}}_w$) generated by \textbf{SoDA} while (d) is the groundtruth one ($\mathbf{S}_w$) produced by \textbf{FDA}. (\WWWH{b}) is the difference ($ \mathbf{S}_w - \tilde{\mathbf{S}}_w$) of the groundtruth within-class covariance matrix and the approximate one. It is noteworthy that the distinction between $ \mathbf{S}_w$ and $\tilde{\mathbf{S}}_w$ is \WWH{\textbf{less than}} $1 \times 10^{-12}$, which indicates that $\tilde{\mathbf{S}}_w$ estimated by SoDA can approximate the groundtruth one}  (Best viewed in color).}
		\label{fig:energyofmatrix}
	\end{center}
\end{figure}

In the following sections, we first discuss the relationship between $\mu_1$ and $\mu_2$. And then this relationship will be used to present a bound for $\mathbf{J}^1_{F}(\whlll{\mathbf{W}^2})$. Note that $\mathbf{J}^1_{F}(\whlll{\mathbf{W}^2})$ is to measure how well the optimal projection learned by our SoDA approximates the optimal solution that maximizes $\mathbf{J}^1_{F}(\mathbf{W})$. Note that our analysis will not take any dimension reduction before extracting discriminant components below for discussion. Our analysis can be extended if the same dimension reduction is applied to all methods discussed below.

\subsection{Relationship Between the Maximum Fisher Score of \whlll{FDA} and \WWH{that} of SoDA}
We first present the relationship between the maximum Fisher score of \whlll{FDA} and the one of SoDA, i.e. the relationship between $\mu_1$ and $\mu_2$.
Suppose that matrix $\mathbf{X}\in\mathcal{R}^{N\times d}$ is the totally training sample set consisting of samples acquired at each time step.
%Suppose that matrix $\mathbf{X}\in\mathcal{R}^{N_0 \times d}$ consists of all sets of samples so far at time point $t$, and rows of $\mathbf{X}$ are $d$-dimensional samples from $C$ classes. In other words, $\mathbf{X} = [\mathbf{X}_0; \mathbf{X}_1; \ldots; \mathbf{X}_t]$.

However, it is not intuitive to obtain the relationship between the maximum Fisher score of \whlll{FDA} and the one of SoDA based on the covariance matrices inferred in Eq. (\ref{equation:sketch_noise_r}).
In order to exploit such a relationship, we first investigate the Fisher score obtained by $\mathbf{S}_b$ and the \WWH{approximate within-class covariance} matrix $\tilde{\mathbf{S}}_w$ as follows:
% \whlllll{(note that $N_0 = N$ here.}):
\begin{equation}
\tilde{\mathbf{S}}_w = \tilde{\mathbf{S}}_t - \mathbf{S}_b = \mathbf{B}^T\mathbf{B}/ \WWWWWWH{N} - \mathbf{m}_0\mathbf{m}_0^{T} - \mathbf{S}_b.
\end{equation}

Let $\mathbf{S}_w$ be the within-class covariance matrix computed in \whlllll{batch mode (i.e. \wss{for offline} FDA)}. Since it is known that $\mathbf{S}_w=\mathbf{S}_t-\mathbf{S}_b = \mathbf{X}^{T}\mathbf{X}/\WWWWWWH{N} - \mathbf{m}_0\mathbf{m}_0^{T} - \mathbf{S}_b$, it can be verified that
\begin{equation}\label{eq:SubSw}
\small
\begin{split}
&\mathbf{S}_w - \tilde{\mathbf{S}}_w\\
=& ( \mathbf{X}^T\mathbf{X}/ \WWWWWWH{N}  - \mathbf{m}_0\mathbf{m}_0^{T} - \mathbf{S}_b )- (\mathbf{B}^T\mathbf{B}/ \WWWWWWH{N} - \mathbf{m}_0\mathbf{m}_0^{T} - \mathbf{S}_b ) \\
=&(\mathbf{X}^T\mathbf{X} - \mathbf{B}^T\mathbf{B})/ \WWWWWWH{N}. \\
\end{split}
\end{equation}
By combining Eq. (\ref{eq:FD2}) \WWH{as stated in the Appendix}, it is not hard to have the following theorem about the relation between $\mathbf{S}_w$ and $\tilde{\mathbf{S}}_w$, and we visualize the approximation between the groundtruth within-class covaraince matrix and our approximate one in Figure \ref{fig:energyofmatrix}. We assume that $\mathbf{S}_w$, $\tilde{\mathbf{S}}_w$ and $\hat{\mathbf{S}}_w$ are not singular in the following analysis \footnote{The analysis can be generalized to the case when $\tilde{\mathbf{S}}_w$ is not invertible if the same regularization is imposed on both $\mathbf{S}_w$, $\tilde{\mathbf{S}}_w$ and $\hat{\mathbf{S}}_w$}.

\begin{thm}\label{th:SW}
	$\tilde{\mathbf{S}}_w \preceq \mathbf{S}_w,$ and $||\mathbf{S}_w - \tilde{\mathbf{S}}_w||\le 2{||\mathbf{X}||}_f^2/(\WWWWWWH{N}\ell)$, where $||*||$ is the induced norm of a matrix and $||*||_f$ is the Frobenius norm.
\end{thm}

%Then from \cite{LibertySketching}, we first have:

%The first theorem is to consider the relationship between  under $2-porm$ matrix form as follows.

Based on the above theorem, we particularly consider the two-class classification case.

%\subsection{An Interpretation for Two-class case}

%\newpage
\begin{thm}\label{th:twoclass}
	Considering the two criteria in Eq. (\ref{eq:fisher_score_two_reduced}) when the discriminant feature transformation is a one-dimensional vector, i.e. $\whlll{\mathbf{W}^1}=\mathbf{w}^{1} \in \mathcal{R}^d$ and $\whlll{\mathbf{W}^2}=\mathbf{w}^{2} \in \mathcal{R}^d$,
	%\begin{equation}\label{eq:two_fisher}
	%\small
	%\mathbf{J}_F^1(\mathbf{w}) = \frac {\mathbf{w}^T\mathbf{S}_b\mathbf{w}}{\mathbf{w}^T\mathbf{S}_w\mathbf{w}} \ , \ \mathbf{J}_F^2(\mathbf{w}) = \frac {\mathbf{w}^T\mathbf{S}_b\mathbf{w}}{\mathbf{w}^T\tilde{\mathbf{S}}_w\mathbf{w}}.
	%\end{equation}
	%Let the largest Fisher scores be
	%\begin{equation}
	%\underset{{\mathbf{w}_1 \in R^d}}{\arg max}\mathbf{J}_F^1(\mathbf{w}_1) = \mu_1, \ \underset{\mathbf{w}_2 \in R^d}{\arg max}\mathbf{J}^2_{F}(\mathbf{w}_2) = \mu_2.
	%\end{equation}
	the relationship between $\mu_1$ and $\mu_2$ is as follow:
	\begin{equation}\label{eq:sw_sw2}
	\mu_1^{-1} - 2(s_0r_b)^{-\frac{1}{2}}{||\mathbf{X}||}_f^2/(\WWWWWWH{N}\ell) \le \mu_2^{-1} \le \mu_1^{-1},
	\end{equation}
	where $s_0$ is the smallest (non-zero) \WWH{singular value} of matrix  $\mathbf{S}_b$ and $r_b=rank(\mathbf{S}_b)$.
\end{thm}

\begin{proof}
	Let $\mathbb{D} = 2{||\mathbf{X}||}_f^2/(\WWWWWWH{N}\ell)$. From the Theorem \ref{th:SW}, we have 
	$\text{for any nonzero } \mathbf{w} \in \mathcal{R}^d, \ 0\le \frac{\mathbf{w}^T(\mathbf{S}_w - \tilde{\mathbf{S}}_w)\mathbf{w}}{||\whlll{\mathbf{w}||_2}} \le \mathbb{D}$. That
	is
	\begin{equation}\forall \mathbf{w} \in \mathcal{R}^d,  \ \mathbf{w}^T\tilde{\mathbf{S}}_w\mathbf{w} \le \mathbf{w}^T\mathbf{S}_w\mathbf{w}, \mathbf{w}^T{\mathbf{S}}_w\mathbf{w} \le \mathbf{w}^T\tilde{\mathbf{S}}_w\mathbf{w} + \mathbb{D}||\mathbf{w}||_2.
	\end{equation}
	%As usual, $\mathbf{w}^T\tilde{\mathbf{S}}_w\mathbf{w} = 1$ and $\mathbf{w}^T\mathbf{S}_w\mathbf{w} = 1$ are imposed for the two criteria
	%in Eq.(\ref{eq:two_fisher}) for computing the discriminant vectors \cite{belhumeur1997eigenfaces}. It can be verified that the same directions of eigenvectors
	%can be obtained by constraining $\mathbf{w}^T\mathbf{S}_b\mathbf{w} = 1$ for the Criteria in Eq.(\ref{eq:two_fisher}) for maximization, thus obtaining the same maximum Fisher scores of the Criteria in Eq.(\ref{eq:two_fisher}). 
	
	Let $\mathbf{w}^1$ and $\mathbf{w}^2$ be the discriminant vectors that \WWH{minimize} the Criterion in Eq. (\ref{eq:fisher_score_two_reduced}) under the constraints ${\mathbf{w}^1}^T\mathbf{S}_b\mathbf{w}^1 = 1$ and ${\mathbf{w}^2}^T\mathbf{S}_b\mathbf{w}^2 = 1$, respectively. That is ${\mathbf{w}^1}^T\mathbf{S}_w\mathbf{w}^1 = \mu_1^{-1}$ and ${\mathbf{w}^2}^T\tilde{\mathbf{S}}_w\mathbf{w}^2 = \mu_2^{-1}$, i.e. $\mathbf{w}^1$ and $\mathbf{w}^2$ would minimize $\mathbf{w}^T\mathbf{S}_w\mathbf{w}$ and $\mathbf{w}^T\tilde{\mathbf{S}}_w\mathbf{w}$ when constraining $\mathbf{w}^T\mathbf{S}_b\mathbf{w} = 1$. In addition, since ${\mathbf{w}^2}^T\mathbf{S}_b\mathbf{w}^2=1$, we have $s_0r_b||\mathbf{w}^2||_2^2\leq 1$, i.e. $||\mathbf{w}^2||_2\leq (s_0r_b)^{-\frac{1}{2}}$. Therefore, based on Theorem \ref{th:SW}, we have
	\begin{equation}
	\begin{split}
	&\mu_2^{-1} = {\mathbf{w}^2}^T\tilde{\mathbf{S}}_w\mathbf{w}^2 \leq{\mathbf{w}^1}^T\tilde{\mathbf{S}}_w\mathbf{w}^1 \leq {\mathbf{w}^1}^T\mathbf{S}_w\mathbf{w}^1=\mu_1^{-1},\\
	&\mu_1^{-1} = {\mathbf{w}^1}^T\mathbf{S}_w\mathbf{w}^1 \leq {\mathbf{w}^2}^T \mathbf{S}_w \mathbf{w}^2 \\
	&\leq {\mathbf{w}^2}^T\tilde{\mathbf{S}}_w \mathbf{w}^2 + \mathbb{D}(s_0r_b)^{-\frac{1}{2}}=\mu_2^{-1} + \mathbb{D}(s_0r_b)^{-\frac{1}{2}}.
	\end{split}
	\end{equation}
	%	\begin{equation}
	%	\begin{split}
	%	&\mu_2^{-1} \leq {\mathbf{w}^2}^T\tilde{\mathbf{S}}_w\mathbf{w}^2 \leq{\mathbf{w}^1}^T\tilde{\mathbf{S}}_w\mathbf{w}^1 \leq {\mathbf{w}^1}^T\mathbf{S}_w\mathbf{w}^1=\mu_1^{-1},\\
	%	&\mu_1^{-1} \leq {\mathbf{w}^1}^T\mathbf{S}_w\mathbf{w}^1 \leq {\mathbf{w}^2}^T \mathbf{S}_w \mathbf{w}^2 \\
	%	&\leq {\mathbf{w}^2}^T\tilde{\mathbf{S}}_w \mathbf{w}^2 + \mathbb{D}(s_0r_b)^{-\frac{1}{2}}=\mu_2^{-1} + \mathbb{D}(s_0r_b)^{-\frac{1}{2}}.
	%	\end{split}
	%	\end{equation}
	Then $\mu_1^{-1} - 2(s_0r_b)^{-\frac{1}{2}}{||\mathbf{X}||}_f^2/(\WWWWWWH{N}\ell) \le \mu_2^{-1} \le \mu_1^{-1}$.
\end{proof}

\vspace{0.1cm}

From the theorem above, we can claim that the largest Fisher score $\mathbf{J}_F^2(\mathbf{w}^2)$ is always greater than or equal to the original one $\mathbf{J}_F^1(\mathbf{w}^1)$ after sketch. From another aspect, the inequalities ``$\mu_1^{-1} - 2(s_0r_b)^{-\frac{1}{2}}{||\mathbf{X}||}_f^2/(\WWWWWWH{N}\ell) \le \mu_2^{-1}\le \mu_1^{-1}$'' means when more rows are set in the sketch matrix $\mathbf{B}$, (i.e. much larger $\ell$ is set), $\mu_2$ becomes $\mu_1$,  and thus SoDA becomes exactly the \whlll{FDA}.

\vspace{0.1cm}

For the multi-class case, we can generalize the above proof below.
\begin{thm}\label{th:multiclass}
	Considering the two criteria in Eq. (\ref{eq:fisher_score_two_reduced}), when the discriminant feature transformation is a $d$-dimensional transformation where $d>1$, we have
	%Consider the following two criteria
	%\begin{equation}\label{eq:multi_fisher}
	%\small
	%\mathbf{J}_F^1(\mathbf{W}) = \frac{trace( \mathbf{W}^T\mathbf{S}_b\mathbf{W})}{trace(\mathbf{W}^T\mathbf{S}_w\mathbf{W})} \ , \ \mathbf{J}_F^2(\mathbf{W}) = \frac {trace(\mathbf{W}^T\mathbf{S}_b\mathbf{W})}{trace(\mathbf{W}^T\tilde{\mathbf{S}}_w\mathbf{W})}.
	%\end{equation}
	%Let the largest Fisher scores be
	%\begin{equation}
	%\underset{{\mathbf{W}_1 \in R^{d\times \ell}}}{\arg max}\mathbf{J}_F^1(\mathbf{W}_1) = \mu_1, \ \underset{\mathbf{W}_2 \in R^{d\times \ell}}{\arg max}\mathbf{J}^2_{F}(\mathbf{W}_2)_{max}  = \mu_2.
	%\end{equation}
	%Then 
	$\mu_1\leq\mu_2$.
\end{thm}

\begin{proof}
	%Under the constraint $trace( \mathbf{W}^T\mathbf{S}_b\mathbf{W})=1$, 
	Note that $\whlll{\mathbf{W}^1}$ and $\whlll{\mathbf{W}^2}$ ($\in R^{d\times k}$) make the two criteria minimized in Eq. (\ref{eq:fisher_score_two_reduced}), respectively.
	%we have
	%. We have $\mathbf{W}_1$ and $\mathbf{W}_2$ are matrices to minimize $trace(\mathbf{W}^T\mathbf{S}_w\mathbf{W})$ and $trace(\mathbf{W}^T\tilde{\mathbf{S}}_w\mathbf{W})$, respectively, and also $trace(\mathbf{W}_1^T\mathbf{S}_w\mathbf{W}_1)$$=\mu_1^{-1}$ and $trace(\mathbf{W}_2^T\tilde{\mathbf{S}}_w\mathbf{W}_1)=\mu_2^{-1}$. \
	Let $\whlll{\mathbf{W}^1}=[\mathbf{w}_1^1,\cdots,\mathbf{w}_k^1]$ and $\whlll{\mathbf{W}^2}=[\mathbf{w}_1^2,\cdots,\mathbf{w}_k^2]$. Since %$trace(\mathbf{A}^T\mathbf{S}_w\mathbf{A})=\sum_i^{\ell} \mathbf{a}_i^T\mathbf{S}_w\mathbf{a}_i$ and
	for any $\mathbf{w}\in R^d$, $\mathbf{w}^T\mathbf{S}_w\mathbf{w}\geq\mathbf{w}^T\tilde{\mathbf{S}}_w\mathbf{w}$ by Theorem \ref{th:SW}, we have
	\begin{equation*}
	\small
	\begin{split}
	\mu_2^{-1}%\sum_i^{\ell} {\mathbf{a}_i^2}^T\tilde{\mathbf{S}}_w\mathbf{a}_i^2
	&=tr({\whlll{\mathbf{W}^2}}^T\tilde{\mathbf{S}}_w\whlll{\mathbf{W}^2})\leq tr(\whlll{{\mathbf{W}^1}^T}\tilde{\mathbf{S}}_w\whlll{\mathbf{W}^1})\\
	&=\sum_{i=1}^{k} {\mathbf{w}_i^1}^T\tilde{\mathbf{S}}_w\mathbf{w}_i^1 \leq\sum_{i=1}^{k} {\mathbf{w}_i^1}^T\mathbf{S}_w\mathbf{w}_i^1\\
	&=tr(\whlll{{\mathbf{W}^1}^T}\mathbf{S}_w\whlll{\mathbf{W}^1})=\mu_1^{-1}.
	\end{split}
	\end{equation*}
	Hence, the theorem is proved.
\end{proof}

%\vspace{0.2cm}
\begin{figure*}[!htp]
	\begin{center}
		%		\label{fig:CMC_jstl_all}
		%		\fbox{\rule{0pt}{1.7in}\rule{0.9\linewidth}{0pt}}
		%		%\includegraphics[width=\linewidth]{figure/ImageDemo.pdf}
		%		\centering\small\caption{/***weihong: add CMC curve on three datasets using jstl***/}
		{\scriptsize
			\subfigure[] % caption for subfigure
			{
				\label{fig:Fisher_JLH_Market}
				\includegraphics[height=0.23\linewidth]{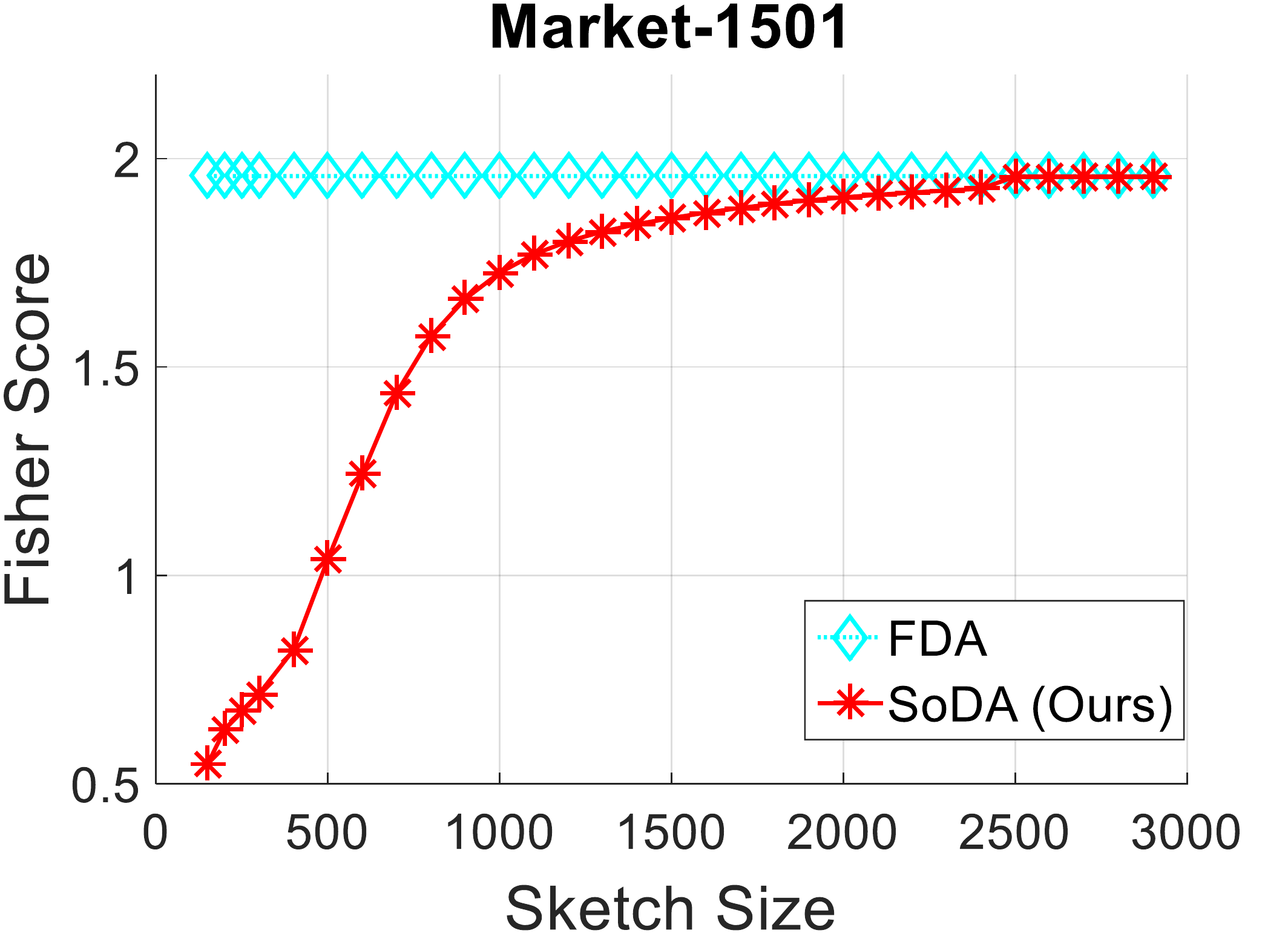}
			}
			\hskip -0.3cm
			\subfigure[] % caption for subfigure
			{
				\label{fig:Fisher_JLH_SYSU}
				\includegraphics[height=0.23\linewidth]{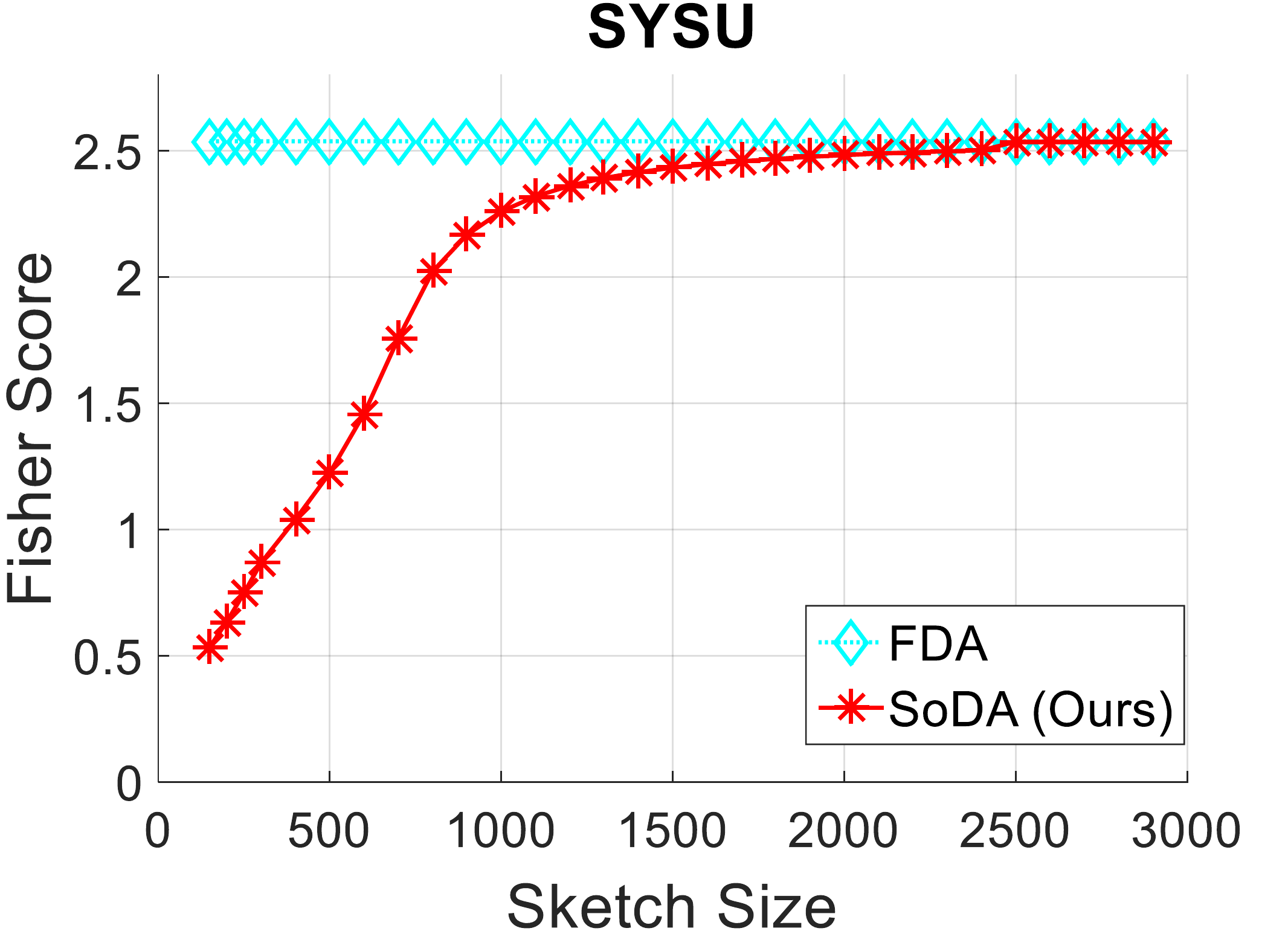}
			}
			\hskip -0.3cm
			\subfigure[] % caption for subfigure
			{
				\label{fig:Fisher_JLH_ExMarket}
				\includegraphics[height=0.23\linewidth]{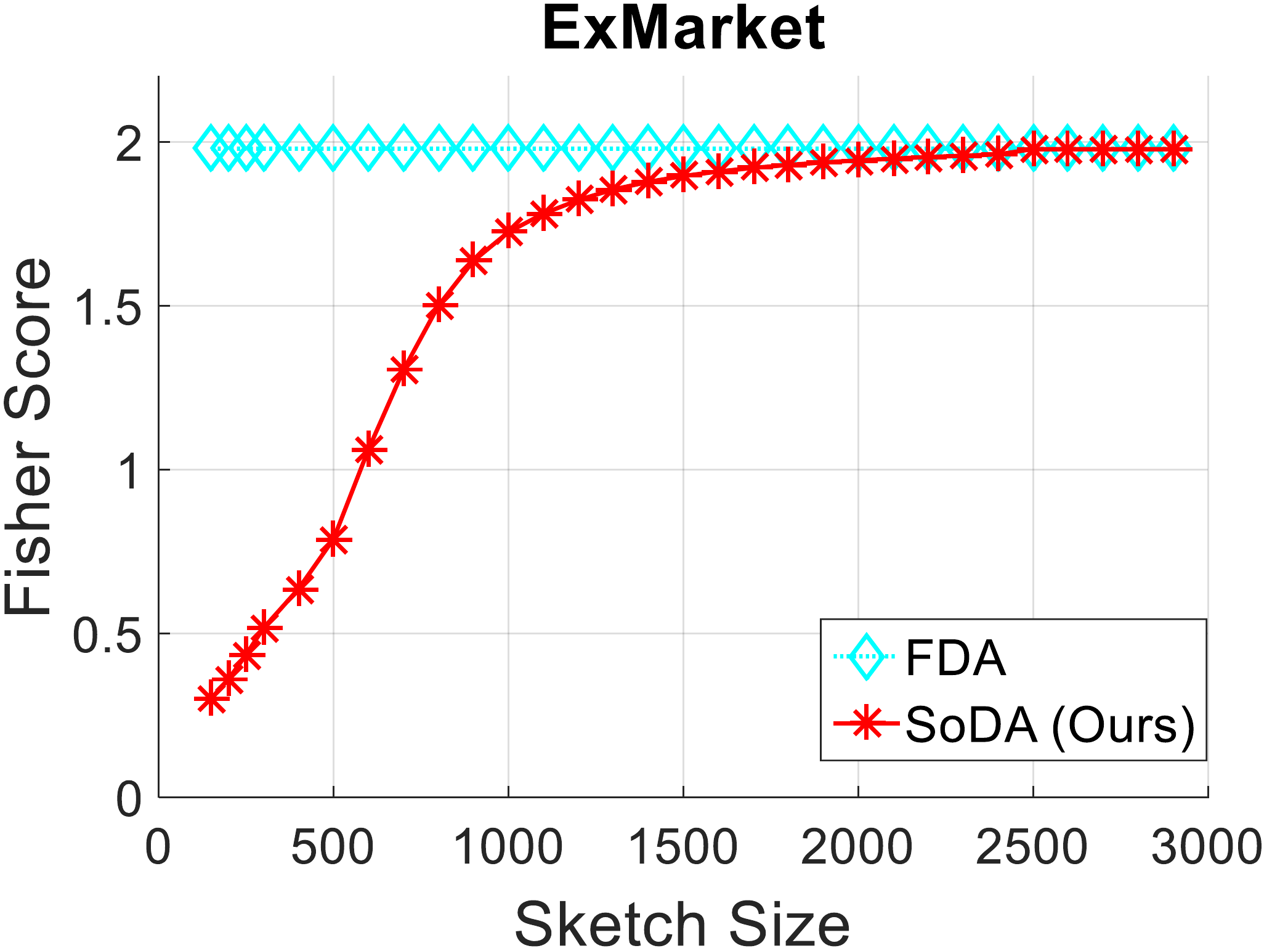}
			}\vskip -0.3cm
			\centering\small\caption{Fisher Score comparison on three datasets using \whll{JLH} feature. (Best viewed in color).}
			\label{fig:FisherScore}
			%			\label{fig:time_methods}
		}
	\end{center}
\end{figure*}

\begin{table*}[htbp]
	\centering
	\caption{Comparision among different online/incremental approaches}
	\resizebox{!}{2.7cm}
	{
		\begin{tabular}{c|c|c|c|c|c|c|c}
			\hline
			Approaches   & IFDA \cite{kim2011incremental} & Pang's IFDA \cite{pang2005incremental} & IDR/QR \cite{ye2005idr}  & OL-IDM  \cite{OLDML_sun2014online} & Wang et al. \cite{hitl_wang2016human} & Martinel et al. \cite{hitl_martinel2016temporal} & {\bf \WWWWWWH{SoDA (Ours)}} \\
			\hline
			\hline
			Save within-class  & \multirow{2}[2]{*}{\Checkmark} & \multirow{2}[2]{*}{\Checkmark} & \multirow{2}[2]{*}{\Checkmark}  & \multirow{2}[2]{*}{-\ -} & \multirow{2}[2]{*}{-\ -} & \multirow{2}[2]{*}{-\ -} & \multirow{2}[2]{*}{\XSolidBrush} \\
			scatter matrix?  &       &       &       &       &       &       &  \\
			\hline
			Save between-class  & \multirow{2}[2]{*}{\Checkmark} & \multirow{2}[2]{*}{\Checkmark} & \multirow{2}[2]{*}{\XSolidBrush}  & \multirow{2}[2]{*}{-\ -} & \multirow{2}[2]{*}{-\ -} & \multirow{2}[2]{*}{-\ -} & \multirow{2}[2]{*}{\XSolidBrush} \\
			scatter matrix?    &       &       &       &       &       &       &  \\
			\hline
			Is \WWH{an} one-pass  & \multirow{2}[2]{*}{\XSolidBrush} & \multirow{2}[2]{*}{\Checkmark} & \multirow{2}[2]{*}{\Checkmark}  & \multirow{2}[2]{*}{\Checkmark} & \multirow{2}[2]{*}{\XSolidBrush} & \multirow{2}[2]{*}{\XSolidBrush} & \multirow{2}[2]{*}{\Checkmark} \\
			algorithm? &       &       &       &       &       &       &  \\
			\hline
			Human feedback   & \XSolidBrush    & \XSolidBrush    & \XSolidBrush     & \XSolidBrush   & \Checkmark   & \Checkmark   & \XSolidBrush \\
			\hline
			Can the model be   & \multirow{2}[2]{*}{\Checkmark} & \multirow{2}[2]{*}{\Checkmark} & \multirow{2}[2]{*}{\Checkmark}  & \multirow{2}[2]{*}{\Checkmark} & \multirow{2}[2]{*}{\XSolidBrush} & \multirow{2}[2]{*}{\XSolidBrush} & \multirow{2}[2]{*}{\Checkmark} \\
			trained on streaming data?  &       &       &       &       &       &       &  \\
			\hline
			%			Can the model be  & \multirow{2}[2]{*}{\Checkmark} & \multirow{2}[2]{*}{\Checkmark} & \multirow{2}[2]{*}{\Checkmark}  & \multirow{2}[2]{*}{\Checkmark} & \multirow{2}[2]{*}{\Checkmark} & \multirow{2}[2]{*}{\Checkmark} & \multirow{2}[2]{*}{\Checkmark} \\
			%
			%			applied on unseen data?   &       &       &       &       &       &       &  \\
			%
			%\hline
			Is the model \WWH{embedded}  & \multirow{2}[2]{*}{\XSolidBrush} & \multirow{2}[2]{*}{\XSolidBrush} & \multirow{2}[2]{*}{\XSolidBrush}  & \multirow{2}[2]{*}{\XSolidBrush} & \multirow{2}[2]{*}{\XSolidBrush} & \multirow{2}[2]{*}{\XSolidBrush} & \multirow{2}[2]{*}{\Checkmark} \\
			\WWH{with} dimension reduction?  &       &       &       &       &       &       &  \\
			\hline
			time  & \multirow{2}[2]{*}{$\mathcal{O}(d^3)$} & \multirow{2}[2]{*}{$\mathcal{O}(nd^2)$} & \multirow{2}[2]{*}{$\mathcal{O}(ndc)$} & \multirow{2}[2]{*}{-\ -} & \multirow{2}[2]{*}{-\ -} & \multirow{2}[2]{*}{-\ -} & \multirow{2}[2]{*}{$\mathcal{O}({min(\ell,d)}^2 max(\ell,d))$ } \\
			complexity &       &       &       &       &       &       &  \\
			\hline
			space & \multirow{2}[2]{*}{$\mathcal{O}(d^2)$} & \multirow{2}[2]{*}{$\mathcal{O}(d^2)$} & \multirow{2}[2]{*}{$\mathcal{O}(d^2)$} & \multirow{2}[2]{*}{-\ -} & \multirow{2}[2]{*}{-\ -} & \multirow{2}[2]{*}{-\ -} & \multirow{2}[2]{*}{$\mathcal{O}((\ell + k + C)d)$} \\
			complexity &       &       &       &       &       &       &  \\
			\hline
		\end{tabular}%
	}
	\label{tab:comparisonapproaches}%
\end{table*}%

\subsection{How Does the Projection Learned by SoDA Optimize the Original Fisher Criterion Approximately?}

In the above, we analyze the quotient values between $\frac{tr( \mathbf{W}^T\mathbf{S}_b\mathbf{W})}{tr(\mathbf{W}^T\mathbf{S}_w\mathbf{W})}$ and $ \frac {tr(\mathbf{W}^T\mathbf{S}_b\mathbf{W})}{tr(\mathbf{W}^T\tilde{\mathbf{S}}_w\mathbf{W})}$. However, in SoDA, our within-class covariance matrix is estimated by sketch and is not the exact within-class covariance matrix. In the following, we will present the effect of the learned discriminant component using SoDA on minimizing the grouth-truth within-class covariance. For this purpose, the following theorems are presented.

\begin{thm}\label{th:lda_SoDA_w}
	For any $\mathbf{w} \in \mathcal{Q}=\{ \mathbf{w} \in R^{d}| \mathbf{w}^T\mathbf{w} =1  \}$, we have
	\begin{equation}\label{eq:multi_fisher}
	\small
	\mathbf{w}^T\tilde{\mathbf{S}}_w\mathbf{w} \leq \mathbf{w}^T \mathbf{S}_w \mathbf{w} \leq \mathbf{w}^T\tilde{\mathbf{S}}_w\mathbf{w}+ \frac{2}{\WWWWWWH{N}} ||\mathbf{X}||_f^2/\ell.
	\end{equation}
\end{thm}

\begin{proof}
	While the inequality $\mathbf{w}^T\tilde{\mathbf{S}}_w\mathbf{w} \leq \mathbf{w}^T \mathbf{S}_w \mathbf{w}$ is obvious by using Theorem \ref{th:SW}, we focus on the latter one.
	Since $\mathbf{S}_w = \tilde{\mathbf{S}}_w+ (\mathbf{X}^T\mathbf{X} - \mathbf{B}^T\mathbf{B})/\WWWWWWH{N}$ in Eq. (\ref{eq:SubSw}), by applying Eq. (\ref{eq:FD2}), we have
	$\mathbf{w}^T \mathbf{S}_w \mathbf{w} = \mathbf{w}^T\tilde{\mathbf{S}}_w\mathbf{w}+ \frac{1}{\WWWWWWH{N}} \mathbf{w}^T(\mathbf{X}^T\mathbf{X} - \mathbf{B}^T\mathbf{B})\mathbf{w}$
	$\leq \mathbf{w}^T\tilde{\mathbf{S}}_w\mathbf{w}+ \frac{2}{\WWWWWWH{N}} ||\mathbf{X}||_f^2/\ell$.
\end{proof}

\begin{thm}\label{th:lad_SoDA_quotient}
	Considering the two criteria in Eq. (\ref{eq:fisher_score_two_reduced}), we define
	$\mathcal{D}=\{ \mathbf{W}=[\mathbf{w}_{1},\cdots,\mathbf{w}_{k}] \in R^{d\times k}| tr(\mathbf{W}^T\mathbf{S}_b\mathbf{W}) =1  \}$,
	denote the smallest non-zero singular value of $\mathbf{S}_b$ as $s_0$, and let $r_b=rank(\mathbf{S}_b)$. Suppose the norm of each data vector \WWWH{$\mathbf{x}_{i}$} (i.e. each row of the data matrix $\mathbf{X}\in R^{\WWWWWWH{N} \times d}$) is bounded by $M$, that is \WWWH{$||\mathbf{x}_{i}||_2^2 \leq M.$}
	Then we have
	\begin{equation}
	\frac{1}{\mu_1^{-1} + \frac{2k}{s_0 r_b} M/\ell}   \leq \mathbf{J}_F^1(\whlll{\mathbf{W}^2}) \leq \mu_1.
	\end{equation}
\end{thm}
\begin{proof}
	First, given $\small \whlll{\mathbf{W}^2} \in \mathcal{D} \text{ that minimize } \{\mathbf{J}^2_{F}(\mathbf{W})\}^{-1}.$
	\vspace{0.2cm}
	\begin{equation}%\label{eq:multi_fisher}
	\small
	\begin{split}
	\{\mathbf{J}^1_{F}(\whlll{\mathbf{W}^2})\}^{-1} =& tr(\whlll{{\mathbf{W}^2}^T} \mathbf{S}_w \whlll{\mathbf{W}^2}) \\
	= & \sum_{i=1}^k {\mathbf{w}^2_i}^T  \mathbf{S}_w \mathbf{w}^2_i\\
	= & \sum_{i=1}^k ||\mathbf{w}^2_i||_2^2 \frac{{\mathbf{w}^2_i}^T}{||\mathbf{w}^2_i||_2}  \mathbf{S}_w \frac{\mathbf{w}^2_i}{||\mathbf{w}^2_i||_2}\\
	\leq & \sum_{i=1}^k ||\mathbf{w}^2_i||_2^2 \frac{{\mathbf{w}^2_i}^T}{||\mathbf{w}^2_i||_2}  \tilde{\mathbf{S}}_w \frac{\mathbf{w}^2_i}{||\mathbf{w}^2_i||_2}\\
	& \ \ \ \ \ \ \ \ \ \ + \sum_{i=1}^k \frac{2}{\WWWWWWH{N}} ||\mathbf{w}^2_i||_2^2 ||\mathbf{X}||_f^2/\ell\\
	\leq & \sum_{i=1}^k {\mathbf{w}^2_i}^T  \tilde{\mathbf{S}}_w \mathbf{w}^2_i\\
	&+ \frac{2k}{\WWWWWWH{N}} ||\mathbf{w}^2_i||_2^2 ||\mathbf{X}||_f^2/\ell.\\
	\end{split}
	\end{equation}
	Since  $\small tr(\whlll{{\mathbf{W}^2}^T}\mathbf{S}_b\whlll{\mathbf{W}^2}) =1$, we have $\small {\mathbf{w}^2_i}^T\mathbf{S}_b\mathbf{w}^2_i \leq 1$. \WWH{Here, for convenience, one can further assume  ${\mathbf{w}^2_i}^T\mathbf{S}_b\mathbf{w}^2_i > 0$, otherwise a much tighter bound can be inferred. And thus} $\small s_0 r_b ||\mathbf{w}^2_i||_2^2 \leq 1$. So we have
	\vspace{0.1cm}
	\begin{equation}%\label{eq:multi_fisher}
	\small
	\begin{split}
	\{\mathbf{J}^1_{F}(\whlll{\mathbf{W}^2})\}^{-1} =&tr(\whlll{{\mathbf{W}^2}^T} \mathbf{S}_w \whlll{\mathbf{W}^2}) \\
	\leq & \sum_{i=1}^k {\mathbf{w}^2_i}^T  \tilde{\mathbf{S}}_w \mathbf{w}^2_i+ \frac{2k}{\WWWWWWH{N}} (s_0 r_b)^{-1} ||\mathbf{X}||_f^2/\ell\\
	=& \WWH{\mu_2^{-1}+ \frac{2k}{\WWWWWWH{N}} (s_0 r_b)^{-1} ||\mathbf{X}||_f^2/\ell}.
	\end{split}
	\end{equation}
	Note that \WWH{$\mu_2^{-1}=\sum_{i=1}^k {\mathbf{w}^2_i}^T  \tilde{\mathbf{S}}_w \mathbf{w}^2_i$ since it is assumed that $\small \whlll{\mathbf{W}^2} \in \mathcal{D} \text{ minimizes } \{\mathbf{J}^2_{F}(\mathbf{W})\}^{-1}$. } Thus, under the constraint $tr(\whlll{{\mathbf{W}^2}^T}\mathbf{S}_b\whlll{\mathbf{W}^2})=1$, we have 
	\begin{equation}
	\frac{1}{\mu_2^{-1} + \frac{2k}{\WWWWWWH{N}} (s_0r_b)^{-1} ||\mathbf{X}||_f^2/\ell}   \leq \mathbf{J}_F^1(\whlll{\mathbf{W}^2}) \leq \mu_1,
	\end{equation}
	where the latter equation is obvious since $\whlll{\mathbf{W}^2}$ may not be the optimal projection for mamixizing ${J}_F^1(\mathbf{W})$.
	Finally, since $\mu_1 \leq \mu_2$ and \wss{$||\mathbf{x}_{i}||_2^2 \leq M$ that means} the norm of any data vector \WWWH{$\mathbf{x}_{i}$} (i.e. each row of the data matrix $\mathbf{X}\in R^{\WWWWWWH{N} \times d}$) is bounded by $M$, we have
	\begin{equation}
	\frac{1}{\mu_1^{-1} + \frac{2k}{s_0 r_b} M/\ell}   \leq \mathbf{J}_F^1(\whlll{\mathbf{W}^2}) \leq \mu_1.
	\end{equation}
\end{proof}

\vspace{0.2cm}

\begin{figure*}[!htp]
	\begin{center}
		%		\label{fig:CMC_jstl_all}
		%		\fbox{\rule{0pt}{1.7in}\rule{0.9\linewidth}{0pt}}
		%		%\includegraphics[width=\linewidth]{figure/ImageDemo.pdf}
		%		\centering\small\caption{/***weihong: add CMC curve on three datasets using jstl***/}
		{\scriptsize
			\subfigure[] % caption for subfigure
			{
				\label{fig:CMC_jstl_Market}
				\includegraphics[height=0.23\linewidth]{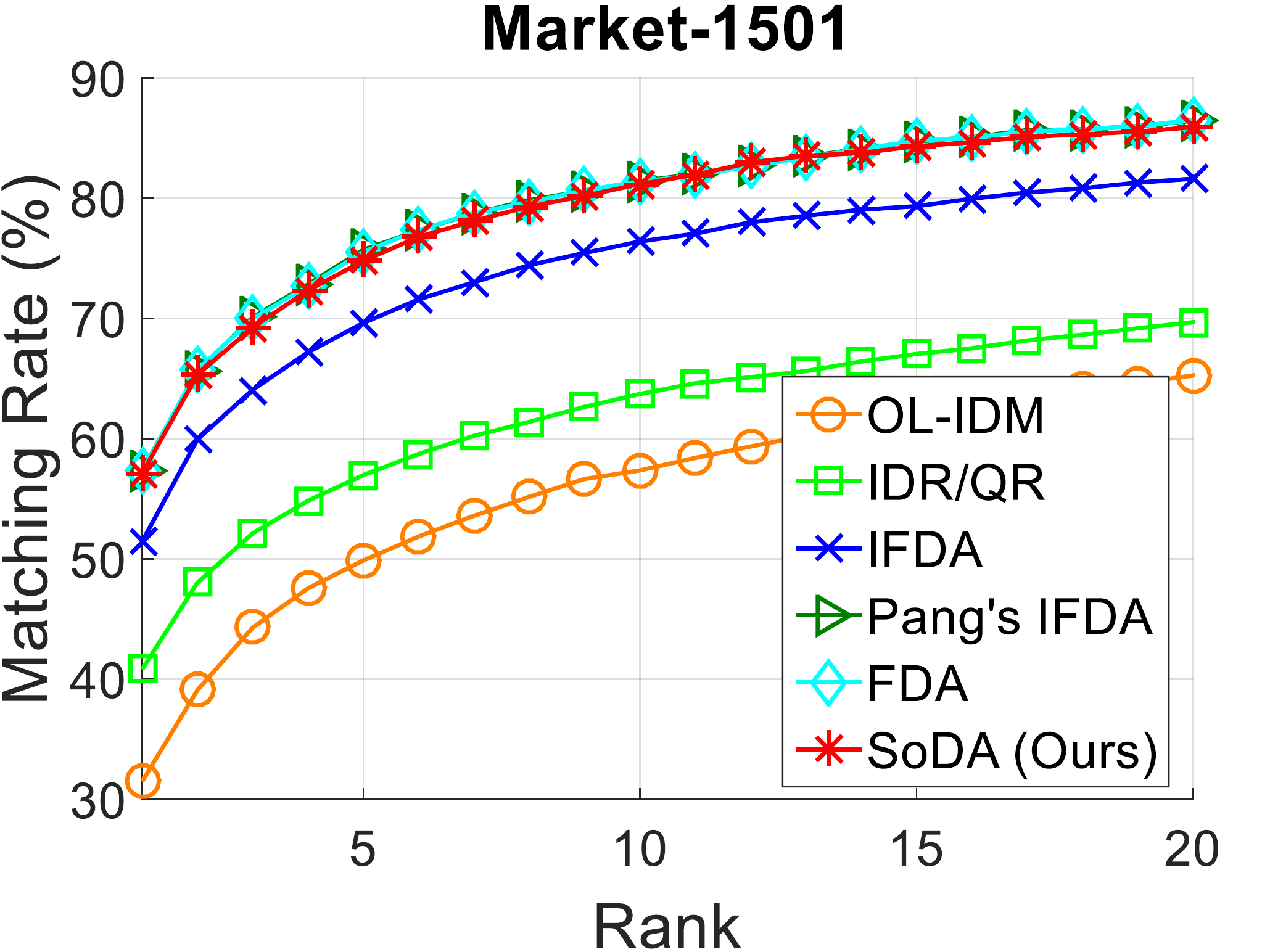}
			}
			\hskip -0.2cm
			\subfigure[] % caption for subfigure
			{
				\label{fig:CMC_jstl_SYSU}
				\includegraphics[height=0.23\linewidth]{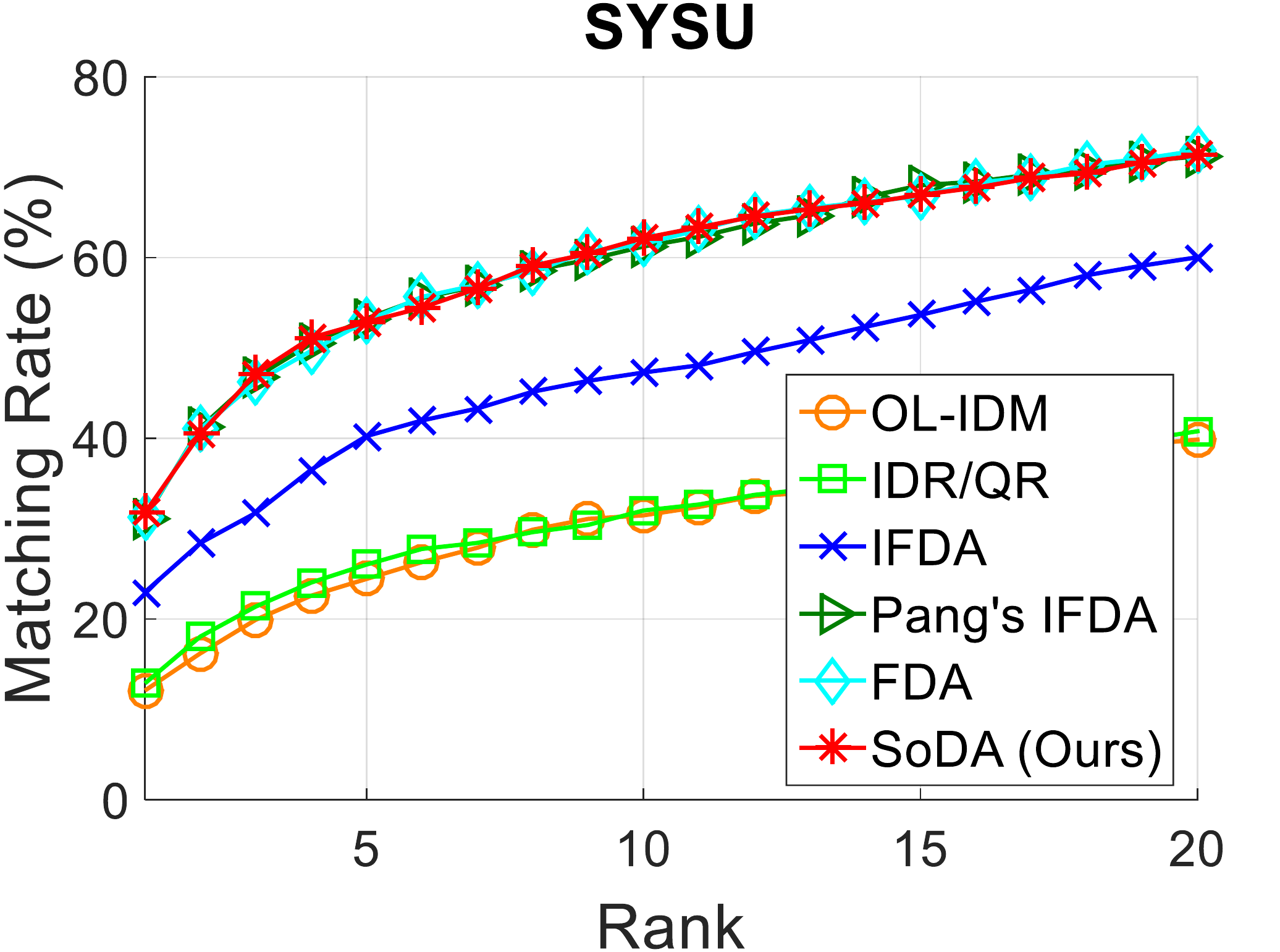}
			}
			\hskip -0.2cm
			\subfigure[] % caption for subfigure
			{
				\label{fig:CMC_jstl_ExMarket}
				\includegraphics[height=0.23\linewidth]{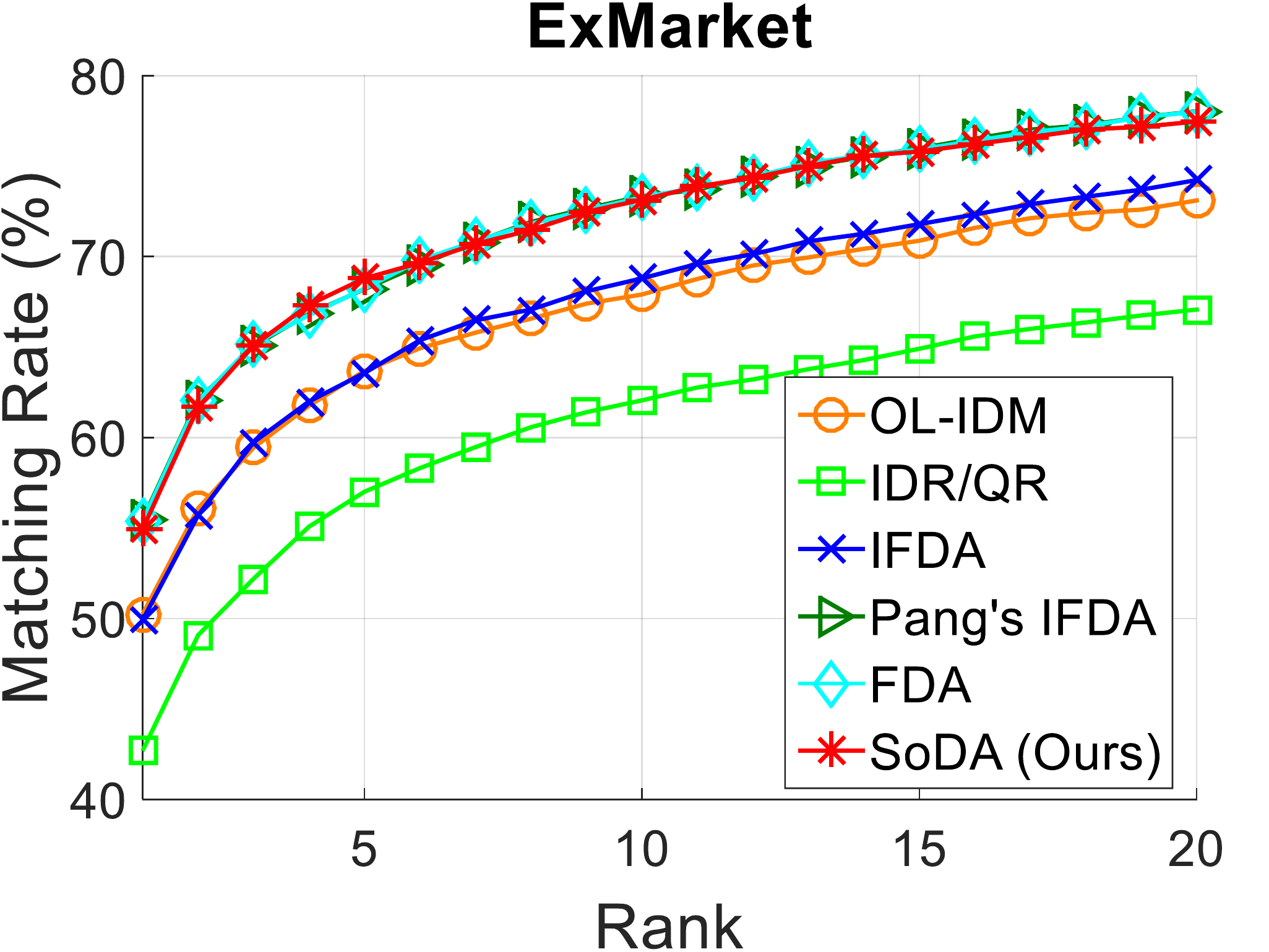}
			}\vskip -0.3cm
			\centering\small\caption{Comparison on three datasets using JSTL feature. (Best viewed in color).}
			\label{fig:CMC_jstl_all}
			%			\label{fig:time_methods}
		}
	\end{center}
\end{figure*}

\begin{figure}[!tp]
	\begin{center}
		%		\label{fig:datasets}
		%		\fbox{\rule{0pt}{1.7in}\rule{0.9\linewidth}{0pt}}
		\includegraphics[height=0.43\linewidth]{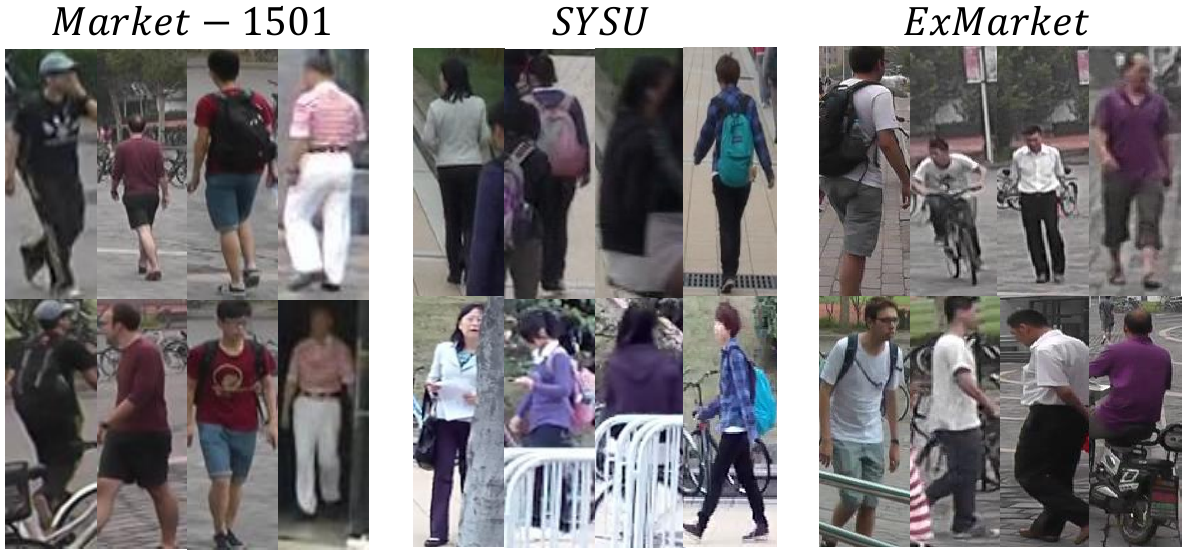}
		\centering\small\caption{Example images from different person re-id datasets. For each dataset, two images in a column correspond to the same person.}
		\label{fig:datasets}
	\end{center}
\end{figure}

\subsection{Discussion}

\subsubsection{SoDA vs. FDA}

The above theorem indicates that 1) the learned transformation by SoDA may not be the optimal one for the FDA directly learned on all observed data since $\mathbf{J}_F^1(\whlll{\mathbf{W}^2}) \leq \mu_1$, which is obvious and reasonable; 2) however, there is a lower bound on $\mathbf{J}_F^1(\whlll{\mathbf{W}^2})$, since $\frac{1}{\mu_1^{-1} + \frac{2k}{s_0 r_b} M/\ell}   \leq \mathbf{J}_F^1(\whlll{\mathbf{W}^2})$; 3) as long as more and more rows are set in the sketch matrix $\mathbf{B}$ used in SoDA, i.e. $\ell$ is larger and larger, $\frac{2k}{s_0} M/\ell \rightarrow 0$ and so that $\mathbf{J}_F^1(\whlll{\mathbf{W}^2}) \approx \mu_1$ in such a case. \WWH{The latter case is reasonable because although the sketch in SoDA enables selecting data variation during the online learning, more data information is kept when a much larger sketch matrix $\mathbf{B}$ is used, and this will be verified in the experiments (see Figure \ref{fig:FisherScore} for example)}. 
%However, keeping more information without selection does not mean better recognition performance, because 

%, and thus it could generate more discriminant feature transformations

%\subsubsection{SoDA vs. Incremental/online FDA}
\subsubsection{\WH{SoDA vs. Incremental/online models}}

%\subsubsection{}

%The performance comparison is reported in Figure \ref{fig:CMC_jstl_all}, Figure \ref{fig:CMC_LOMO_all}, Figure \ref{fig:CMC_HIPHOP_all}, Table \ref{tab:JSTL}, Table \ref{tab:LOMO} and Table \ref{tab:HIPHOP}. It is evident that our SoDA surpasses the state-of-the-art method MLAPG, and gains comparable performance with the Batch FDA and the best related incremental FDA.
%For instance, on Market-1501 dataset with JSTL feature, the rank-1 matching rate of SoDA, Batch FDA and Pang's IFDA is 57.13\%, 57.30\% and 57.36\%.
%It verifies that our SoDA can approximate the Batch FDA, the standard Fisher Discriminant Analysis learned in batch way, while sometimes our SoDA surpasses Batch FDA since SoDA can alleviate noise during sketch.
%More importantly, our SoDA spend much less time and space on learning discriminant projection (e.g., the accumulative time on Market-1501 of SoDA, Batch FDA and Pang's IFDA is /***weihong: fill the AT***/).
%This shows the advantages and effectiveness of our approach in a broad context.

%The comparison of Fisher Score of the two methods on all datasets are plot

In Table \ref{tab:comparisonapproaches}, we compare SoDA with related incremental/online FDA models in details. \WH{A \wss{distinct and} important} characteristic of SoDA \WH{is that it }is able to perform one-pass online learning \WWH{directly only relying on sketch data information. SoDA does not \WH{have} to keep within-class covariance matrix and between-class covariance matrix in memory during online learning, due to embedding sketch processing, which has not been considered for online learning of FDA before.}
%saved 
Moreover, as compared to \wss{the} others, SoDA does not need any extra online learning progress on dimension reduction, which is naturally embedded. Thus the training cost of SoDA is much lighter. 
%/***Jason: does any other method esimate approximate within-class scatter matrix?***/\WH{/***weihong: In my opinion, the most distinct characteristic is the way we estimate the within-class scatter matrix and instead of whether the within-class scatter matrix is approximate or exact. I think the within-class scatter matrix of all online methods are all approximate. So I ***/}

When applied SoDA to person re-id, we perform the comparison with related online person re-id models. \WH{An important distinction is that no extra human feedback is required, and SoDA is able to be applied on streaming data in an one-pass learning manner. In comparison with OL-IDM, SoDA has its merits: 1) dimension reduction is naturally embedded in SoDA; 2) embedding sketch into person re-id model learning is a more efficient and effective way to maintain the main variations of data, which has been verified by our experimental results.}
%/***weihong: All online person re-id methods can work for unseen data.***/
%An important \WH{distinction} is that no extra human feedback is required, and moreover SoDA can be applied to unseen class when deployed.

%While our method is a succinct online version of FDA model, it has its merits compared with other related approaches (Table \ref{tab:comparisonapproaches}). We highlight them as follows:\\
%	1) Compared with other approaches, SoDA only keep two small matrices (the Sketch matrix $\mathbf{B}$ and the basis matrix for dimension reduction) instead of preserving all passed data during the whole training procedure. Therefore it is an one-pass method and can be trained on streaming data lightly.
%	2) In additon to the efficiency of SoDA on processing streaming data, another unique characteristic of SoDA is the dimension reduction has already been embeded in its processing.Thus, the time and space complexity of it is the lowest, which indicates that SoDA is more effcient than any other approaches.
%	2) SoDA does not require any human involvement and can be run for unseen data during testing.
	

\begin{figure*}[!htp]
	\begin{center}
%		\label{fig:CMC_LOMO_all}
		{\scriptsize
			\subfigure[] % caption for subfigure
			{
				\label{fig:CMC_LOMO_Market}
				\includegraphics[height=0.23\linewidth]{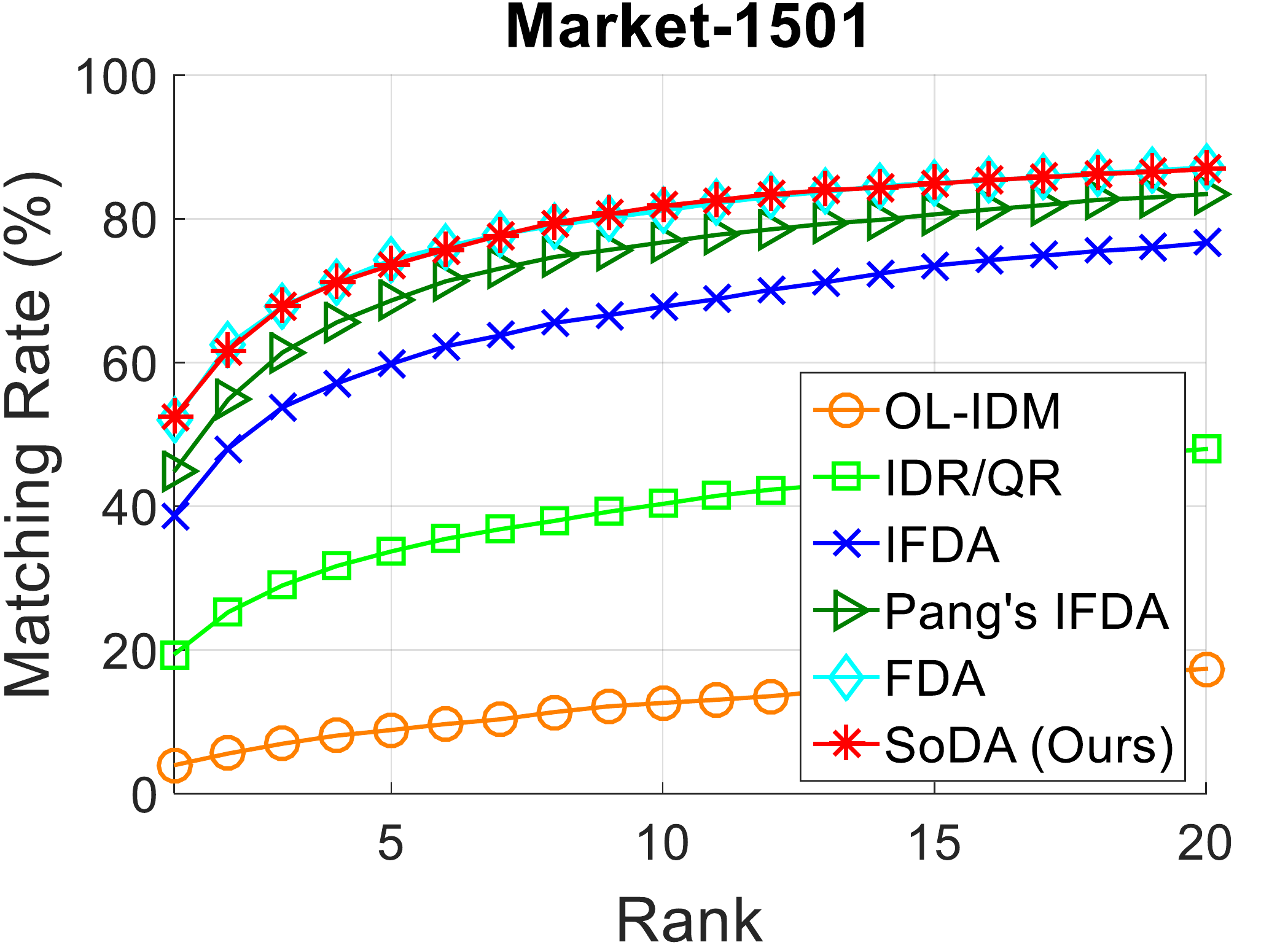}
			}
			\hskip -0.2cm
			\subfigure[] % caption for subfigure
			{
				\label{fig:CMC_LOMO_SYSU}
				\includegraphics[height=0.23\linewidth]{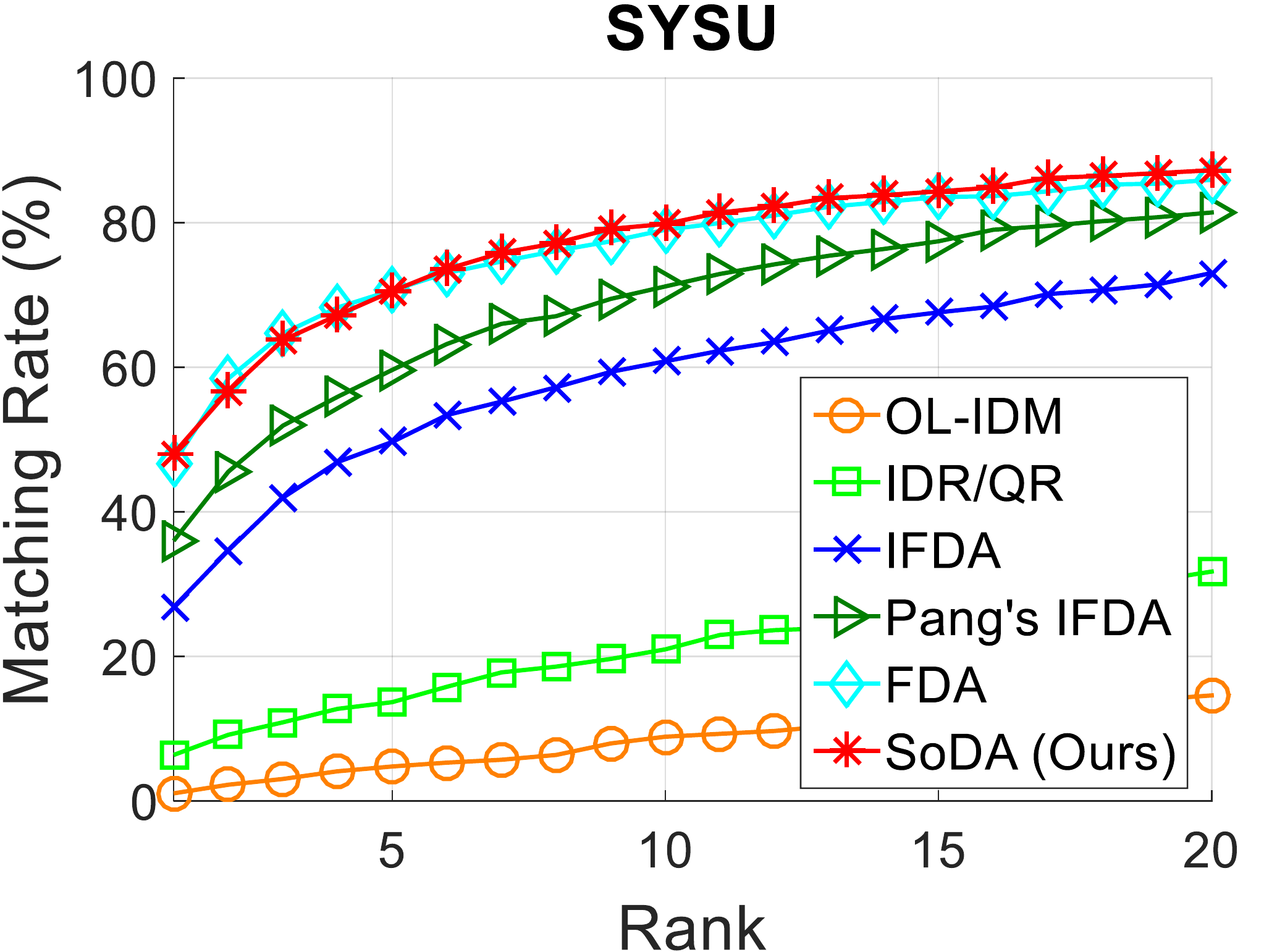}
			}
%			\fbox{\rule{0pt}{1.3in}\rule{0.25\linewidth}{0pt}}
			\hskip -0.2cm
			\subfigure[] % caption for subfigure
			{
				\label{fig:CMC_LOMO_ExMarket}
				\includegraphics[height=0.23\linewidth]{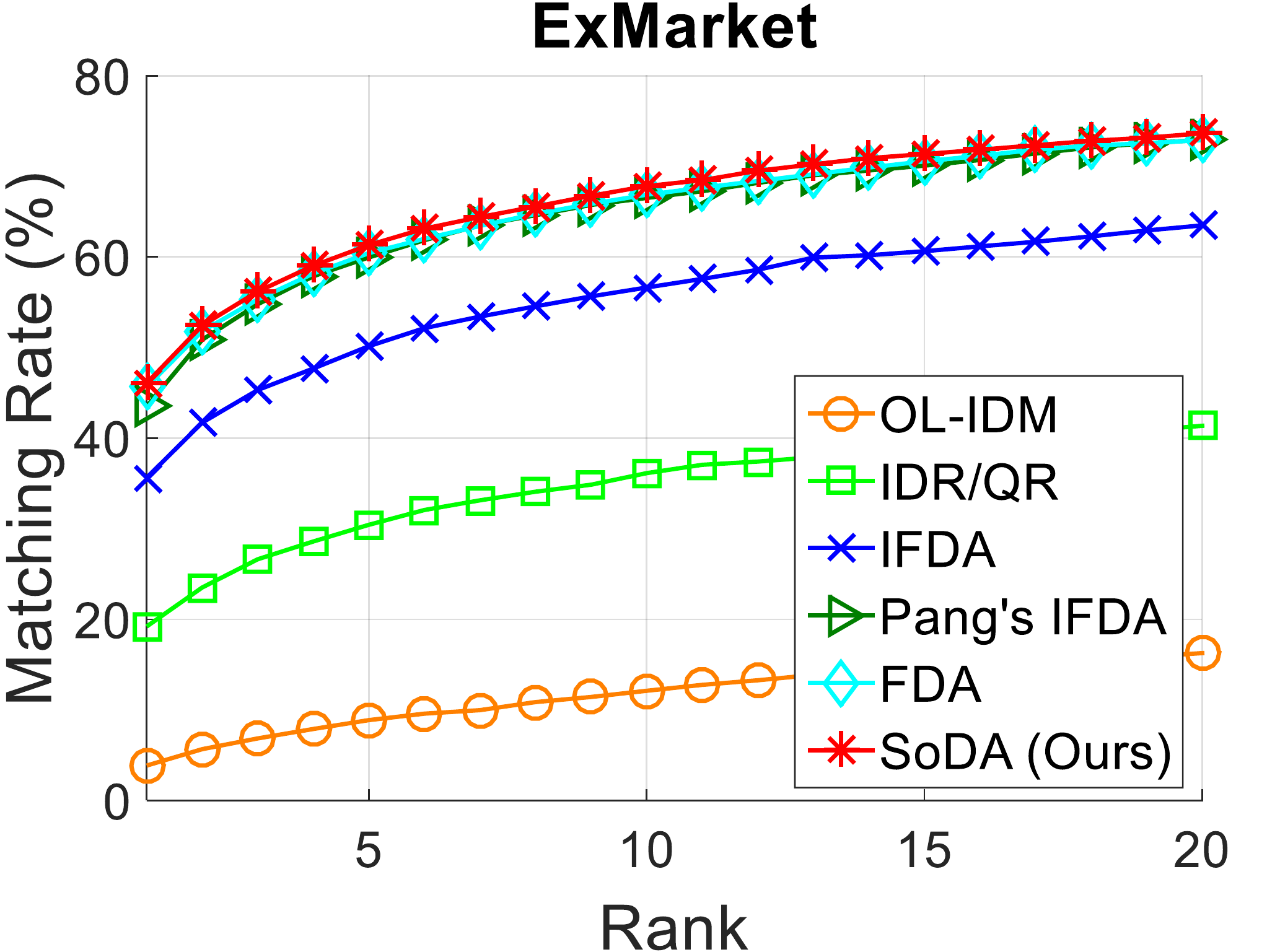}
			}\vskip -0.3cm
			\centering\small\caption{Comparison on three datasets using LOMO feature. (Best viewed in color).}
			\label{fig:CMC_LOMO_all}
%			\label{fig:time_methods}
		}
	\end{center}
\end{figure*}
\begin{figure*}[!htp]
	\begin{center}
%		\label{fig:CMC_HIPHOP_all}
		{\scriptsize
			\subfigure[] % caption for subfigure
			{
				\label{fig:CMC_HIPHOP_Market}
				\includegraphics[height=0.23\linewidth]{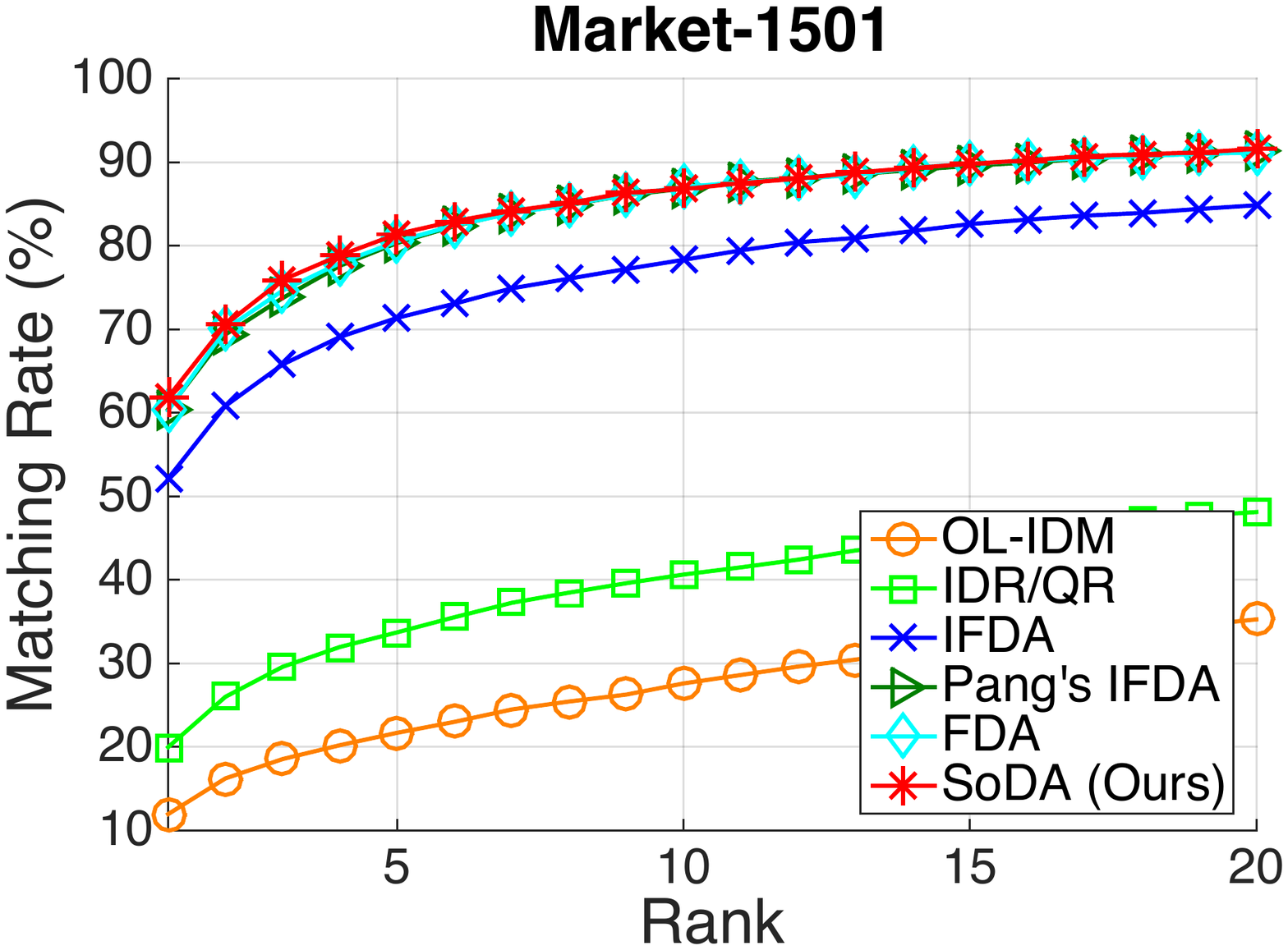}
			}\hskip -0.15cm
			\subfigure[] % caption for subfigure
			{
				\label{fig:CMC_HIPHOP_SYSU}
				\includegraphics[height=0.23\linewidth]{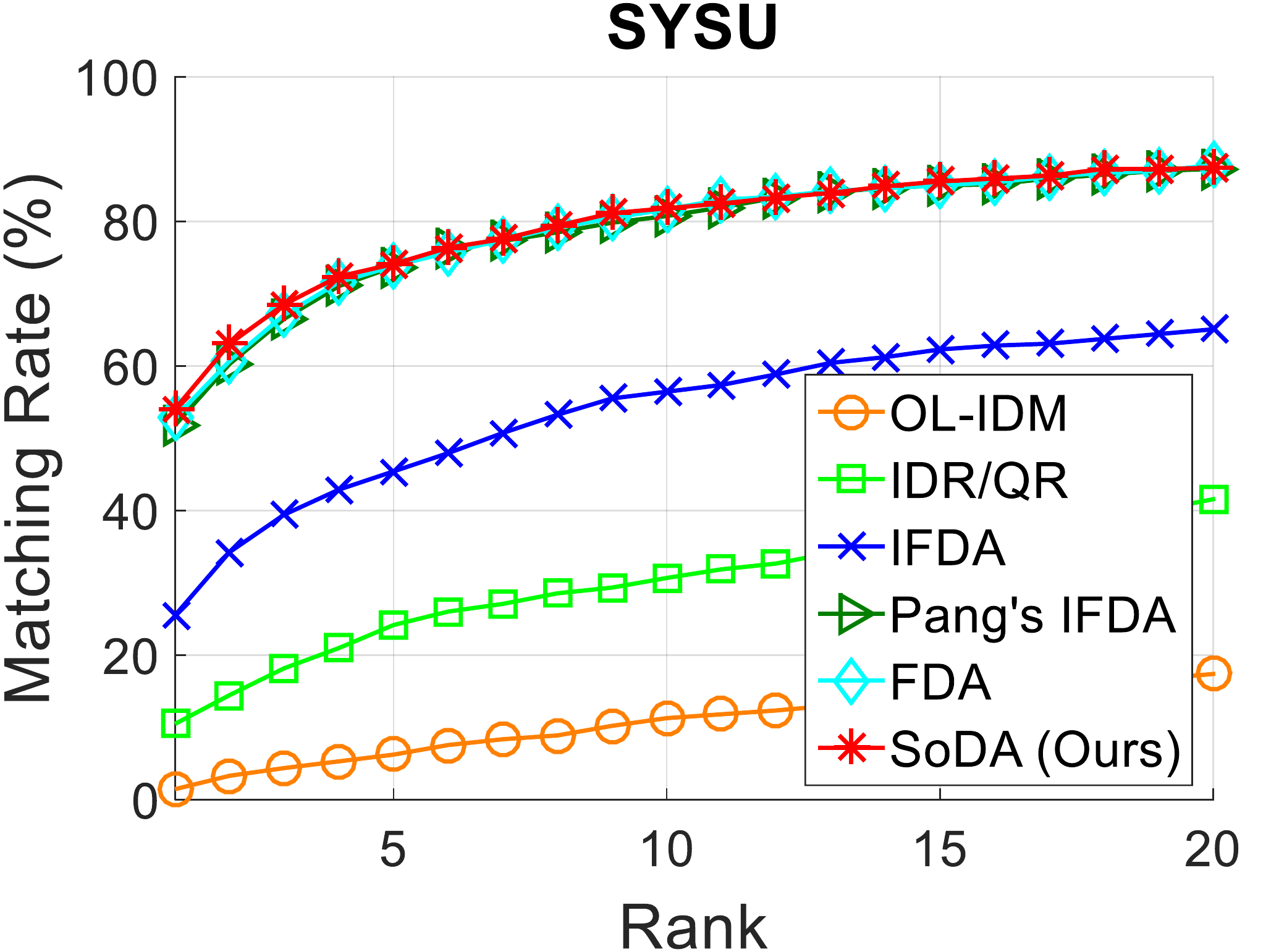}
			}
			%			\hskip -0.5cm
			\hskip -0.3cm
			\subfigure[] % caption for subfigure
			{
				\label{fig:CMC_HIPHOP_ExMarket}
				\includegraphics[height=0.23\linewidth]{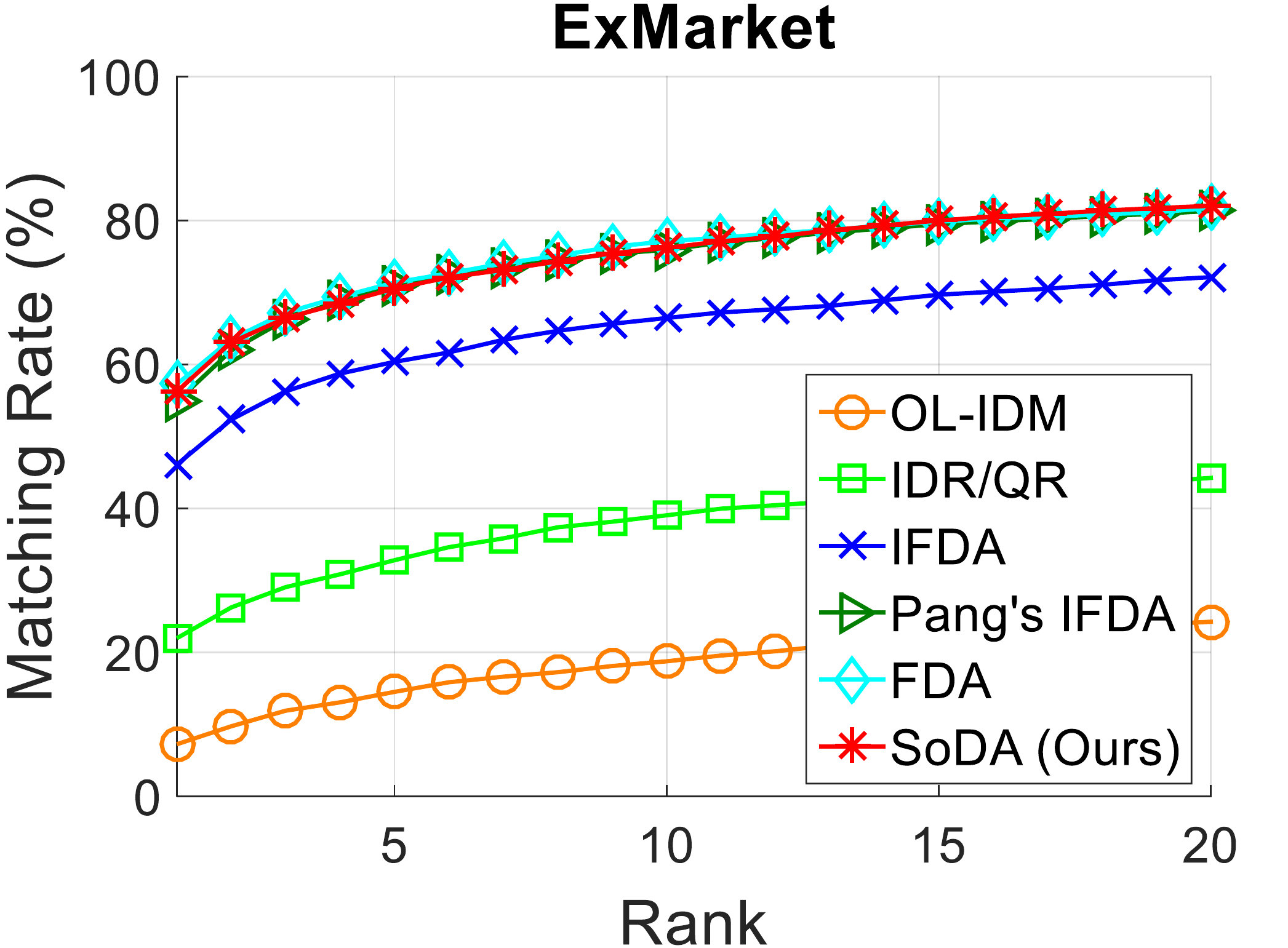}
			}
%			\fbox{\rule{0pt}{1.3in}\rule{0.25\linewidth}{0pt}}
			\vskip -0.3cm
			\centering\small\caption{Comparison on three datasets using HIPHOP feature. (Best viewed in color).}
			\label{fig:CMC_HIPHOP_all}
%			\label{fig:time_methods}
		}
	\end{center}
\end{figure*}

\begin{figure*}[!htp]
	\begin{center}
		%		\label{fig:CMC_HIPHOP_all}
		{\scriptsize
			\subfigure[] % caption for subfigure
			{
				\label{fig:CMC_JLH_Market}
				\includegraphics[height=0.23\linewidth]{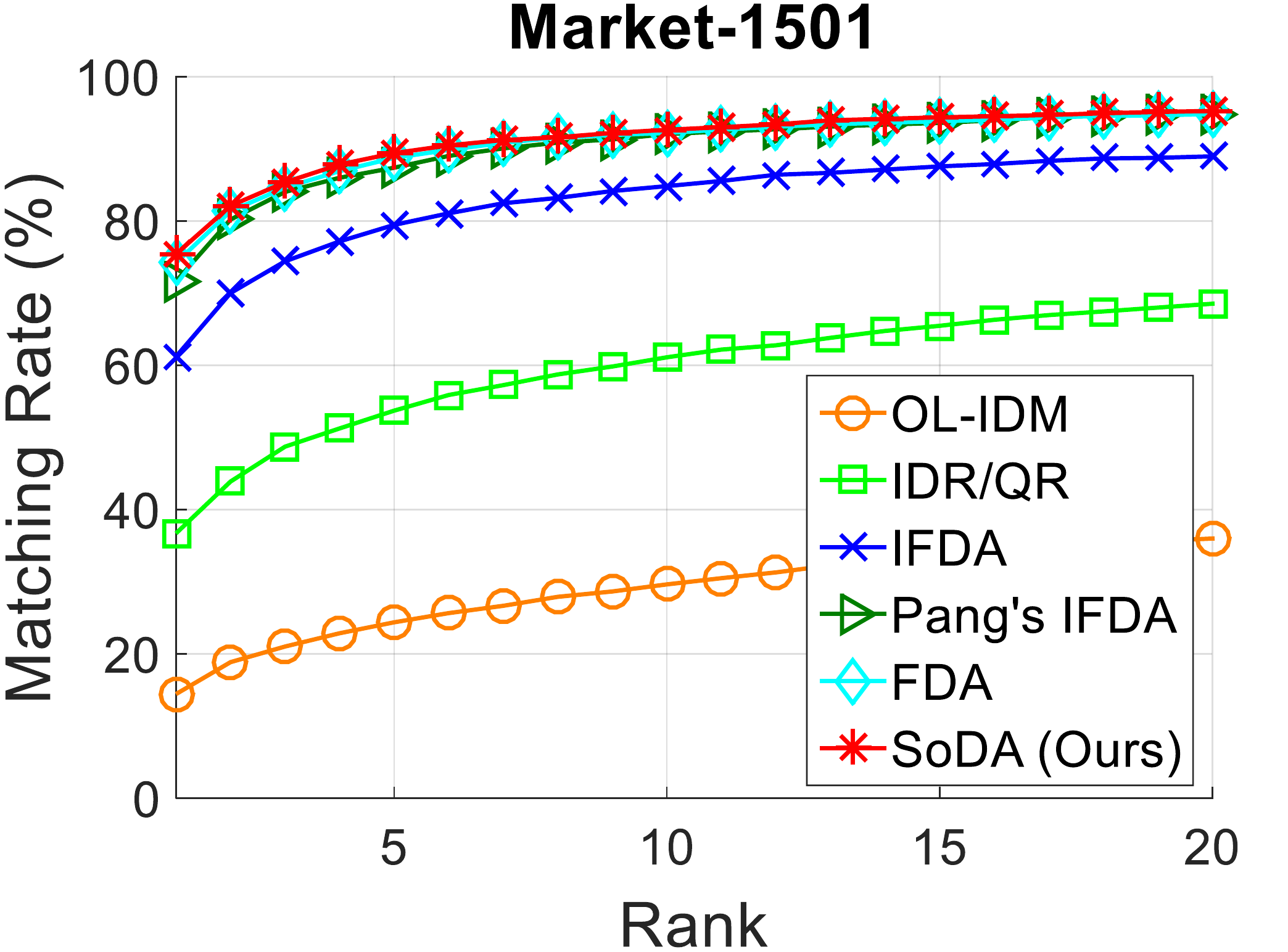}
			}\hskip -0.2cm
			\subfigure[] % caption for subfigure
			{
				\label{fig:CMC_JLH_SYSU}
				\includegraphics[height=0.23\linewidth]{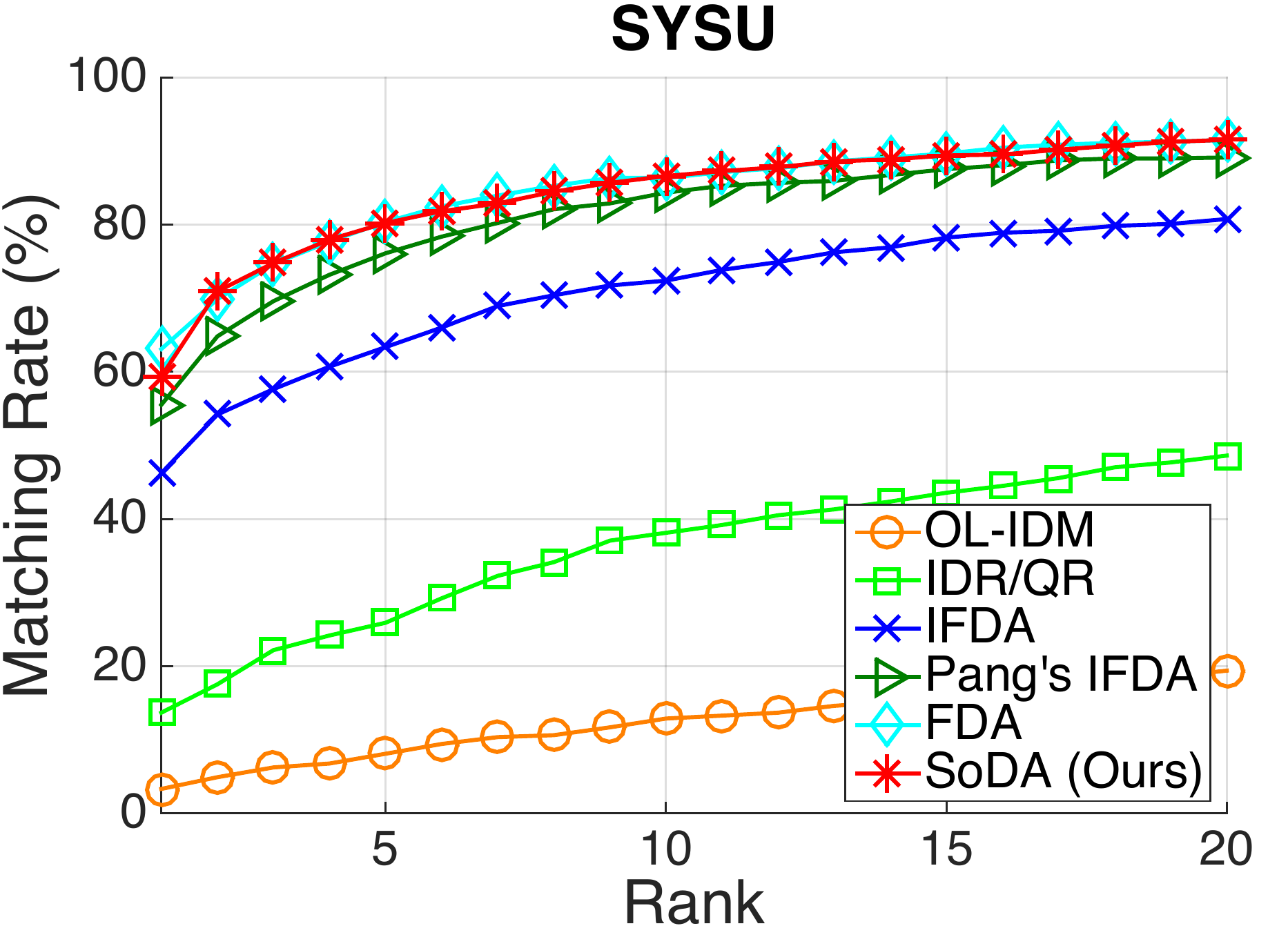}
			}
			%			\hskip -0.5cm
			\hskip -0.25cm
			\subfigure[] % caption for subfigure
			{
				\label{fig:CMC_JLH_ExMarket}
				\includegraphics[height=0.23\linewidth]{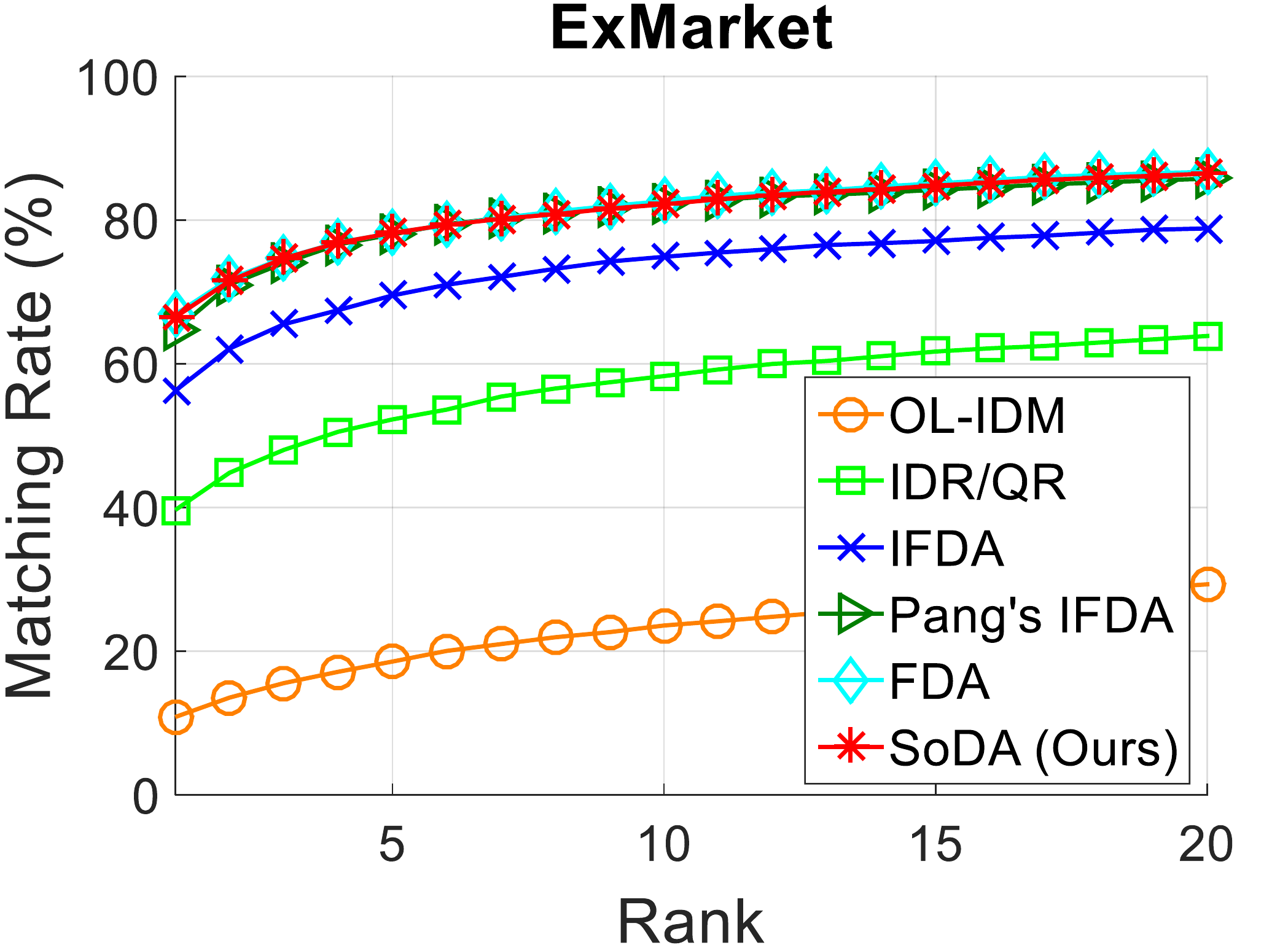}
			}
			%			\fbox{\rule{0pt}{1.3in}\rule{0.25\linewidth}{0pt}}
			\vskip -0.3cm
			\centering\small\caption{Comparison on three datasets using JLH feature. (Best viewed in color).}
			\label{fig:CMC_AllDatasets_jstl_LOMO_HIPHOP}
			%			\label{fig:time_methods}
		}
	\end{center}
\end{figure*}

\section{Experiments}\label{section:Experiments}

\subsection{Datasets and Evaluation Settings}
\subsubsection{Datasets}

We extensively evaluated the proposed approach on three large person re-id benchmarks: 
Market-1501, SYSU, and ExMarket.
\begin{itemize}
\item \WWH{\textbf{Market-1501}} dataset \whll{\cite{reiddata_zheng2015scalable}} contains person images collected in front of a campus supermarket at a University. It consists of 32,643 person images of 1,501 identities. 
\item \WWH{\textbf{SYSU}} dataset contains totally 48,892 images of 502 pedestrians captured by two cameras. Similar to \whll{\cite{TCSVT_chen2016asymmetric}}, we randomly selected 251 identities from two views as training set which contains 12308 images. And we randomly selected three images \wss{of each person from} the rest 251 identities of both cameras to form \wss{the} testing set, where the 753 images of the first camera were used as query images.
\item \WWH{\textbf{ExMarket}} dataset was formed by combining \whlll{the MARS dataset} \whll{\cite{MARS_zheng2016mars}} and Market-1501 dataset. 
MARS was formed as a video dataset for person re-identification. All the identities from MARS are of a
subset of those from Market. 
More specifically, for each identity, we extracted one frame for each five consecutive frames firstly and combined images extracted from MARS and the ones from Market-1501 of the same person. Therefore, ExMarket contains 237147 images of 1501 identities, % and is split into three sets, 
the largest population size among the three benchmark datasets tested.
\end{itemize}

\subsubsection{Features}
In this work, we conducted the evaluation based on four types of feature for evaluation: 1) JSTL, 2) LOMO, 3) HIPHOP, 4) JSTL + LOMO + HIPHOP (JLH).
\begin{itemize}
\item JSTL is a kind of low-dimensional deep feature representation ($\mathcal{R}^{256}$) extracted by a deep convolutional network \whlll{\cite{reid_jstl_xiao2016learning}};
\item LOMO is an effective handcraft feature proposed for person re-id in \whll{\cite{reid_liao2015person}}, \WWH{ and} it is a 26960-dimensional vector;
\item HIPHOP is another recently proposed person re-id feature ($\mathcal{R}^{84096}$) \whll{\cite{yingcongpami_chen2017person}} that extracts more view invariant histogram features from shallow layers of a convolution network. 
\end{itemize}

In addition, since person re-id can benefit from using multiple different types of appearance features as shown in \whllll{\cite{yingcongpami_chen2017person,reid_MF_farenzena2010person,reid_MF_gray2008viewpoint,reid_MF_zhang2016learning,reid_MF_zheng2015query}}.
we concatenated JSTL, LOMO and HIPHOP as a high dimensional feature ($\mathcal{R}^{111312}$), named \textbf{JLH} in this work for convenience of description.
On all datasets, we report experimental results of SoDA using the concatenated feature in Table \ref{tab:SoDA_jstl_LOMO_HIPHOP}.
\WWWWWH{Since LOMO, HIPHOP, and JLH are of high dimension, for all methods except SoDA, we first reduced their feature dimension of the three types of feature to 2000, 2000 and 2500, respectively, on all datasets.  For SoDA, we set the sketch size ($\ell$) to the (reduced) feature dimension menthioned above on all datasets.}

%The results indicate that SoDA can acquire comparable performance with state-of-the-art offline person re-id model and related online/incremental methods with much less time and space cost.

\subsubsection{Evaluation protocol}
 On all datasets, we followed the standard evaluation settings on person re-identification, i.e. images of half of the persons \WWH{were} used for training and images of the rest half \WWH{were} used for testing, so that there is no overlap in persons between training and testing sets. 
 %On testing set, only image of each person is used to the gallery set and the rest are used to form the probe set. 
More specifically, on Market-1501 dataset, we used the standard training (12936 images of 750 people) and testing (19732 images of 751 people) sets provided in \whll{\cite{reiddata_zheng2015scalable}}.
On SYSU dataset, similar to \whll{\cite{TCSVT_chen2016asymmetric}}, we randomly \whll{picked} all images of the selected 251 identities from two views to form the training set which contains 12308 images, and we randomly picked 3 images of each pedestrian of the rest 251 identities in each view for forming the gallery and query sets for testing. On ExMarket dataset, we conducted the same identity split as the Market-1501 dataset.
The training set contains 112351 images, and the testing set contains 124796 images, among which 3363 images are considered as query images and the rest are considered as gallery images.

%We further combined \textbf{the MARS dataset} \whll{\cite{MARS_zheng2016mars}} with
%Market. MARS is a video-based re-id dataset which contains
%20,715 tracklets of 1,261 pedestrians.
%We then took 20\% frames (each one in every five successive frames) from the tracklets
%and combined them with Market to obtain an extended version of Market (\textbf{ExMarket}). There are 237,256 images
%in ExMarket in total, and 112,460 images of them are of training set.  For evaluation, 

%We also extend the Market-1501 dataset by ExMarket benchmark, , and images of each identity consists of images extracted from Market-1501 and MARS  . 

%In this evaluation, we considered two scenarios: Person re-id with 

%More importantly, on ExMarket dataset

\begin{table*}[!htbp]
	\centering
	
	\caption{\whlll{Comparison with FDA on all benchmarks.}}
	\vspace{-0.15cm}
	\resizebox{!}{1.3cm}
	{
		\begin{tabular}{c|c||c|c|c|c|c||c|c|c|c|c||c|c|c|c|c||c|c|c|c|c}
			\hline 
			\multicolumn{2}{c||}{Feature} & \multicolumn{5}{c||}{JSTL} & \multicolumn{5}{c||}{LOMO} & \multicolumn{5}{c||}{HIPHOP} & \multicolumn{5}{c}{JLH} \\
			\hline
			Dataset & Method & rank-1 & rank-5 & rank-10 & rank-20 & mAP & rank-1 & rank-5 & rank-10 & rank-20 & mAP & rank-1 & rank-5 & rank-10 & rank-20 & mAP & rank-1 & rank-5 & rank-10 & rank-20 & mAP \\
			\hline
			\hline
			Market & FDA & 57.30 & 75.53 & 81.38 & 86.49 & 28.57 & 51.90 & 74.26 & 81.12 & 87.14 & 23.60 & 60.27 & 80.52 & 87.05 & 91.18 & 31.45 & 74.20 & 88.75 & 92.19 & 94.80 & 49.01 \\
			%		\hline
			-1501 & SoDA & 57.13 & 74.79 & 81.18 & 85.90 & 28.25 & 52.41 & 73.37 & 81.38 & 87.17 & 23.58 & \WWWWWH{61.88} & 81.41 & 86.70 & 91.60 & 33.39 & \WWWWWH{75.27} & 89.28 & 92.70 & 95.22 & 49.82 \\
			\hline
			\hline
			\multirow{2}[2]{*}{SYSU} & FDA & \whllll{31.21} & \whllll{52.99} & \whllll{61.49} & \whllll{71.85} & \whllll{25.86} & 46.61 & 70.78 & 79.42 & 86.19 & 41.81 & 52.86 & 73.84 & 81.67 & 87.78 & 48.20 & 63.08 & 80.35 & 86.32 & 91.50 & 56.82 \\
			& SoDA & \whllll{31.74} & \whllll{52.86} & \whllll{62.15} & \whllll{71.31} & \whllll{26.04} & \WWWWWH{47.81} & 70.39 & 78.75 & 86.72 & 41.69 & 53.12 & 73.97 & 80.88 & 87.25 & 48.48 & 64.81 & 80.74& 87.25 & 91.77 & 59.82 \\
			\hline
			\hline
			Ex- & FDA & 53.89 & 68.11 & 73.13 & 77.97 & 22.71 & 45.64 & 60.42 & 66.86 & 72.89 & 17.98 & 57.24 & 71.38 & 77.11 & 81.74 & 27.20 & 66.86 & 78.18 & 82.63 & 86.70 & 39.00 \\
			Market& SoDA & 54.93 & 68.79 & 73.13 & 77.46 & 22.87 & 46.08 & 61.31 & 67.81 & 73.63 & 17.77 & \WWWWWH{55.76} & 70.40 & 76.10 & 81.59 & 24.97 & \WWWWWH{66.18} & 78.36 & 82.48 & 86.64 & 37.11 \\
			\hline
			\hline

		\end{tabular}%
	}
	\label{tab:FDAvsSoDA}%
\end{table*}%

\begin{table*}[htbp]
	\centering
	\caption{\whlll{Comparison with incremental FDA models and online method using JSTL. }}
	\vspace{-0.15cm}
	\begin{tabular}{c||c|c|c||c|c|c||c|c|c}
		\hline
		Dataset & \multicolumn{3}{c||}{Market-1501} & \multicolumn{3}{c||}{SYSU} & \multicolumn{3}{c}{ExMarket} \\
		\hline
		\multirow{2}[2]{*}{Method} & \multicolumn{1}{c|}{rank-1} & \multicolumn{1}{c|}{\multirow{2}[2]{*}{mAP (\%)}} & \multicolumn{1}{c||}{Accumulative} & \multicolumn{1}{c|}{rank-1} & \multicolumn{1}{c|}{\multirow{2}[2]{*}{mAP (\%)}} & \multicolumn{1}{c||}{Accumulative} & \multicolumn{1}{c|}{rank-1} & \multicolumn{1}{c|}{\multirow{2}[2]{*}{mAP (\%)}} & \multicolumn{1}{c}{Accumulative} \\
		& \multicolumn{1}{c|}{matching rate (\%)} &       & \multicolumn{1}{c||}{ Time (s)} & \multicolumn{1}{c|}{matching rate (\%)} &       & \multicolumn{1}{c||}{ Time (s)} & \multicolumn{1}{c|}{matching rate (\%)} &       & \multicolumn{1}{c}{Time (s)} \\
		%		\hline
		%		\multicolumn{1}{c||}{FDA} & \multirow{2}[2]{*}{57.30} & \multirow{2}[2]{*}{28.57}& \multirow{2}[2]{*}{8.20}& \multirow{2}[2]{*}{39.64}& \multirow{2}[2]{*}{22.67}& \multirow{2}[2]{*}{5.28}& \multirow{2}[2]{*}{55.37}& \multirow{2}[2]{*}{22.98}& \multirow{2}[2]{*}{300.95} \\
		%		\multicolumn{1}{c||}{(Chunk Mode)} & & & & & & & & & \\
		\hline
		\whllll{OL-IDM} & 31.50  & 10.48  & 3706.84  & 12.08  & 10.29  & 10588.15  &   50.24   &   18.93    & 1646433.70 \\
		%		\whlll{MLAPG} & 53.27  & 25.78  & 182800.00  & 21.12  & 17.35  & 133748.29  &   -\ -    &   -\ -    & -\ - \\
		\hline
		IDR/QR & 41.15  & 13.20  & 803.59  & \whllll{12.88} & \whllll{10.24}  & 247.17  & 42.70  & 11.20  & 6172.79  \\
		\hline
		IFDA  & 51.45  & 21.21  & 38.22  & \whllll{22.97}  & \whllll{18.12}  & 12.40  & 49.91  & 16.58  & 394.31  \\
		\hline
		Pang's IFDA & {\bf 57.36}  & {\bf 28.58}  & 13.68  &  \whllll{31.08}  & \whllll{25.28}  & 7.65  & {\bf 55.46}  & {\bf 22.97}  & 120.94  \\
		%		Batch FDA & 57.30  & 28.57  & 8.20  & 39.64  & 22.67  & 5.28  & 55.37  & 22.98  & 300.95  \\
		\hline
		{\bf SoDA} & 57.13  & 28.25  & {\bf 7.84}  & \whllll{{\bf 31.74}}  & \whllll{{\bf 26.04}}  & {\bf 4.68}  & 54.93  & 22.87  & {\bf 50.52}  \\
		\hline
	\end{tabular}%
	\label{tab:JSTL}%
\end{table*}%

\begin{table*}[htbp]
	\centering
	\caption{\whlll{Comparison with incremental FDA models and online method using LOMO. }}
	\vspace{-0.15cm}
	\begin{tabular}{c||c|c|c||c|c|c||c|c|c}
		\hline
		Dataset & \multicolumn{3}{c||}{Market-1501} & \multicolumn{3}{c||}{SYSU} & \multicolumn{3}{c}{ExMarket} \\
		\hline
		\multirow{2}[2]{*}{Method} & rank-1 & \multirow{2}[2]{*}{mAP (\%)} & Accumulative & rank-1 & \multirow{2}[2]{*}{mAP (\%)} & Accumulative & rank-1 & \multirow{2}[2]{*}{mAP (\%)} & Accumulative \\
		& matching rate (\%) &       &  Time (s) & matching rate (\%) &       &  Time (s) & matching rate (\%) &       & Time (min) \\
		%		\hline
		%		\multicolumn{1}{c||}{FDA} & \multirow{2}[2]{*}{51.90} & \multirow{2}[2]{*}{23.60}& \multirow{2}[2]{*}{313714.76}& \multirow{2}[2]{*}{46.75}& \multirow{2}[2]{*}{41.82}& \multirow{2}[2]{*}{66509.82}& \multirow{2}[2]{*}{45.64}& \multirow{2}[2]{*}{17.98}& \multirow{2}[2]{*}{69627.00} \\
		%		\multicolumn{1}{c||}{(Chunk Mode)} & & & & & & & & & \\
		\hline
		\whllll{OL-IDM} & 3.95  & 0.73  & 736707.11  &   1.06    &  1.59     &   743335.02   &    3.86   &     0.33  & $>$ 1 week  \\
		%		\whlll{MLAPG} & 38.36  & 19.66  & 9882.52  &   11.76    &  10.62     &   61682.14   &    -\ -   &     -\ -  & -\ -   \\
		\hline
		IDR/QR & 19.36  & 5.09  & 345181.63  &    6.37   &   5.16    &   83903.98    & 19.92  & 3.58  & 74393.24 \\
		\hline
		IFDA  & 38.75  & 13.32  & 314470.08 &   26.83    &   22.59    &    67003.60   & 35.63  & 10.43  & 69668.26 \\
		\hline
		Pang's IFDA & 44.80  & 18.64  & 314461.09 &   35.99    &    31.82   &   66646.88    & 43.50  & 15.42  & 69625.84 \\
		%		Batch FDA & 51.90  & 23.60  & 313714.76 &   46.75    & 41.82      &    66509.82   & 45.64  & 17.98  & 69627.00 \\
		\hline
		{\bf SoDA} & {\bf 52.41}  & {\bf 23.53}  & {\bf 2127.47}  &    {\bf \WWWWWH{47.81}}   &    {\bf \WWWWWH{41.69}}   &   {\bf 3345.30}    & {\bf 46.08}  & {\bf 17.77}  & {\bf 359.28}  \\
		\hline
	\end{tabular}%
	\label{tab:LOMO}%
\end{table*}%

\begin{table*}[htbp]
	\centering
	\caption{\whlll{Comparison with incremental FDA models and online method \WWWWWWH{using} HIPHOP.}}
	\vspace{-0.15cm}
	\begin{tabular}{c||c|c|c||c|c|c||c|c|c}
		\hline
		Dataset & \multicolumn{3}{c||}{Market-1501} & \multicolumn{3}{c||}{SYSU} & \multicolumn{3}{c}{ExMarket} \\
		\hline
		\multirow{2}[2]{*}{Method} & rank-1 & \multirow{2}[2]{*}{mAP (\%)} & Accumulative & rank-1 & \multirow{2}[2]{*}{mAP (\%)} & Accumulative & rank-1 & \multirow{2}[2]{*}{mAP (\%)} & Accumulative \\
		& matching rate (\%) &       &  Time (s) & matching rate (\%) &       &  Time (s) & matching rate (\%) &       & Time (s) \\
		%		\hline
		%		\multicolumn{1}{c||}{FDA} & \multirow{2}[2]{*}{56.92} & \multirow{2}[2]{*}{27.84}& \multirow{2}[2]{*}{184542.32}& \multirow{2}[2]{*}{48.07}& \multirow{2}[2]{*}{42.57}& \multirow{2}[2]{*}{65447.61}& \multirow{2}[2]{*}{47.03}& \multirow{2}[2]{*}{19.75}& \multirow{2}[2]{*}{2135588.18} \\
		%		\multicolumn{1}{c||}{(Chunk Mode)} & & & & & & & & & \\
		\hline
		\whllll{OL-IDM} & 11.97  & 2.22  & 277104.72 & 1.46  & 2.00  &   252626.33   &  7.24     &   0.54    & $>$ 1 week \\
		%		\whlll{MLAPG} & 41.86  & 22.12  & 211678.74 & 4.12  & 5.24  &    91885.23   &  -\ -     &   -\ -    & -\ - \\
		\hline
		IDR/QR & 19.98  & 6.00  & 225226.32 & 10.49  & 9.34  &   86513.64  &    21.97  &    5.32   & 2392922.71 \\
		\hline
		IFDA  & 52.14  & 21.30  & 185390.31 & 25.50  & 22.32  &  66202.88  &  46.08   &    15.50   & 2133499.12 \\
		\hline
		Pang's IFDA & 60.42  & 31.30  & 185174.97 & 51.79  & 47.51  &  \whlllll{65593.56}   &  \whllll{54.84}    &  {\bf \WWWWWWH{25.11}}    & 2135671.23 \\
		%		Batch FDA & 56.92  & 27.84  & 184542.32 & 48.07  & 42.57  &  65447.61   &   47.03    &    19.75  & 2135588.18 \\
		\hline
		{\bf SoDA} & {\bf \WWWWWH{61.88}}  & {\bf \WWWWWH{33.39}}  & {\bf 3620.00}  & {\bf \WWWWWH{53.12}}  & {\bf \WWWWWH{48.48}}  &    {\bf 13849.61}   &   {\bf \WWWWWH{55.76}}    &   {\WWWWWH{24.97}}    & {\bf 83319.79} \\
		\hline
	\end{tabular}%
	\label{tab:HIPHOP}%
\end{table*}%

On all datasets, the cumulative matching characteristic (CMC) curves is shown to measure the performance of the compared methods \whlll{on re-identifying individuals across different camera views under online setting}.
% /***...***/. 
%Besides reporting CMC results,
\WWH{In addition to this,}
we also report results using another two performance metrics: 1) \WWH{r}ank-1 \whlll{Matching Rate}, and 2) mean Average Precision (mAP). mAP \whlll{first computes the area under the Precision-Recall curve for each query and then calculates the mean of Average Precision over all query persons.}
%/***weihong: add examples***/
All experiments were implemented using MATLAB on a machine with CPU E5 2686 2.3 GHz and 256 GB RAM, and the accumulative time of all compared methods were also computed and reported for measuring efficiency.

\subsection{SoDA vs. FDA}

In Sec. \ref{section:theory}, we provide theoretical analysis on the relation between SoDA and FDA.  In this section, we provide empirical evaluation on three datasets by the comparison on Fisher Score between SoDA and FDA in Figure \ref{fig:FisherScore}. The figure indicates that by keeping more rows in the sketch matrix, SoDA can acquire more similar Fisher Score as the one of FDA, and this is supported by Theorem \ref{th:lad_SoDA_quotient}.
We also compared SoDA with FDA on the three datasets in 
Table \ref{tab:FDAvsSoDA},
%Figure \ref{fig:CMC_jstl_all}, \ref{fig:CMC_LOMO_all}, \ref{fig:CMC_HIPHOP_all}, \ref{fig:CMC_AllDatasets_jstl_LOMO_HIPHOP},
 and the comparison \WWH{shows} that they work comparably. Therefore the results reported here have validated that our sketch approach approximates FDA (i.e. the offline model) for extracting discriminant information very well, and thus the effectiveness of our model is verfied both theoretically and empirically.

\subsection{SoDA vs. Incremental FDA Model}

\begin{figure*}[t]
	\begin{center}
		%		\label{fig:CMC_jstl_all}
		%		\fbox{\rule{0pt}{1.7in}\rule{0.9\linewidth}{0pt}}
		%		%\includegraphics[width=\linewidth]{figure/ImageDemo.pdf}
		%		\centering\small\caption{/***weihong: add CMC curve on three datasets using jstl***/}
		{\scriptsize
			\subfigure[] % caption for subfigure
			{
				\label{fig:l_T_Market_JLH}
				\includegraphics[height=0.23\linewidth]{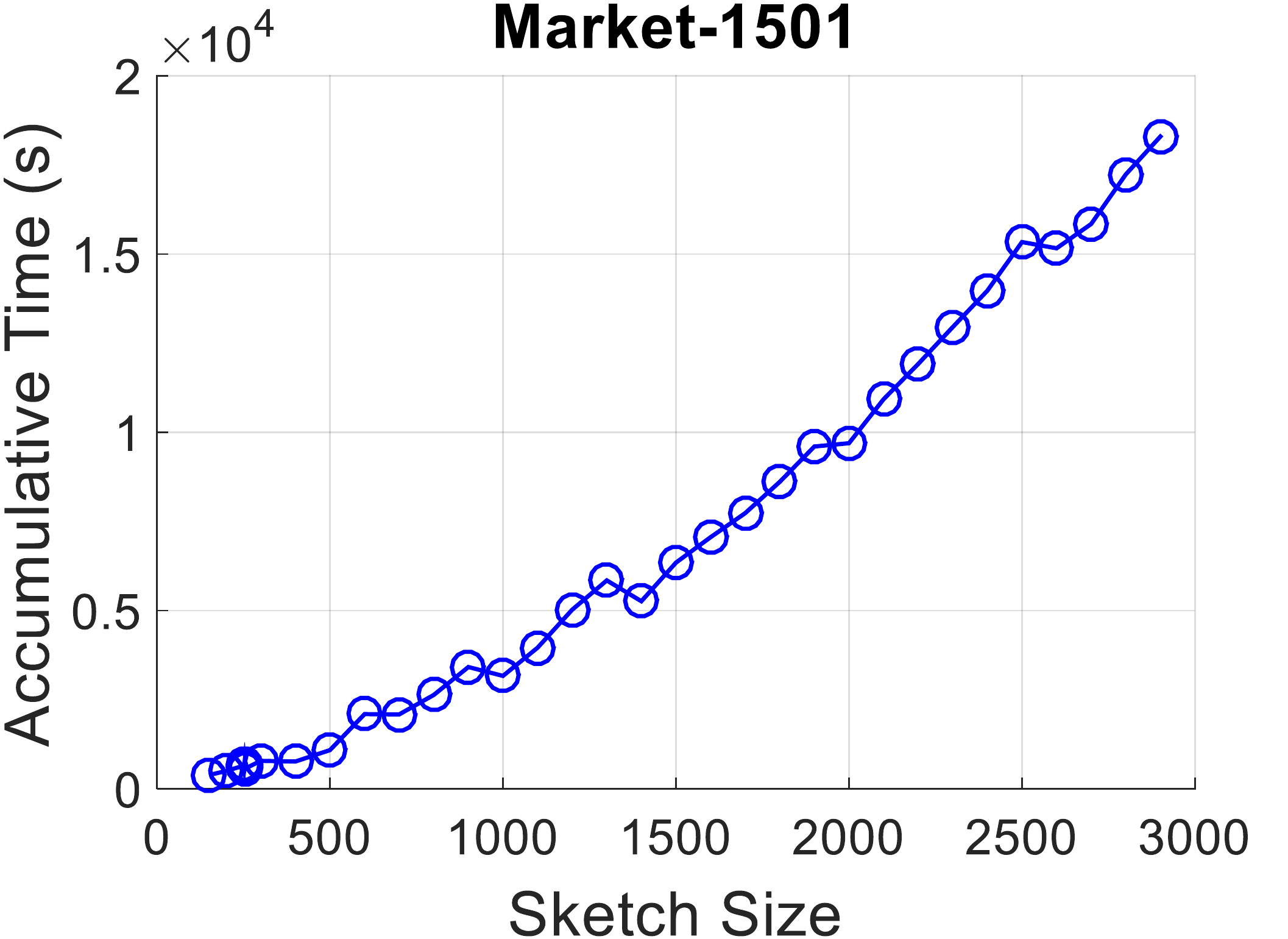}
			}
			\hskip -0.2cm
			\subfigure[] % caption for subfigure
			{
				\label{fig:l_T_SYSU_JLH}
				\includegraphics[height=0.23\linewidth]{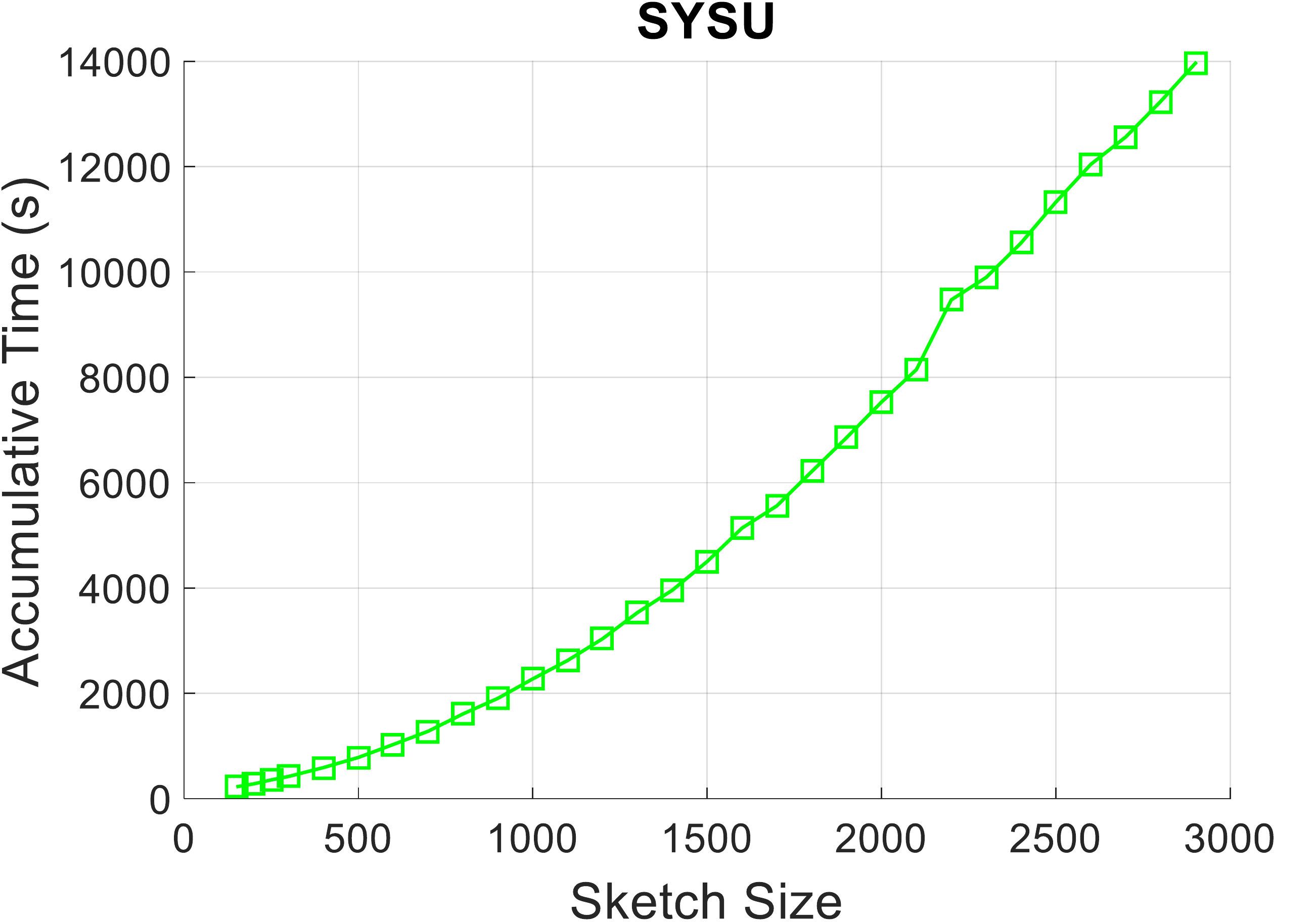}
			}
			\hskip -0.2cm
			\subfigure[] % caption for subfigure
			{
				\label{fig:l_T_ExMarket_JLH}
				\includegraphics[height=0.23\linewidth]{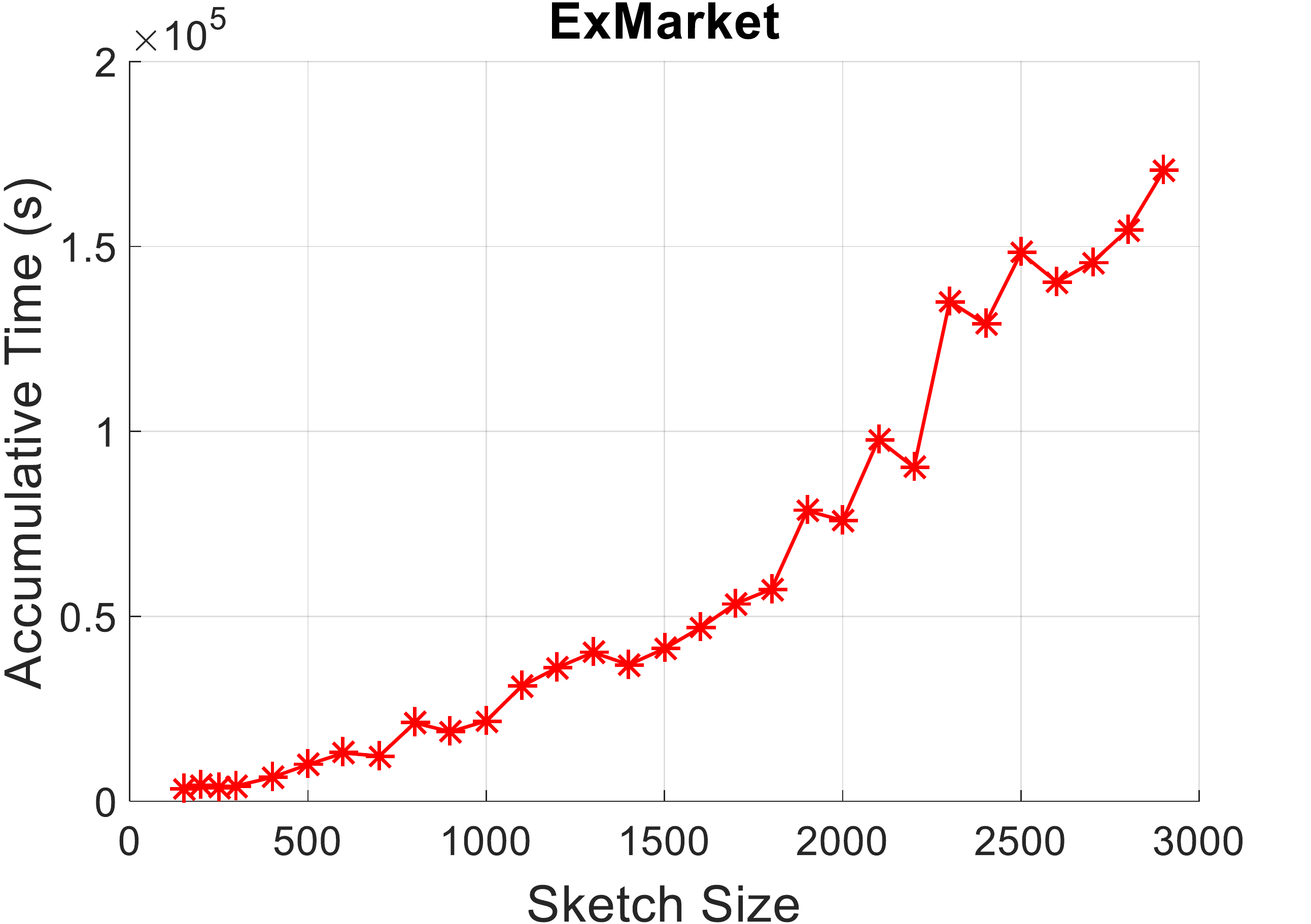}
			}\vskip -0.2cm
			\centering\small\caption{Effect of the sketch size on accumulative time consumption. (Best viewed in color).}
			\label{fig:l_T_JLH}
			%			\label{fig:time_methods}
		}
	\end{center}
\end{figure*}

\begin{figure}[!htbp]
	\begin{center}
		%		\label{fig:CMC_jstl_all}
		%		\fbox{\rule{0pt}{1.7in}\rule{0.9\linewidth}{0pt}}
		%		%\includegraphics[width=\linewidth]{figure/ImageDemo.pdf}
		%		\centering\small\caption{/***weihong: add CMC curve on three datasets using jstl***/}
		{\scriptsize
			\subfigure[] % caption for subfigure
			{
				\label{fig:l_CMC_jstl}
				\includegraphics[height=0.35\linewidth]{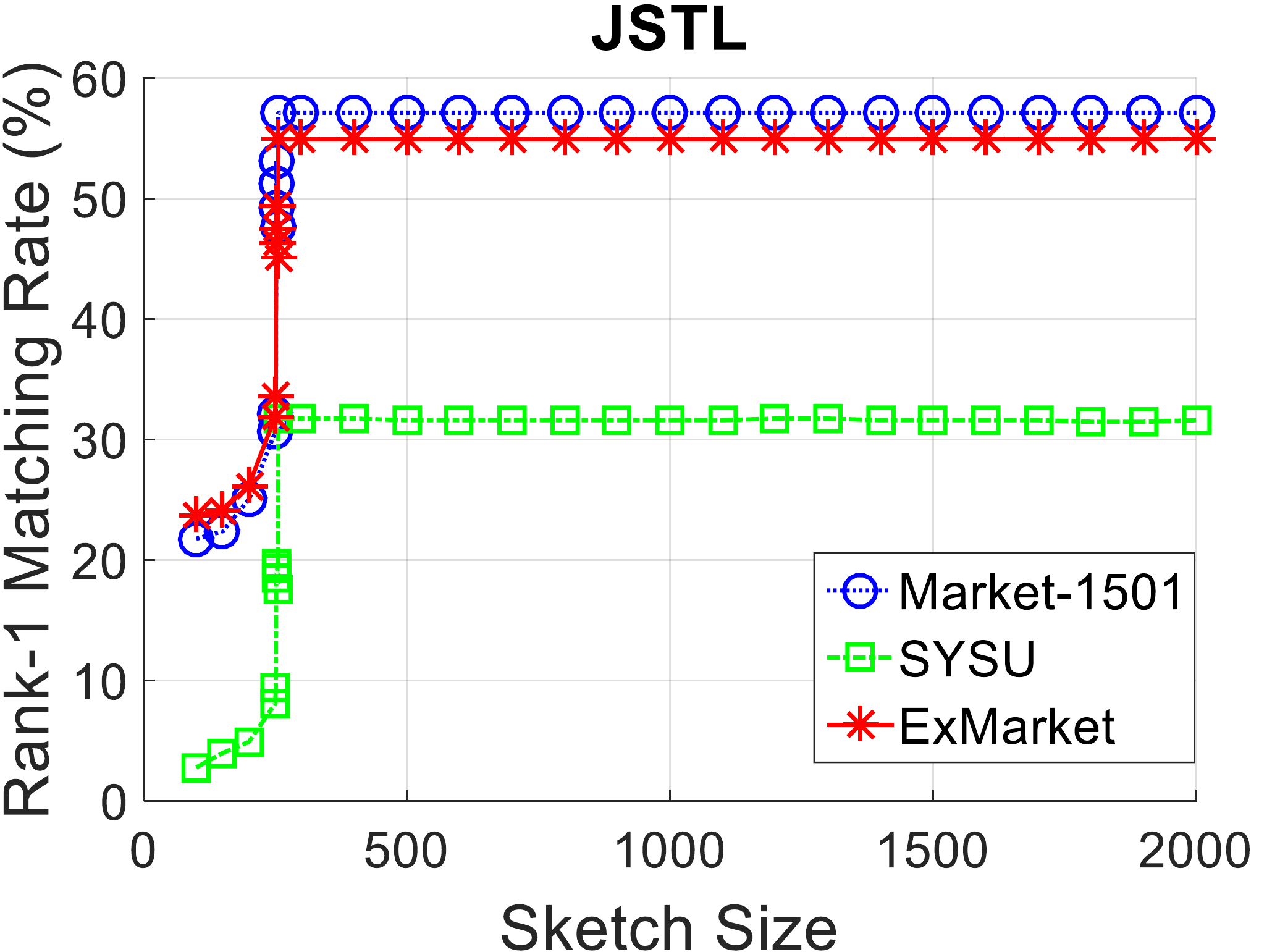}
			}
			\hskip -0.2cm
			\subfigure[] % caption for subfigure
			{
				\label{fig:l_CMC_JLH}
				\includegraphics[height=0.35\linewidth]{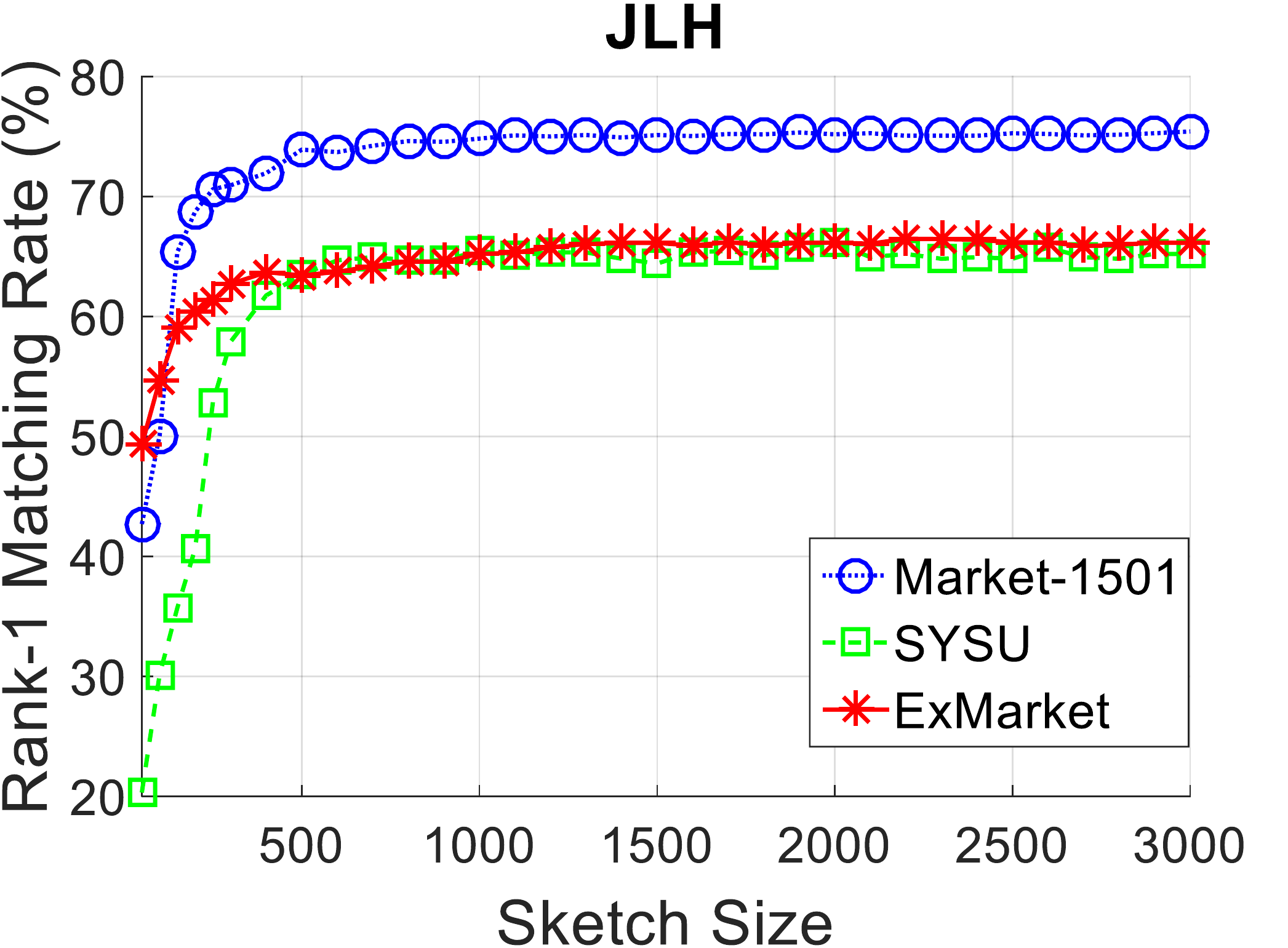}
			}
			\vskip -0.3cm
			\centering\small\caption{Effect of the sketch size on \WWH{r}ank-1 Matching Rate. (Best viewed in color).}
			\label{fig:l_CMC}
			%			\label{fig:time_methods}
		}
	\end{center}
\end{figure}

\begin{table*}[!htbp]
	\centering
	\caption{\whlll{Comparison with incremental FDA models and online method using JLH. }}
	\vspace{-0.15cm}
	\begin{tabular}{c||c|c|c||c|c|c||c|c|c}
		\hline
		Dataset & \multicolumn{3}{c||}{Market-1501} & \multicolumn{3}{c||}{SYSU} & \multicolumn{3}{c}{ExMarket} \\
		\hline
		\multirow{2}[2]{*}{Method} & rank-1 & \multirow{2}[2]{*}{mAP (\%)} & Accumulative & rank-1 & \multirow{2}[2]{*}{mAP (\%)} & Accumulative & rank-1 & \multirow{2}[2]{*}{mAP (\%)} & Accumulative \\
		& matching rate (\%) &       &  Time (s) & matching rate (\%) &       &  Time (s) & matching rate (\%) &       & Time (s) \\
		%		\hline
		%		\multicolumn{1}{c||}{FDA} & \multirow{2}[2]{*}{71.70} & \multirow{2}[2]{*}{43.72}& \multirow{2}[2]{*}{202952.79}& \multirow{2}[2]{*}{51.53}& \multirow{2}[2]{*}{45.69}& \multirow{2}[2]{*}{190070.19}& \multirow{2}[2]{*}{65.32}& \multirow{2}[2]{*}{35.90}& \multirow{2}[2]{*}{2033548.83} \\
		%		\multicolumn{1}{c||}{(Chunk Mode)} & & & & & & & & & \\
		\hline
		\whllll{OL-IDM} & 14.43  & 2.48  & 356136.53  &  \whllll{3.32} &  \whllll{4.91}  &  \whllll{554908.70}   &  10.84    &   0.70    & $>$ 1 week \\
		%		\whlll{MLAPG} & 71.73  & 47.08  & 62174.08  &  20.19 &  21.30  &   256957.20    &  -\ -     &   -\ -    & -\ - \\
		\hline
		IDR/QR &  36.70  & 13.73   & 251934.68  & 15.80   & 12.82  &  220962.28   &    39.64  &    10.85   & 2479401.99 \\
		\hline
		IFDA  &  61.19  & 30.36   &  203537.09 & 21.65   &  18.23  & 189960.96   &   56.24    &    23.46   & 2032679.17 \\
		\hline
		Pang's IFDA & 71.64  &  45.15  & 204406.03  &  56.31  & 49.60  &   189897.24  &   64.64  &   34.80    & 2036601.02 \\
		%		Batch FDA & 71.70  & 43.72   & 202952.79  &  51.53  &  45.69  &  190070.19   &    65.32    &  35.90    & 2033548.83 \\
		\hline
		{\bf SoDA} & {\bf \WWWWWH{75.27}}  & {\bf \WWWWWH{49.82}}  & {\bf 12952.07}  & {\bf \WWWWWH{64.81}}  & {\bf \WWWWWH{59.82}}  &    {\bf 9951.20}   &   {\bf \WWWWWH{66.18}}    &   {\bf \WWWWWH{37.11}}    & {\bf 164475.67} \\
		\hline
	\end{tabular}%
	\label{tab:SoDA_jstl_LOMO_HIPHOP}%
\end{table*}%

\begin{table}[!htbp]
	\centering
	\caption{\whlll{Comparison with offline \wss{re-id} models on Market-1501 using JLH (\%).}}
	\vspace{-0.15cm}
	\begin{tabular}{c||c|c|c|c||c}
		\hline
		Method & rank-1 & rank-5 & rank-10 & rank-20 & Map\\
		%		\hline
		%		\multicolumn{1}{c||}{FDA} & \multirow{2}[2]{*}{71.70} & \multirow{2}[2]{*}{43.72}& \multirow{2}[2]{*}{202952.79}& \multirow{2}[2]{*}{51.53}& \multirow{2}[2]{*}{45.69}& \multirow{2}[2]{*}{190070.19}& \multirow{2}[2]{*}{65.32}& \multirow{2}[2]{*}{35.90}& \multirow{2}[2]{*}{2033548.83} \\
		%		\multicolumn{1}{c||}{(Chunk Mode)} & & & & & & & & & \\
		\hline
		\whlll{CRAFT} & 71.20  & 87.35  & 391.69  &  94.39 &  44.24 \\
		%		\whlll{MLAPG} & 71.73  & 47.08  & 62174.08  &  20.19 &  21.30  &   256957.20    &  -\ -     &   -\ -    & -\ - \\
		\hline
		MLAPG &  69.33  & 85.63   & 90.23  & 93.82   & 46.16 \\
		\hline
		KISSME  &  67.99  & 83.67   &  88.93 & 92.79   &  39.79  \\
		\hline
		XQDA & 67.96 & 83.91 & 88.95 & 93.14 & 43.89 \\
		%		MFA & 70.40  &  86.91  & 91.45  &  94.21  & 42.22\\
		%		\hline
		%		FDA & 74.20  &  88.75  & 92.19  &  94.80  & 49.01  \\
		%		Batch FDA & 71.70  & 43.72   & 202952.79  &  51.53  &  45.69  &  190070.19   &    65.32    &  35.90    & 2033548.83 \\
		\hline
		{\bf SoDA} & \WWWWWH{{\bf 75.27}} & {\bf 89.28} & {\bf 92.70} & {\bf 95.22} & {\bf 49.82} \\
		%		{\bf 75.42}  & {\bf 89.40}  & {\bf 92.64}  & {\bf 95.25}  & {\bf 49.86} \\
		\hline
	\end{tabular}%
	\label{tab:Market_JLH_batch}%
\end{table}%

\begin{table}[!htbp]
	\centering
	\caption{\whlll{Comparison with offline \wss{re-id} models on SYSU using JLH(\%). }}
	\vspace{-0.15cm}
	\begin{tabular}{c||c|c|c|c||c}
		\hline
		Method & rank-1 & rank-5 & rank-10 & rank-20 & mAP\\
		%		\hline
		%		\multicolumn{1}{c||}{FDA} & \multirow{2}[2]{*}{71.70} & \multirow{2}[2]{*}{43.72}& \multirow{2}[2]{*}{202952.79}& \multirow{2}[2]{*}{51.53}& \multirow{2}[2]{*}{45.69}& \multirow{2}[2]{*}{190070.19}& \multirow{2}[2]{*}{65.32}& \multirow{2}[2]{*}{35.90}& \multirow{2}[2]{*}{2033548.83} \\
		%		\multicolumn{1}{c||}{(Chunk Mode)} & & & & & & & & & \\
		\hline
		\whlll{CRAFT} & \WWWWWH{24.70}  & 43.03  & 55.11  &  67.73 &  23.31 \\
		%		\whlll{MLAPG} & 71.73  & 47.08  & 62174.08  &  20.19 &  21.30  &   256957.20    &  -\ -     &   -\ -    & -\ - \\
		\hline
		MLAPG &  18.46  & 35.86   & 47.01  & 58.83   & 18.03 \\
		\hline
		KISSME  &  62.28  & 79.81   &  86.06 & 90.31   &  56.23  \\
		\hline
		XQDA & 64.14 & \WWWWWWH{{\bf 80.88}} & 86.85 & \WWWWWWH{{\bf 91.90}} & 59.12 \\
		%		MFA & 20.58  &  37.05  & 48.07  &  61.89  & 19.47\\
		%		\hline
		%		FDA & 74.20  &  88.75  & 92.19  &  94.80  & 49.01  \\
		%		Batch FDA & 71.70  & 43.72   & 202952.79  &  51.53  &  45.69  &  190070.19   &    65.32    &  35.90    & 2033548.83 \\
		\hline
		{\bf SoDA} & \WWWWWH{{\bf 64.81}} & {80.74} & {\bf 87.25} & {91.77} & {\bf 59.82} \\
		%		{\bf 66.14}  & {\bf 81.41}  & {\bf 88.45}  & {\bf 92.96}  & {\bf 60.45} \\
		\hline
	\end{tabular}%
	\label{tab:SYSU_JLH_batch}%
\end{table}%

\begin{table}[!htbp]
	\centering
	\caption{\whlll{Comparison with offline \wss{re-id} models on ExMarket using JLH(\%). }}
	\vspace{-0.15cm}
	\begin{tabular}{c||c|c|c|c||c}
		\hline
		Method & rank-1 & rank-5 & rank-10 & rank-20 & mAP\\
		%		\hline
		%		\multicolumn{1}{c||}{FDA} & \multirow{2}[2]{*}{71.70} & \multirow{2}[2]{*}{43.72}& \multirow{2}[2]{*}{202952.79}& \multirow{2}[2]{*}{51.53}& \multirow{2}[2]{*}{45.69}& \multirow{2}[2]{*}{190070.19}& \multirow{2}[2]{*}{65.32}& \multirow{2}[2]{*}{35.90}& \multirow{2}[2]{*}{2033548.83} \\
		%		\multicolumn{1}{c||}{(Chunk Mode)} & & & & & & & & & \\
		\hline
		\whlll{CRAFT} & 54.51  & 69.39  & 75.56  &  80.94 &  24.26 \\
		%		\whlll{MLAPG} & 71.73  & 47.08  & 62174.08  &  20.19 &  21.30  &   256957.20    &  -\ -     &   -\ -    & -\ - \\
		\hline
		MLAPG &  50.21  & 65.29   & 70.90  & 77.20   & 25.63 \\
		\hline
		KISSME  &  57.42  & 69.71   &  74.23 & 78.83   &  30.03  \\
		\hline
		XQDA & 55.05 & 68.02 & 73.10 & 77.73 & 28.36 \\
		%		MFA & 47.57  &  63.39  & 69.71  &  75.77  & 18.81\\
		%		\hline
		%		FDA & 74.20  &  88.75  & 92.19  &  94.80  & 49.01  \\
		%		Batch FDA & 71.70  & 43.72   & 202952.79  &  51.53  &  45.69  &  190070.19   &    65.32    &  35.90    & 2033548.83 \\
		\hline
		{\bf SoDA} & \WWWWWH{{\bf 66.18}} & {\bf 78.36} & {\bf 82.48} & {\bf 86.64} & {\bf 37.11} \\
		%		{\bf 66.48}  & {\bf 78.21}  & {\bf 82.19}  & {\bf 86.46}  & {\bf 37.04} \\
		\hline
	\end{tabular}%
	\label{tab:ExMarket_JLH_batch}%
\end{table}%

There are existing works that are related to incremental learning of FDA, which also process sequential data and update the models online.
We compared extensively our method SoDA with three related online/incremental FDA methods, including IFDA \cite{kim2011incremental}, IDR/QR \cite{ye2005idr} and Pang's IFDA \cite{pang2005incremental}.
We show CMC curve of all methods using different types of features in Figure \ref{fig:CMC_jstl_all}, Figure \ref{fig:CMC_LOMO_all}, Figure \ref{fig:CMC_HIPHOP_all} and Figure \ref{fig:CMC_AllDatasets_jstl_LOMO_HIPHOP}.
The results \whlll{illustrate} that the proposed SoDA outperformed the compared incremental FDA. For instance, when using JLH, \WWH{SoDA outperformed Pang's IFDA and achieved \WWWWWH{75.27\%, 64.81\% and 66.18\% rank-1 matching rate} on Market, SYSU and ExMarket, respectively.}
We further report mAP and accumulative time in Table \ref{tab:JSTL}, Table \ref{tab:LOMO}, Table \ref{tab:HIPHOP} and Table \ref{tab:SoDA_jstl_LOMO_HIPHOP}. It suggests that SoDA has a better mAP values especially on SYSU and spends much less time, where for instance SoDA gains {around} 60\% reduction on the cost of computation time, \whlll{as compared with Pang's ILDA}.

\subsection{\whlll{SoDA vs. Related Person re-id Models}}

% online two state-of-the-art models, two basic methods and one online person re-id algorithm were compared.  While 
%We carried out all offline methods on all benchmarks with JLH, and report results in Table /***weihong: refer to a table***/.

\vspace{0.1cm}

\noindent \textbf{\whlll{Comparison with online re-id model}}.
\whlll{
We compared the online re-id method OL-IDM \whllll{\cite{OLDML_sun2014online}} that addresses the same setting as ours in this work. Table \ref{tab:JSTL}, \ref{tab:LOMO}, \ref{tab:HIPHOP} and \ref{tab:SoDA_jstl_LOMO_HIPHOP} tabulate the comparison results. It is noteworthy that our SoDA obtains much more stable results on rank-1 matching rate and mAP performance. Moreover, SoDA is more efficient than OL-IDM, taking 30 times smaller accumulative time.
}

\vspace{0.1cm}

\noindent \textbf{\whlll{Comparison with related subspace model and classical models}}.
\whlll{
We also compared two related subspace model for person re-identification: 1) CRAFT \cite{yingcongpami_chen2017person} ; 2) MLAPG \cite{reid_liao2015efficient}, and two classical methods: 1) KISSME \cite{reidsub_koestinger2012large} ; 2) XQDA \cite{reid_liao2015person}, when the JLH feature was applied on all datasets. All of these methods were learned in an offline way, and the results of these methods on all benchmarks using JLH features are presented in Table \ref{tab:Market_JLH_batch}, \ref{tab:SYSU_JLH_batch} and \ref{tab:ExMarket_JLH_batch}. Among all compared methods, the rank-1 matching rate and mAP of SoDA are the highest, and its accumulative time is the lowest. This indicates that SoDA achieves better or comparable performance of the related off-line subspace person re-id models.
%Although a larger dataset which consists of more training samples (larger dataset size) and longer feature could provide much more information, it could also contain much more noise and could bring some negative impacts on performance. 
%And these promising results shown above indicate that our SoDA is capable of mitigating the influence of noise with high efficiency, especially on large scale dataset with high dimensional feature.
}

%Note that, results of MLAPG on ExMarket is not reported since it requires more memory than 256 GB RAM (out of memory).

\subsection{Further Evaluation of SoDA}

We report the performance of SoDA in Figure \ref{fig:l_CMC} and Figure \ref{fig:l_T_JLH} when varying two key parameters $\ell$.
%, where the x-axis is the sketch size $\ell$, and the y-axis is the rank-1 matching rate
%, mean Average Precision 
%and accumulative time,
%in the subfigures from left to right, 
%respectively. 
%And we evaluate and analyze the proposed SoDA in the following aspects: 1) effect of the sketch size $\ell$; 2) relationship between SoDA and Batch FDA.
%2) Effect of SoDA on high dimensional features; 

\vspace{0.1cm}

\noindent \textbf{Effect of the sketch size $\ell$ using low dimensional feature}.
On all benchmarks, we conducted experiments using JSTL feature ($256-$dimensional) for evaluating the effect of the sketch size $\ell$ on low dimensional feature.
The experimental results in \WWH{Figure \ref{fig:l_CMC_jstl}} indicate that the performance of our proposed SoDA can be improved when $\ell$ (i.e. the rank of $\mathbf{B}$) is larger. That is the performance is better when more variations of passed data are remained in the sketch matrix. It is reasonable because
when more data variations are reserved, the estimated within-class covariance matrix from the sketch matrix $\mathbf{B}$ can approximate the ground-truth one better.
%And the discriminant components generated by minimizing within-class variances subject to the constraint $tr(\mathbf{W}^{T}\mathbf{S}_b\mathbf{W})=1$ are more capable of discriminating different identities.
However, larger $\ell$ indeed increases the accumulative time since the computation complexity and memory depend on $\ell$ when the number of samples and the dimensionality of features are determined (Sec. \ref{section:CC}). 
Fortunately, we empirically find that good performance and low accumulative time can be achieved at the same time when setting the rank of the sketch matrix $\mathbf{B}$ to a properly small value, i.e. $\ell=d=256$.

%This means that the whole large scale dataset can be sketched into a small matrix ($\mathbf{B}\in\mathcal{R}^{d\times d}$) which can be sufficiently small but maintains entire data variances when samples are represented by low-dimensional feature.
%/***weihong: discuss the time consumption***/

\vspace{0.1cm}

\noindent \textbf{Effect of the sketch size $\ell$ using high dimensional features}.
We also show the effect of $\ell$ when using high dimensional features, as some recent proposed state-of-the-art person re-id features are of high dimension, such as LOMO ($26960-$dimensional), HIPHOP ($84096-$dimensional) \WWH{and} also the JLH ($111312-$dimensional) formed in this work. High dimensionality will increase the computational and space complexities (e.g., the whole training data matrix of ExMarket is a $112351\times 111312$ matrix). 
Instead of conducting another online learning for dimension reduction, 
SoDA utilizes 
a set of orthogonal frequent directions maintained by the sketch matrix $\mathbf{B}$ \WWH{for reducing} feature dimension. The experimental results shown in Figure \ref{fig:l_CMC_JLH} and Figure \ref{fig:l_T_JLH} again verify that increasing the sketch size $\ell$ can improve the performance of SoDA but also increase the accumulative time due to extra computation for dimension reduction. Also, on high dimensional feature, setting $\ell$ to be a properly small value (e.g. \WWH{$\ell=\whll{1000}$}) can gain a good balance between good performance and low accumulative computation time.

%/***weihong: modify the expression***/

%\vspace{0.5cm}

\section{Conclusion}\label{section:conclusion}
\whlll{
We contribute to developing a succinct and \WWH{effective} online person re-identification (re-id) methods namely SoDA. Compared with \WWH{existing} online person re-id models, SoDA performs one-pass online learning without any explicit storage of passed observed data samples, meanwhile preserving a small sketch matrix that describes the main variation of passed observed data samples. \WWH{And moreover, SoDA is able to be trained on streaming data efficiently with low computational cost, upon on no elaborated human feedback. Compared with the related online FDA models, we take a novel approach by embedding sketch processing into FDA, and we approximately estimate the within-class variation from a sketch matrix and finally derive SoDA for extracting discriminant components.} More importantly, we have provided in-depth theoretical analysis on how the sketch information affects the discriminant component analysis. The rigorous upper and lower bounds on how SoDA approaches its offline model (i.e. the classical Fisher Discriminant Analysis) are given and proved. Extensive experimental results have clearly illustrated the effectiveness of our SoDA and verified our theoretical analysis.
%, receive a new data sample and maintain the main information of all passed data in a small size sketch matrix before extracting within-class covariance matrix from the sketch matrix. Finally, a set of discriminant components learned at the beginning is updated by minimizing within-class variance and between-class variance. 
}

%\vspace{0.5cm}
\section*{Acknowledgement}

This research was supported by the NSFC (No. 61472456, No. 61573387, No. 61522115). 
%This work was finished when Mr. Li was a master student at Sun Yat-sen University. 
%The corresponding author and principal investigator for this paper is Wei-Shi Zheng.

%\vspace{0.5cm}
{\small
\bibliographystyle{ieee}
\bibliography{online}

\begin{thebibliography}{10}\itemsep=-1pt

\bibitem{reid_ahmed2015improved}
E.~Ahmed, M.~Jones, and T.~K. Marks.
\newblock An improved deep learning architecture for person re-identification.
\newblock In {\em CVPR}, 2015.

\bibitem{online_al_chechik2010large}
G.~Chechik, V.~Sharma, U.~Shalit, and S.~Bengio.
\newblock Large scale online learning of image similarity through ranking.
\newblock {\em JMLR}, 11(Mar):1109--1135, 2010.

\bibitem{reidsub_chen2015mirror}
Y.-C. Chen, W.-S. Zheng, and J.~Lai.
\newblock Mirror representation for modeling view-specific transform in person
  re-identification.
\newblock In {\em IJCAI}, 2015.

\bibitem{TCSVT_chen2016asymmetric}
Y.-C. Chen, W.-S. Zheng, J.-H. Lai, and P.~Yuen.
\newblock An asymmetric distance model for cross-view feature mapping in person
  re-identification.
\newblock {\em TCSVT}, 2016.

\bibitem{yingcongpami_chen2017person}
Y.-C. Chen, X.~Zhu, W.-S. Zheng, and J.-H. Lai.
\newblock Person re-identification by camera correlation aware feature
  augmentation.
\newblock {\em TPAMI}, 2017.

\bibitem{online_al_crammer2006online}
K.~Crammer, O.~Dekel, J.~Keshet, S.~Shalev-Shwartz, and Y.~Singer.
\newblock Online passive-aggressive algorithms.
\newblock {\em JMLR}, 7(Mar):551--585, 2006.

\bibitem{reid_MF_farenzena2010person}
M.~Farenzena, L.~Bazzani, A.~Perina, V.~Murino, and M.~Cristani.
\newblock Person re-identification by symmetry-driven accumulation of local
  features.
\newblock In {\em CVPR}, 2010.

\bibitem{SOINN_furao2006incremental}
S.~Furao and O.~Hasegawa.
\newblock An incremental network for on-line unsupervised classification and
  topology learning.
\newblock {\em NN}, 19(1):90--106, 2006.

\bibitem{reid_MF_gray2008viewpoint}
D.~Gray and H.~Tao.
\newblock Viewpoint invariant pedestrian recognition with an ensemble of
  localized features.
\newblock In {\em ECCV}, 2008.

\bibitem{hiraoka2000successive}
K.~Hiraoka, K.-i. Hidai, M.~Hamahira, H.~Mizoguchi, T.~Mishima, and
  S.~Yoshizawa.
\newblock Successive learning of linear discriminant analysis: Sanger-type
  algorithm.
\newblock In {\em ICPR}, 2000.

\bibitem{online_al_huang2017online}
L.-K. Huang, Q.~Yang, and W.-S. Zheng.
\newblock Online hashing.
\newblock {\em TNNLS}, 2017.

\bibitem{online_al_jain2009online}
P.~Jain, B.~Kulis, I.~S. Dhillon, and K.~Grauman.
\newblock Online metric learning and fast similarity search.
\newblock In {\em ANIPS}, 2009.

\bibitem{reid_jing2015super}
X.-Y. Jing, X.~Zhu, F.~Wu, X.~You, Q.~Liu, D.~Yue, R.~Hu, and B.~Xu.
\newblock Super-resolution person re-identification with semi-coupled low-rank
  discriminant dictionary learning.
\newblock In {\em CVPR}, 2015.

\bibitem{online_fr_kim2010line}
T.-K. Kim, J.~Kittler, and R.~Cipolla.
\newblock On-line learning of mutually orthogonal subspaces for face
  recognition by image sets.
\newblock {\em TIP}, 19(4):1067--1074, 2010.

\bibitem{kim2011incremental}
T.-K. Kim, B.~Stenger, J.~Kittler, and R.~Cipolla.
\newblock Incremental linear discriminant analysis using sufficient spanning
  sets and its applications.
\newblock {\em IJCV}, 91(2):216--232, 2011.

\bibitem{reidsub_koestinger2012large}
M.~Koestinger, M.~Hirzer, P.~Wohlhart, P.~M. Roth, and H.~Bischof.
\newblock Large scale metric learning from equivalence constraints.
\newblock In {\em CVPR}, 2012.

\bibitem{KISSME_koestinger2012large}
M.~Koestinger, M.~Hirzer, P.~Wohlhart, P.~M. Roth, and H.~Bischof.
\newblock Large scale metric learning from equivalence constraints.
\newblock In {\em CVPR}, 2012.

\bibitem{online_tracking_li2016deeptrack}
H.~Li, Y.~Li, and F.~Porikli.
\newblock Deeptrack: Learning discriminative feature representations online for
  robust visual tracking.
\newblock {\em TIP}, 25(4):1834--1848, 2016.

\bibitem{li2016online}
X.~Li, C.~Shen, A.~Dick, Z.~M. Zhang, and Y.~Zhuang.
\newblock Online metric-weighted linear representations for robust visual
  tracking.
\newblock {\em TPAMI}, 38(5):931--950, 2016.

\bibitem{reid_li2015multi}
X.~Li, W.-S. Zheng, X.~Wang, T.~Xiang, and S.~Gong.
\newblock Multi-scale learning for low-resolution person re-identification.
\newblock In {\em ICCV}, 2015.

\bibitem{online_ir_liang2017semisupervised}
J.~Liang, Q.~Hu, W.~Wang, and Y.~Han.
\newblock Semisupervised online multikernel similarity learning for image
  retrieval.
\newblock {\em TMM}, 19(5):1077--1089, 2017.

\bibitem{reid_liao2015person}
S.~Liao, Y.~Hu, X.~Zhu, and S.~Z. Li.
\newblock Person re-identification by local maximal occurrence representation
  and metric learning.
\newblock In {\em CVPR}, 2015.

\bibitem{reid_liao2015efficient}
S.~Liao and S.~Z. Li.
\newblock Efficient psd constrained asymmetric metric learning for person
  re-identification.
\newblock In {\em ICCV}, 2015.

\bibitem{LibertySketching}
E.~Liberty.
\newblock Simple and deterministic matrix sketching.
\newblock In {\em SIGKDD}, KDD '13, 2013.

\bibitem{hitl_liu2013pop}
C.~Liu, C.~Change~Loy, S.~Gong, and G.~Wang.
\newblock Pop: Person re-identification post-rank optimisation.
\newblock In {\em ICCV}, 2013.

\bibitem{lu2012incremental}
G.-F. Lu, J.~Zou, and Y.~Wang.
\newblock Incremental complete lda for face recognition.
\newblock {\em PR}, 45(7):2510--2521, 2012.

\bibitem{reidsub_ma2014person}
L.~Ma, X.~Yang, and D.~Tao.
\newblock Person re-identification over camera networks using multi-task
  distance metric learning.
\newblock {\em TIP}, 23(8):3656--3670, 2014.

\bibitem{reid_martinel2015re}
N.~Martinel, A.~Das, C.~Micheloni, and A.~K. Roy-Chowdhury.
\newblock Re-identification in the function space of feature warps.
\newblock {\em TPAMI}, 37(8):1656--1669, 2015.

\bibitem{hitl_martinel2016temporal}
N.~Martinel, A.~Das, C.~Micheloni, and A.~K. Roy-Chowdhury.
\newblock Temporal model adaptation for person re-identification.
\newblock In {\em ECCV}, 2016.

\bibitem{reidsub_mignon2012pcca}
A.~Mignon and F.~Jurie.
\newblock Pcca: A new approach for distance learning from sparse pairwise
  constraints.
\newblock In {\em CVPR}, 2012.

\bibitem{reid_paisitkriangkrai2015learning}
S.~Paisitkriangkrai, C.~Shen, and A.~van~den Hengel.
\newblock Learning to rank in person re-identification with metric ensembles.
\newblock In {\em CVPR}, 2015.

\bibitem{reid_panda2017unsupervised}
R.~Panda, A.~Bhuiyan, V.~Murino, and A.~K. Roy-Chowdhury.
\newblock Unsupervised adaptive re-identification in open world dynamic camera
  networks.
\newblock In {\em CVPR}, 2017.

\bibitem{pang2005incremental}
S.~Pang, S.~Ozawa, and N.~Kasabov.
\newblock Incremental linear discriminant analysis for classification of data
  streams.
\newblock {\em TSMCB}, 35(5):905--914, 2005.

\bibitem{peng2013chunk}
Y.~Peng, S.~Pang, G.~Chen, A.~Sarrafzadeh, T.~Ban, and D.~Inoue.
\newblock Chunk incremental idr/qr lda learning.
\newblock In {\em IJCNN}, 2013.

\bibitem{reidsub_prosser2010person}
B.~Prosser, W.-S. Zheng, S.~Gong, T.~Xiang, and Q.~Mary.
\newblock Person re-identification by support vector ranking.
\newblock In {\em BMCV}, 2010.

\bibitem{online_fr_schroff2015facenet}
F.~Schroff, D.~Kalenichenko, and J.~Philbin.
\newblock Facenet: A unified embedding for face recognition and clustering.
\newblock In {\em CVPR}, 2015.

\bibitem{OLDML_sun2014online}
Y.~Sun, H.~Liu, and Q.~Sun.
\newblock Online learning on incremental distance metric for person
  re-identification.
\newblock In {\em RB}, 2014.

\bibitem{uray2007incremental}
M.~Uray, D.~Skocaj, P.~M. Roth, H.~Bischof, and A.~Leonardis.
\newblock Incremental lda learning by combining reconstructive and
  discriminative approaches.
\newblock In {\em BMVC}, 2007.

\bibitem{hitl_wang2016human}
H.~Wang, S.~Gong, X.~Zhu, and T.~Xiang.
\newblock Human-in-the-loop person re-identification.
\newblock In {\em ECCV}, 2016.

\bibitem{online_al_warmuth2008randomized}
M.~K. Warmuth and D.~Kuzmin.
\newblock Randomized online pca algorithms with regret bounds that are
  logarithmic in the dimension.
\newblock {\em JMLR}, 9(Oct):2287--2320, 2008.

\bibitem{webb2003statistical}
A.~R. Webb.
\newblock {\em Statistical pattern recognition}.
\newblock 2003.

\bibitem{online_ir_wu2016online}
P.~Wu, S.~C. Hoi, P.~Zhao, C.~Miao, and Z.-Y. Liu.
\newblock Online multi-modal distance metric learning with application to image
  retrieval.
\newblock {\em TKDE}, 28(2):454--467, 2016.

\bibitem{reid_jstl_xiao2016learning}
T.~Xiao, H.~Li, W.~Ouyang, and X.~Wang.
\newblock Learning deep feature representations with domain guided dropout for
  person re-identification.
\newblock In {\em CVPR}, 2016.

\bibitem{reidsub_xiong2014person}
F.~Xiong, M.~Gou, O.~Camps, and M.~Sznaier.
\newblock Person re-identification using kernel-based metric learning methods.
\newblock In {\em ECCV}, 2014.

\bibitem{yan2004immc}
J.~Yan, B.~Zhang, S.~Yan, Q.~Yang, H.~Li, Z.~Chen, W.~Xi, W.~Fan, W.-Y. Ma, and
  Q.~Cheng.
\newblock Immc: incremental maximum margin criterion.
\newblock In {\em SIGKDD}, 2004.

\bibitem{yang2005kpca}
J.~Yang, A.~F. Frangi, J.-y. Yang, D.~Zhang, and Z.~Jin.
\newblock Kpca plus lda: a complete kernel fisher discriminant framework for
  feature extraction and recognition.
\newblock {\em TPAMI}, 27(2):230--244, 2005.

\bibitem{reid_yaolarge}
H.~Yao, S.~Zhang, D.~Zhang, Y.~Zhang, J.~Li, Y.~Wang, and Q.~Tian.
\newblock Large-scale person re-identification as retrieval.

\bibitem{ye2005idr}
J.~Ye, Q.~Li, H.~Xiong, H.~Park, R.~Janardan, and V.~Kumar.
\newblock Idr/qr: an incremental dimension reduction algorithm via qr
  decomposition.
\newblock {\em TKDE}, 17(9):1208--1222, 2005.

\bibitem{reid_MF_zhang2016learning}
L.~Zhang, T.~Xiang, and S.~Gong.
\newblock Learning a discriminative null space for person re-identification.
\newblock In {\em CVPR}, 2016.

\bibitem{MARS_zheng2016mars}
L.~Zheng, Z.~Bie, Y.~Sun, J.~Wang, C.~Su, S.~Wang, and Q.~Tian.
\newblock Mars: A video benchmark for large-scale person re-identification.
\newblock In {\em ECCV}, pages 868--884. Springer, 2016.

\bibitem{reiddata_zheng2015scalable}
L.~Zheng, L.~Shen, L.~Tian, S.~Wang, J.~Wang, and Q.~Tian.
\newblock Scalable person re-identification: A benchmark.
\newblock In {\em ICCV}, 2015.

\bibitem{reid_MF_zheng2015query}
L.~Zheng, S.~Wang, L.~Tian, F.~He, Z.~Liu, and Q.~Tian.
\newblock Query-adaptive late fusion for image search and person
  re-identification.
\newblock In {\em CVPR}, 2015.

\bibitem{reidsub_zheng2011person}
W.-S. Zheng, S.~Gong, and T.~Xiang.
\newblock Person re-identification by probabilistic relative distance
  comparison.
\newblock In {\em CVPR}, 2011.

\bibitem{reid_zheng2015partial}
W.-S. Zheng, X.~Li, T.~Xiang, S.~Liao, J.~Lai, and S.~Gong.
\newblock Partial person re-identification.
\newblock In {\em ICCV}, 2015.

\end{thebibliography}
}
%\vspace{0.5cm}

\begin{appendix}
	\noindent \textbf{\whlllll{Matrix Sketch.}}
	The sketch technique we discuss in this work is related to the matrix sketch \cite{LibertySketching}, which is pass-efficient to read streaming data at most a constant number of time.
	The sketch algorithm learns a set of frequent directions from an $ N \times d$ matrix $\mathbf{X}\in\mathcal{R}^{ N \times d}$ in a stream, where each row of $\mathbf{X}$ is a $d$-dimensional vector. It maintains a sketch matrix $\mathbf{B}\in\mathcal{R}^{\ell\times d} $ containing $\ell~(\ell << N)$ rows and guarantees that:
	\begin{equation}\label{eq:FD2}
	\small
	\mathbf{B}^T\mathbf{B}\preceq\mathbf{X}^T\mathbf{X} \ \ \text{\&} \  \ ||\mathbf{X}^T\mathbf{X} - \mathbf{B}^T\mathbf{B}||\le 2{||\mathbf{X}||}_f^2/\ell.
	\end{equation}
	Such a sketch processing is light in both processing time (bounded by \WWH{$\mathcal{O}(  d\ell^2)$}
%	$\mathcal{O}( N  d\ell)$
	) and space (bounded by $\mathcal{O}(\ell d)$).
	% Also, the whole learning process can be perfectly parallelizable.	
\end{appendix}

\begin{IEEEbiography}
	[{\includegraphics[width=1in,height=1.25in,clip,keepaspectratio]{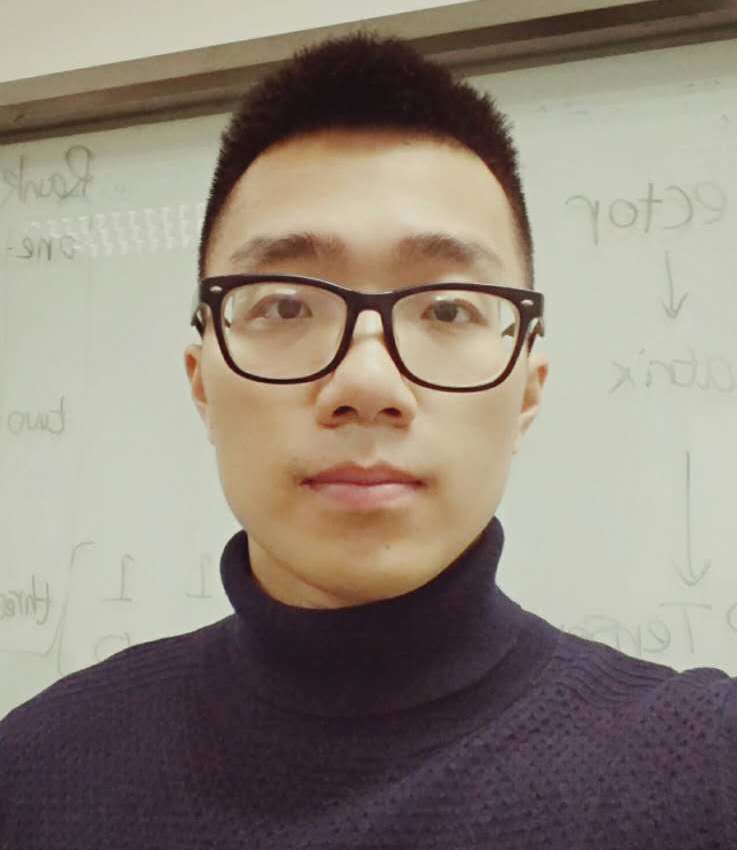}}]
	%[{\fbox{\rule{0pt}{1in}\rule{0.9\linewidth}{0pt}}}]
	{Wei-Hong Li}
	is currently a postgraduate student majoring in Information and Communication Engineering in School of Electronics and Information Technology at Sun Yat-sen University.
	He received the bachelor's degree in intelligence science and technology from Sun Yat-Sen University in 2015. His research interests include person re-identification, object tracking, object detection and image-based modeling. \\ Homepage: \url{https://weihonglee.github.io}.\\
\end{IEEEbiography}

\begin{IEEEbiography}
	[{\includegraphics[width=1in,height=1.5in,clip,keepaspectratio]{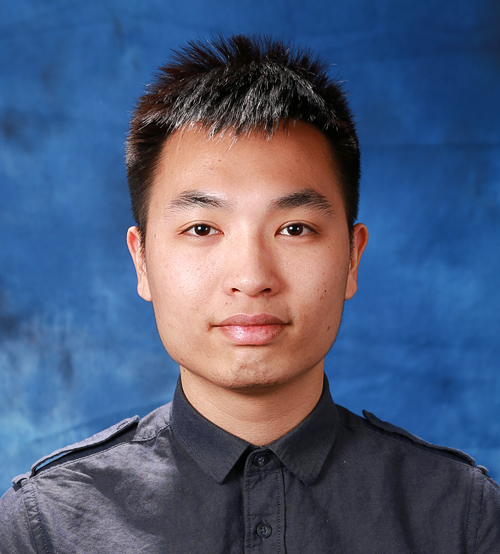}}]
	%[{\fbox{\rule{0pt}{1in}\rule{0.9\linewidth}{0pt}}}]
	{Zhuowei Zhong}
	is a student from Sun Yat-sen University under the joint supervision program of the Chinese University of Hong Kong. He is now graduated and received BSc degree in computer science. His research interest is in Artificial Intelligence, especially in machine learning and constraint satisfaction problem. 

\end{IEEEbiography}

\begin{IEEEbiography}
	[{\includegraphics[width=1in,height=1.25in,clip,keepaspectratio]{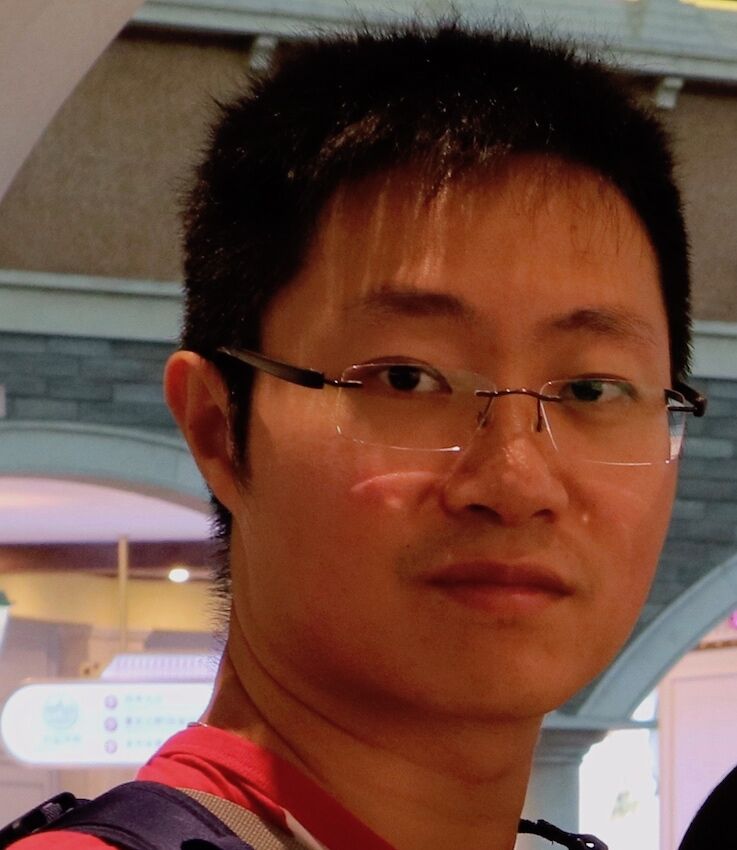}}]
	%[{\fbox{\rule{0pt}{1in}\rule{0.9\linewidth}{0pt}}}]
	{Wei-Shi Zheng}
	is currently a Professor with Sun Yat-sen University. He has joined Microsoft Research Asia Young Faculty Visiting Programme. He has authored over 90 papers, including over 60 publications in main journals (TPAMI, TNN/TNNLS, TIP, TSMC-B, and PR) and top conferences (ICCV, CVPR, IJCAI, and AAAI). His recent research interests include person association and activity understanding in visual surveillance. He was a recipient of Excellent Young Scientists Fund of the National Natural Science Foundation of China, and Royal Society-Newton Advanced Fellowship, U.K. \\ Homepage: \url{http://isee.sysu.edu.cn/~zhwshi}.
\end{IEEEbiography}

\end{document}